%% file: icml2024.tex
\documentclass{article}

\usepackage{microtype}
\usepackage{graphicx}
\usepackage{subfigure}
\usepackage{booktabs} %
\usepackage{preamble_icml2024}

\usepackage{hyperref}

\usepackage[accepted]{icml2024}

\usepackage{amsmath}
\usepackage{amssymb}
\usepackage{mathtools}
\usepackage{amsthm}

\theoremstyle{plain}
\newtheorem{theorem}{Theorem}[section]

\newtheorem{lemma}[theorem]{Lemma}

\theoremstyle{definition}
\newtheorem{definition}[theorem]{Definition}

\theoremstyle{remark}

\usepackage[textsize=tiny]{todonotes}

\icmltitlerunning{Multi-group Learning for Hierarchical Groups}

\begin{document}

\twocolumn[
\icmltitle{Multi-group Learning for Hierarchical Groups}

\icmlsetsymbol{equal}{*}

\begin{icmlauthorlist}
\icmlauthor{Samuel Deng}{yyy}
\icmlauthor{Daniel Hsu}{yyy}
\end{icmlauthorlist}

\icmlaffiliation{yyy}{Department of Computer Science, Columbia University}

\icmlcorrespondingauthor{Samuel Deng}{samdeng@cs.columbia.edu}
\icmlcorrespondingauthor{Daniel Hsu}{djhsu@cs.columbia.edu}

\icmlkeywords{Machine Learning, ICML, multi-group learning, fairness, learning theory}

\vskip 0.3in
]

\printAffiliationsAndNotice{\icmlEqualContribution} %

\begin{abstract}
The multi-group learning model formalizes the learning scenario in which a single predictor must generalize well on multiple, possibly overlapping subgroups of interest. We extend the study of multi-group learning to the natural case where the groups are hierarchically structured. We design an algorithm for this setting that outputs an interpretable and deterministic decision tree predictor with near-optimal sample complexity. We then conduct an empirical evaluation of our algorithm and find that it achieves attractive generalization properties on real datasets with hierarchical group structure.
\end{abstract}

\section{Introduction}\label{sec:intro}
\input{intro/intro}

\subsection{Summary of Results}\label{sec:results}
\input{intro/results}
\subsection{Related Work}
\input{intro/related}

\section{Problem Setup}\label{sec:setup}
\subsection{Setting and notation}
\input{setup/notation}
\subsection{Background on multi-group agnostic learning}\label{sec:background_mgl}
\input{setup/mgal}
\subsection{Hierarchical group structure}\label{sec:hierarchical}
\input{setup/hierarchical}

\section{Algorithms for learning hierarchical groups}
\input{algos/intro}
\subsection{Algorithm: ``Decoupled'' Group ERM}\label{sec:group_erm}
\input{algos/germ}
\subsection{Algorithm: Multi-group Tree}\label{sec:algo_tree}
\input{algos/tree}
\section{Empirical evaluation}\label{sec:experiments}
\input{experiments/intro}
\subsection{Experiment setup and learning task}\label{sec:experiment_setup}
\input{experiments/setup}
\subsection{Main findings}\label{sec:experiment_findings}
\input{experiments/findings}

\subsection{Comparing \prepend and \mgltree}\label{sec:comparison}
\input{experiments/comparison_small}

\section{Conclusion}
\input{conclusion}

\section*{Acknowledgements}
We acknowledge support from the NSF under grant IIS-2040971.

\section*{Impact Statement}
This work focuses on a learning setting initially inspired by the fairness literature (see, e.g., \citet{diana_minimax_2021}, \citet{hebert-johnson_multicalibration_2018}, \citet{kearns_preventing_2018}, \citet{buolamwini_gender_2018}) requiring some fairness criterion to be met on multiple, possibly overlapping subgroups of individuals. If our algorithms are used in the fairness setting, it is important to note that the jury is still out on what definition of algorithmic fairness is philosophically and legally sound. Our particular learning objective corresponds to ensuring that every subgroup is as well-off as the best possible hypothesis for that subgroup; however, this is at odds with a fairness definition such as, for instance, statistical parity \citep{barocas_fairness_2023}. We also note that our setting may be used in applications more general than fairness, when fine-grained accuracy is the only concern. In those cases, the usual societal implications of the application area apply.

\bibliography{final_references.bib}
\bibliographystyle{icml2024}

\newpage
\appendix
\onecolumn
\input{appendix}
\end{document}

%% file: intro/intro.tex
\topic{In the classical statistical learning setup, the goal is to construct a predictor with high \textit{average} accuracy. However, in many practical learning scenarios, an aggregate, on-average measure of performance is insufficient.} In general, average-case performance can obscure performance for subgroups of examples --- a predictor that boasts 95\% accuracy on average might only be 50\% accurate on an important subgroup comprising 10\% of the population. Such a subgroup might be difficult for a predictor trained for aggregate performance because it is not represented well during training \citep{oakden-rayner_hidden_2019} or it admits spurious correlations \citep{borkan_nuanced_2019}. 

\topic{Recent work has shown that, in various learning domains, constructing a predictor that performs well on multiple subgroups is crucial.} For instance, when \textit{fairness} is a concern, a natural desideratum is that a predictor be accurate not only on average, but also conditional on possibly intersecting subgroups such as race and gender \citep{hardt_equality_2016, diana_minimax_2021}. In medical imaging, a model might systematically err on rarer cancers to achieve a higher average accuracy, leading to possibly life-threatening predictions on individuals with the rare cancer \citep{oakden-rayner_hidden_2019}. Similar demands for group-wise accuracy appear in domains as varied as medical imaging \citep{bissoto_constructing_2019, oakden-rayner_hidden_2019, degrave_ai_2021}, facial recognition \citep{buolamwini_gender_2018, kuehlkamp_gender--iris_2017}, object recognition \citep{devries_does_2019}, and NLP \citep{orr_bootleg_2020, varma_cross-domain_2021, borkan_nuanced_2019}. A common theme is that these high error subgroups come from some \textit{hierarchical group structure}. Demographic subgroups in fairness-aware scenarios naturally arrange via, say, race, gender, and age \citep{borkan_nuanced_2019} or subgroups of face images in facial recognition naturally arrange via, say, gender, facial expression, and hair color \citep{liu_deep_2015}.

\topic{Motivated by these concerns, we study the \textit{multi-group (agnostic PAC) learning} model first proposed by \citet{rothblum_multi-group_2021}.} In multi-group learning, there is a collection of \textit{groups} $\cG$ comprised of possibly overlapping subsets of the input space and a \textit{benchmark hypothesis class} $\cH$ of predictors. The objective in multi-group learning is to construct a predictor $f$ whose accuracy is not much worse than the best (possibly different) predictor $h_g \in \cH$ for each group $g \in \cG$ simultaneously. This generalizes the traditional (on-average) statistical learning setting when $\cG$ contains the entire input space. \citet{rothblum_multi-group_2021} gives an initial boosting-based algorithm that achieves multi-group learning, and \citet{tosh_simple_2022} provide simpler algorithms with improved sample complexity. These results all assume no further structure on $\cG$. In this paper, we study the natural special case in which $\cG$ has hierarchical structure. 

\topic{To this end, we present two main contributions. First, we identify a mathematically natural structural property for groups --- namely, hierarchical structure --- that permits a simple and efficient algorithm that achieves multi-group learning with near-optimal group-wise error rates. Second, we conduct an extensive empirical evaluation of this algorithm against other baselines for multi-group learning on several real datasets with hierarchical group structure.} This partially addresses a question posed by \citet{tosh_simple_2022} to design a learning algorithm that outputs a simple, deterministic predictor like their \prepend algorithm enjoys near-optimal error rates.
The empirical results show our algorithm indeed achieves multi-group generalization on par with---and, on some groups, better than---existing multi-group learning methods.
This supports our theory and suggests that, in the case where groups are hierarchically structured, the tree representation of our algorithm is a good inductive bias for generalization on the subgroups.

%% file: intro/results.tex
\topic{First, we analyze two algorithms for hierarchical multi-group learning, showing ``near-optimal'' group-wise excess error rates for the latter.} The naïve first algorithm, detailed in Section \ref{sec:group_erm}, simply applies ERM to a partition of the input space formed by the hierarchical groups. This procedure of training a separate ``decoupled'' predictor on disjoint subsets of the input space to achieve some desired multi-group property has been studied before \citep{dwork_decoupled_2017}; we simply extend the analysis to the hierarchical multi-group learning framework. 

Our main algorithm, \mgltree (Algorithm \ref{alg:tree}), outputs a simple decision tree predictor that is guaranteed to achieve error competitive to the best possible predictor in a benchmark hypothesis class for every group simultaneously. Importantly, the best benchmark hypothesis for each group may differ. Namely, we show that, for finite $\cH$, \mgltree achieves multi-group learning with a group-wise excess error rate of $O\left(\sqrt{\log (|\cH||\cG|)/n_g} \right)$, where $\cH$ is the hypothesis class, $\cG$ is the collection of hierarchical groups, and $n_g$ is the number of training examples for a group $g.$ We note that this also extends to $\cH$ of bounded VC dimension, and these guarantees hold for any bounded loss function.

\topic{In comparison, an algorithm of \citet{tosh_simple_2022} also achieves the state-of-the-art group-wise excess error rate of $O\left(\sqrt{\log (|\cH||\cG|)/n_g} \right)$ for finite $\cH$ and $\cG$ (without the hierarchical restriction on $\cG$) with a black-box online-to-batch reduction.} The resulting predictor, however, is a randomized ensemble of $n$ (also) random predictors, one for each training sample, and requires the explicit enumeration of a finite hypothesis class. It is unclear if a practical implementation of such an algorithm exists.
A separate algorithm
of \citet{tosh_simple_2022} and \citet{globus-harris_algorithmic_2022}, called \prepend, on the other hand, outputs a simple, interpretable, and determnistic decision list predictor that avoids enumerating $\cH$. Its group-wise excess error rate, however, is not optimal: $O\left( \sqrt[3]{\log(|\cH||\cG|)/\gamma n_g} \right)$, where $\gamma$ is the probability of the group with the smallest mass. Resolving this trade-off between simplicity and optimal rates was an unresolved issue posed in \citet{tosh_simple_2022}. Our algorithm avoids the trade-off and achieves the state-of-the-art excess error rate with a simple predictor that refines its predictions through a breadth-first search on the tree induced by the hierarchical group structure. 

\topic{Second, we conduct an empirical evaluation of existing multi-group learning algorithms on datasets with natural hierarchical group structure.} We evaluate \mgltree, the \prepend algorithm, and the ERM baseline
in several binary classification problems. We evaluate different benchmark hypothesis classes (linear predictors, decision trees, and tree ensembles) and different hierarchical collections of groups. We consider twelve US Census datasets from \citet{ding_retiring_2022} for the tasks of predicting income, health care coverage, and employment. We consider four different collections of hierarchically structured groups arising from different collections of demographic attributes: (1) race, sex, and age, (2) race, sex, and education, (3) US state, race, and sex, and (4) US state, race, and age.

On each dataset, we evaluate the generalization of each algorithm through classification accuracy on a held-out test set. We find that \mgltree consistently improves on the accuracy of the global ERM and group-specific ERM hypotheses from our benchmark class, validating our theory. Our algorithm also consistently achieves equal or higher accuracy than \prepend, suggesting that the decision tree representation of our final predictor is a more favorable inductive bias to hierarchically structured groups than a decision list of possibly arbitrary order. We also find that, in some cases, \mgltree even outperforms more complex ERM predictors (e.g., tree ensembles) which were empirically observed in previous work to have good multi-group generalization properties \citep{gardner_subgroup_2023, pfohl_comparison_2022}.

%% file: intro/related.tex
Our work touches upon several areas of theoretical and empirical literature, including multi-group learning, fairness, and hidden stratification.

\textbf{Multi-group learning.} \topic{The multi-group agnostic PAC learning setting was first formalized by \citet{rothblum_multi-group_2021}, and initial algorithms for multi-group learning relied on a reduction to the outcome indistinguishability framework of \citet{dwork_outcome_2020}.} \citet{tosh_simple_2022} introduced additional algorithms for achieving multi-group learning with improved, near-optimal sample complexity, discussed in Section \ref{sec:results}. Both of these works, like ours, study this problem in the batch setting. \citet{blum_advancing_2019} studies the online setting, where the goal is regret minimization with respect to multiple, possibly intersecting subgroups. \citet{rittler_agnostic_2023} considers the problem in the active learning setting, where the learner can decide which examples to obtain labels on. \citet{blum_collaborative_2017, haghtalab_-demand_2022} study another learning model in which the learner may decide the number of samples to obtain from different distributions (which are not necessarily subsets of the input space). This setup is related to \textit{multi-task learning}, in which samples are drawn from multiple, possibly related tasks (i.e., distributions, which are also not necessarily subsets of the input space), and the goal is to perform well on a single test distribution while avoiding so-called ``negative transfer'' between tasks \cite{ben-david_theoretical_2002, ben-david_exploiting_2003, mansour_domain_2008, li_identification_2023}. Multi-group agnostic PAC learning is also related to a recent strand of work in \textit{multicalibration} for providing more stringent calibration guarantees across multiple, possibly intersecting groups \citep{hebert-johnson_multicalibration_2018, kim_multiaccuracy_2019, globus-harris_multicalibration_2023}. 

\textbf{Group and minimax fairness.} \topic{A primary motivation for multi-group learning comes from the group fairness literature, where the goal is to achieve specific fairness criterion over multiple intersecting subgroups of individuals defined by characteristics such as sex or race.} \citet{hebert-johnson_multicalibration_2018} and \citet{kearns_preventing_2018} initiated the study of fairness across multiple \textit{intersecting} subgroups constructed from different combinations of protected attribute values (e.g.~intersections of race and gender). These works developed algorithms for various fairness notions in the algorithmic fairness literature (e.g.~equalized odds or false positive parity). In the context of fairness, multi-group learning is related to the natural fairness criterion of ``minimax fairness,'' in which the objective is to minimize the maximum loss across all groups rather than achieving parity in losses \citep{diana_minimax_2021}. We note that, when fairness is a concern, ``gerrymandered,'' intersectional groups naturally admit a hierarchical group structure \citep{kearns_preventing_2018} once the defining features are ordered.

\textbf{Slice discovery and hidden stratification.} \topic{In many other learning scenarios, fine-grained high error subgroups are unknown ahead of time.} A plethora of empirical work has uncovered various settings in which predictors with high overall accuracy predict poorly on meaningful subgroups of the input space \citep{buolamwini_gender_2018, oakden-rayner_hidden_2019, borkan_nuanced_2019}. The burgeoning literature of \textit{slice discovery} aims to develop methods that identify high-error subgroups that are not known \textit{a priori} \citep{chung_automated_2019, eyuboglu_domino_2022, sohoni_no_2022}. These slice discovery methods find these high-error subgroups by recursively slicing features into smaller and smaller subcategories, so the collection of found subgroups are presented in a hierarchical structure. This literature serves as a motivation to our work, but we note that, in multi-group learning, the collection of groups is known ahead of time (perhaps already ``discovered'' as a result of one of these methods).

%% file: setup/notation.tex
We are interested in the standard supervised statistical learning setting with additional group structure. We denote the input space as $\cX$, the output space as $\cY$, and the decision space as $\cZ$. It is not necessary that $\cZ = \cY.$ We denote $\cD$ as an arbitrary joint probability distribution over $\cX \times \cY$. We denote $S = \{(x_1, y_1), \dots, (x_n, y_n)\}$ as an i.i.d.~set of random examples drawn from $\cD.$ A \textit{predictor (hypothesis)} is a function $h: \cX \rightarrow \cZ$. A \textit{benchmark hypothesis class} $\cH$ is a collection of such functions.

We will compute probabilities and expectations with respect to the true distribution $\cD$, or the empirical distribution over $S$ (where a new example $(x, y)$ is drawn uniformly at random from $S$). We denote these as $\PP[\cdot]$, $\EE[\cdot]$ for the true distribution, and $\PP_S[\cdot]$, $\EE_S[\cdot]$ for the empirical distribution. 

For a bounded loss function $\ell: \cZ \times \cY \rightarrow [0, 1]$, we denote the \textit{risk} as $L_{\cD}(f) := \EE_{\cD}[\ell(f(x), y)]$ and the \textit{empirical risk} as $L_S(f) := \EE_S[\ell(f(x), y)] = \frac{1}{n} \sum_{i = 1}^n \ell(f(x_i), y_i)$.

\textbf{Multi-group learning notation.} A \textit{group} is a subset of the input space, $g \subseteq \cX$. We will overload $g$ to also denote the indicator function for the set, $g(x) := \indic{x \in g}$. $\cG$ denotes the collection of groups of interest, so it is a collection of subsets of the input space. For a set of $n$ examples, denote $n_g$ as the number of examples where $x \in g$, i.e. $n_g := \sum_{i = 1}^n g(x_i).$

For a distribution $P$ (in our case, either $\cD$ or $S$), we will write $\PP[A \mid g]$ (respectively, $\EE[X \mid g]$) to denote the probability of an event $A$ (respectively, expectation of a random variable $X$) conditioned on the event that, on a random $X \sim P$, $g(X) = 1$ (i.e. $X \in g$). 

Our main metric for comparison will be the \textit{group-conditional risk}: $L_{\cD}(f \mid g) := \EE_{\cD}[\ell(f(x), y) \mid g]$. We also denote the \textit{group-conditional empirical risk} as $L_S(f \mid g) := \frac{1}{n_g} \sum_{i = 1}^n g(x_i) \ell(f(x_i), y_i).$ For a group $g \in \cG$, we use $\hat{h}^g \in \argmin_{h \in \cH} L_S(h \mid g)$ to refer to a hypothesis in $\cH$ that minimizes the empirical risk on \textit{only} the examples from $g.$

%% file: setup/mgal.tex
Multi-group (agnostic PAC) learning is a generalization of the traditional PAC learning setting \cite{valiant_theory_1984}. The goal in multi-group learning is to construct a predictor that achieves small risk \textit{conditional} on the membership of an example in a group, for a collection of (potentially overlapping) groups $\cG$ simultaneously. ``Small risk'' is measured with respect to the best hypothesis in a \textit{benchmark hypothesis class}, $\cH.$

A \emph{multi-group learning algorithm} for $(\cG, \cH)$ takes as input a dataset of $n$ i.i.d.~labeled examples from a distribution $\cD$ and outputs a predictor $f: \cX \rightarrow \cZ$. The goal of such an algorithm is to ensure that the \textit{group-wise excess errors} $L_\cD(f \mid g) - \inf_{h \in \cH} L_{\cD}(h \mid g)$ are small for all $g \in \cG$. Formally, we say that such an algorithm achieves the \emph{multi-group learning property} with group-wise excess error bounds of $\epsilon_g(n, \delta).$

\begin{definition}[Multi-group learning property]\label{def:mgl}
A multi-group learning algorithm for $(\cG, \cH)$ satisfies the \emph{multi-group learning property} with group-wise excess error bounds $\epsilon_g(\cdot, \cdot)$ for $g \in \cG$ if it returns a predictor $f: \cX \rightarrow \cZ$ such that, with probability at least $1 - \delta$ over $n$ i.i.d.~examples from $\cD$,
\begin{equation*}\label{eq:unif_mgl}
    L_{\cD}(f \mid g) - \inf_{h \in \cH} L_{\cD}(h \mid g) \leq \epsilon_g(n, \delta) \quad \text{ for all } g \in \cG.
\end{equation*}
\end{definition}

These group-wise excess errors will all be decreasing functions of $n_g$, and it is straightforward to convert a group-wise excess error bound to an upper bound on the sample size $n_g$ sufficient to guarantee a given excess error bound $\epsilon_g > 0$ for group $g$.

Note that $f$ is a \textit{single} predictor that simultaneously achieves this objective with respect to the best hypothesis in \textit{every} group. In fact, every group may have a different best predictor achieving $\inf_{h \in \cH} L_{\cD}(h \mid g)$, and there may be no single $h \in \cH$ that achieves small risk for all groups simultaneously, as Figure \ref{fig:best_h} shows. Thus, we focus on the \textit{improper} learning setting, where $f$ is allowed to be outside of $\cH.$ 

\begin{figure}[ht]
\vskip 0.2in
\begin{center}
\centerline{\includegraphics[width=\columnwidth]{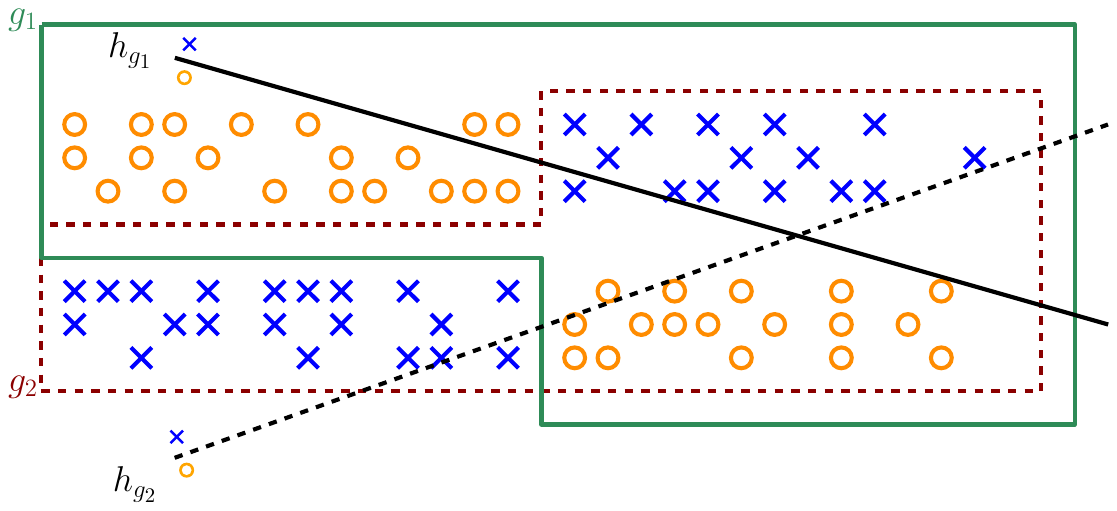}}
\caption{\textbf{No best $h$ for all groups simultaneously.} Letting $\cH$ be the class of halfspaces, the groups $g_1$ (indicated by the green solid line) and $g_2$ (indicated by the red dotted line) overlap, but their optimal predictors $h_{g_1}$ and $h_{g_2}$ are much different.} 
\label{fig:best_h}
\end{center}
\vskip -0.2in
\end{figure}

When $\cX \in \cG$, this is a generalization of the traditional agnostic PAC learning setting, where it is well-known that \textit{empirical risk minimization} (ERM) is necessary and sufficient. \citet{tosh_simple_2022} observe that, if we are \textit{a priori} concerned with the conditional distribution on each $g \in \cG$, then the optimal excess error from ERM on each group is $O(\sqrt{\log(|\cH|)/n_g})$ for finite $\cH$ and $O(\sqrt{d \log(n_g)/n_g}$ for $\cH$ with VC dimension $d < \infty.$ They show that there exists an algorithm that achieves this at a ``near-optimal'' rate that incurs an extra $O\sqrt{\log|\cG|}$ factor for finite $\cG$, with rate $O(\sqrt{\log(|\cG||\cH|)/n_g})$.
In Section \ref{sec:algo_tree}, we give a much simpler, deterministic algorithm that achieves this rate in the case of hierarchically structured groups.

We briefly remark that, when \citet{rothblum_multi-group_2021} introduced the learning model in Definition \ref{def:mgl}, they required the additional assumption of ``multi-PAC compatibility.'' This is the requirement that $\cH$ contains a hypothesis that is nearly optimal for every group simultaneously. This precludes settings where the best hypothesis on each group can be vastly different, such as in Figure \ref{fig:best_h} or the practical settings discussed in Section \ref{sec:intro}. Therefore, we avoid this assumption and provide group-wise excess error bounds where the best $h \in \cH$ is allowed to be different across the groups $\cG$.

%% file: setup/hierarchical.tex
\topic{We note a couple important features of the collection of groups $\cG$ in this learning model.} First, $\cG$ is fixed \textit{a priori} and is known during both training and test. In addition, the group identity of each example is known during training and test; that is, $g(x)$ can always be evaluated. In some applications to fairness, group membership for certain ``protected attributes'' may not be available at training or test time due to privacy or legal concerns \citep{veale_fairer_2017, holstein_improving_2019, lahoti_fairness_2020}. 

Because of such restrictions, recent works in fairness have pivoted from demanding fairness on a small, fixed number of coarse demographic attributes (e.g.~race, sex) to instead fixing a combinatorially large number of subgroups based on conjunctions of attributes \citep{kearns_preventing_2018}. For example, considering all the possible subgroups generated by conjunctions of race, sex, and age generates an exponentially large number of subgroups. This motivates the importance of obtaining near-optimal sample complexity rates that are \textit{logarithmic} in $|\cG|$, the total number of subgroups.

\topic{Subdividing an input space $\cX$ based on relevant attributes in a given order generates an exponentially large number of hierarchically structured subgroups.} We can also obtain a hierarchically structured collection of groups from dividing an input space top-down into categories and further subcategories. In the unsupervised machine learning context, hierarchically structured groups are obtained via, say, agglomerative clustering or a dyadic partitioning.

\begin{definition}[Hierarchically Structured Groups]\label{def:hierarchical}
A collection of groups $\cG$ is \emph{hierarchically structured} if, for every pair of distinct groups $g, g' \in \cG$, exactly one of the following holds:
\begin{enumerate}
    \item $g \cap g' = \emptyset$ ($g$ and $g'$ are disjoint).
    \item $g \subset g'$ ($g$ is contained in $g'$).
    \item $g' \subset g$ ($g'$ is contained in $g$).
\end{enumerate}
\end{definition}

\topic{A hierarchically structured collection of groups also naturally admits a rooted tree or collection of trees, where each group $g \in \cG$ is a node in the tree.} Figure \ref{fig:tree_example} gives an example.

\begin{figure}[ht]
\vskip 0.2in
\begin{center}
\centerline{\includegraphics[width=\columnwidth]{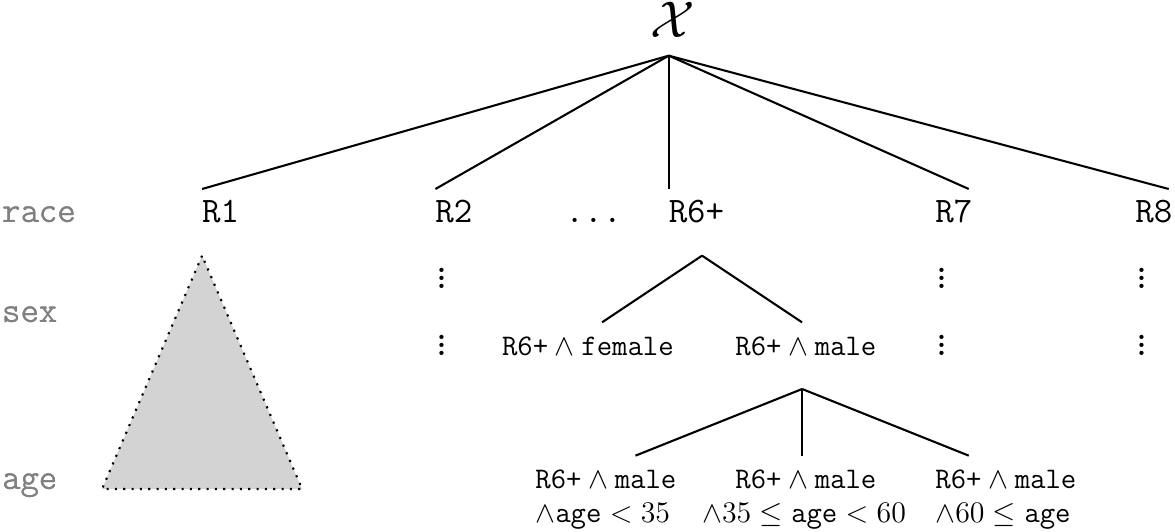}}
\caption{\textbf{Example of a hierarchically structured tree.} Each level of the tree above corresponds to a demographic attribute ($\texttt{race}$, $\texttt{sex}$, and $\texttt{age}$). Proceeding down the tree yields increasingly granular subgroups. The leaves are the most granular level, with subgroups such as $\texttt{R6+} \land \texttt{male} \land \texttt{age} < 35$.}
\label{fig:tree_example}
\end{center}
\vskip -0.2in
\end{figure}

\begin{definition}[Hierarchical Tree]\label{def:tree}
Let $\cX$ be an input space. A collection of hierarchically structured groups $\cG$ on $\cX$ with $\cX \in \cG$ (without loss of generality) admits a \emph{(rooted) tree} $\gtree$ such that each node of $\gtree$ is a group $g$, $\cX$ is the root of $\gtree$, and, for every pair $g, g' \in \gtree$:
\begin{enumerate}
    \item If $g$ is a child of $g'$, then $g \subset g'$.
    \item If $g$ and $g'$ are the same distance from the root, then $g \cap g' = \emptyset.$
\end{enumerate}
\end{definition}

Subdividing on attributes in this way can create different hierarchical trees depending on the order of attributes. For example, subdividing by $\texttt{race}$ first then $\texttt{sex}$ gives a different tree than subdividing on $\texttt{race}$ first and then $\texttt{age}$ (although the leaves are the same). In our experiments (Section \ref{sec:experiments}), similar findings held for all orderings. In practice, this ordering may be decided by a domain expert.

%% file: algos/intro.tex
\topic{In this section, we analyze two multi-group learning algorithms for hierarchically structured groups.}
The first, in Section \ref{sec:group_erm}, is a baseline algorithm of \citet{dwork_decoupled_2017} studied in the context of fairness, but we include a brief analysis in our setting as a warm-up. It does not achieve the state-of-the-art group error rates for multi-group learning.
Our main algorithm, in Section \ref{sec:algo_tree}, outputs a simple decision tree predictor that obtains the near-optimal excess error rates, satisfying a desideratum from \citet{tosh_simple_2022} for the case of hierarchically structured $\cG$.
We restrict attention to finite $\cH$; infinite $\cH$ with finite VC dimenison are considered in Appendices~\ref{sec:g_erm_proof} and~\ref{sec:tree_proof}, along with the proofs.

%% file: algos/germ.tex
\topic{We first analyze a simple algorithm for multi-group learning in the hierarchical setting: training a predictor on each disjoint leaf node.}
This procedure has been considered before in previous work in the context of fairness \citep{dwork_decoupled_2017}, but we include an analysis of its sample complexity in the multi-group learning framework with hierarchical groups for completeness and comparison.

Let $\cG$ be a hierarchically structured collection of groups, with hierarchical tree $\gtree$. Let $g_1, \dots, g_N$ be the leaves of the tree, and let $\hat{h}^{g_i} \in \argmin_{h \in \cH} L_S(h \mid g_i)$ be the ERM predictor trained only on samples from leaf $g_i$. Our predictor $f: \cX \rightarrow \cZ$ is simply:
\begin{equation}\label{alg:g-erm}
    f(x) := \hat{h}^{g_i}(x) \text{ if } x \in g_i,
\end{equation}
which is well-defined if the leaves $g_1, \dots, g_N$ form a partitioning of the input space $\cX$. 

We note that this ``decoupled'' predictor was originally proposed under the assumption that the groups partition the input space \citep{dwork_decoupled_2017}, so we make that assumption here on the leaves. However, we note that our main algorithm in Section \ref{sec:algo_tree} avoids this assumption and works with general hierarchically structured groups.

\begin{theorem}\label{thrm:g_erm}
Let $\cH$ be a hypothesis class and let $\cG$ be a collection of hierarchically structured groups with leaf nodes $g_1, \dots, g_N$ partitioning the input space $\cX.$ Let $\ell(\cdot, \cdot) \in [0, 1]$ be any bounded loss function. Then, with probability $1 - \delta$ over $n$ i.i.d.~examples $(x,y) \sim \cD$ over $\cX \times \cY$, $f$ in Equation \eqref{alg:g-erm} satisfies the multi-group learning property:
\begin{equation}
    L_{\cD}\left(f \mid g \right) - \inf_{h \in \cH} L_{\cD}\left(h \mid g \right) \leq \epsilon_g(n, \delta),
\end{equation}
for any $g \in \cG$ such that $g = \bigcup_{i = 1}^k g_i$ and $\epsilon_g(n, \delta) := 9\sum_{i = 1}^k \frac{\PP[x \in g_i]}{\sum_{j=1}^k \PP[x \in g_j]} \sqrt{\frac{\log(8|\cG||\cH|/\delta)}{n_{g_i}}}.$
\end{theorem}

\topic{This simple algorithm achieves multi-group learning, but it is not optimal.}
In the worst case (e.g., when all disjoint leaf nodes have the same probability mass), the error rate grows with $\sqrt{|\cG|}$, which is far larger than the $\sqrt{\log|\cG|}$ achievable by other methods, as discussed in Section \ref{sec:hierarchical}.

%% file: algos/tree.tex
\topic{In this section, we present \mgltree (Algorithm~\ref{alg:tree}), which achieves the group-wise excess error of $O(\sqrt{\log|\cH||\cG|/n_g})$ with a simple and interpretable decision tree predictor.} One can view \mgltree as an adaptation of the \prepend algorithm of \citet{tosh_simple_2022} and \citet{globus-harris_algorithmic_2022}. Our analysis of \mgltree depends on the hierarchical structure of $\cG$ and is very different from that of \prepend, leading to the improved near-optimal sample complexity rate.

\topic{One unresolved issue stated in \citet{tosh_simple_2022} was to find an algorithm for multi-group learning that is both ``simple'' and enjoys near-optimal sample complexity.}
If we were only concerned with the conditional distribution of a single fixed group, standard statistical learning theory suggests that, for finite $\cH$, ERM is necessary and sufficient to achieve an excess error of
$\epsilon_g = O( \sqrt{\log(|\cH|)/n_g} )$.
\citet{tosh_simple_2022} gives an algorithm (different from \prepend) that achieves the near-optimal excess error of
$\epsilon_g = O( \sqrt{\log(|\cH||\cG|)/n_g})$.
The resulting predictor, however, is a randomized ensemble of $n$ (also) random predictors, one for each training sample, and requires the enumeration of a finite hypothesis class. On the other hand, \prepend outputs a simple, interpretable, and determnistic decision list predictor that avoids enumerating $\cH$. Its excess error, however, is not optimal: 
$\epsilon_g = O( \sqrt[3]{\log(|\cH||\cG|)/(\gamma n_g)} )$,
where $\gamma := \inf_{g \in \cG} \PP_{\cD}[x \in g]$. This trade-off between simplicity and optimal rates was an unresolved issue stated in \citet{tosh_simple_2022}. For hierarchical $\cG$, \mgltree avoids this trade-off between simplicity (and implementability) and near-optimal rates, as stated in Theorems \ref{thrm:tree_finite} and \ref{thrm:tree_infinite}. This allows us to also conduct an empirical evaluation of \mgltree in Section \ref{sec:experiments}.

\begin{theorem}\label{thrm:tree_finite}
Suppose $\cH$ is a finite benchmark hypothesis class and $\cG$ is a collection of hierarchically structured groups. Let $\ell(\cdot, \cdot) \in [0, 1]$ be any bounded loss function. Then, with probability $1 - \delta$ over $n$ i.i.d.~examples $(x,y) \sim \cD$ over $\cX \times \cY$, \mgltree taking input
$$
\epsilon_n(g) := 18\sqrt{\frac{2 \log (|\cG||\cH|) + \log(8/\delta)}{n_g}}
$$
outputs a predictor $f$ satisfying the multi-group learning property with:
\begin{equation}\label{eqn:mgl_tree_mglprop}
    L_{\cD}(f \mid g) - \min_{h \in \cH} L(h \mid g) \leq 2\epsilon_n(g) \quad \text{for all $g \in \cG.$}
\end{equation}
If $\cH$ is infinite, with finite VC dimension $d \geq 1$, then setting \mgltree with
$$
\epsilon_n(g) := 18\sqrt{\frac{2d\log(16|\cG|n/\delta)}{n_g}}
$$
outputs a predictor $f$ satisfying Equation \eqref{eqn:mgl_tree_mglprop}.
\end{theorem}

\begin{algorithm}
\caption{\mgltree}\label{alg:tree}
\begin{algorithmic}[1]
\REQUIRE 
\STATE $S$, a training dataset. 
\STATE Collection of hierarchically structured groups $\mathcal{G} \subseteq 2^{\mathcal{X}}$. 
\STATE Error rates $\epsilon_n(g) \in (0, 1)$ for all $g \in \mathcal{G}$
\ENSURE Decision tree $f: \mathcal{X} \rightarrow \cZ$.
\STATE Order $\mathcal{G}$ into a \textit{hierarchical tree} $\gtree$ (Definition \ref{def:tree}).
\STATE Initialize the root: $f^{\mathcal{X}} := \hat{h}^\mathcal{X}$.
\FOR {each node $g \in \gtree \setminus \{\mathcal{X}\}$ in breadth-first order}
\STATE Compute:
$\mathrm{err}_g := \EE_S[ \ell(f^g(x),y) \mid x \in g ] - \EE_S[ \ell(\hat{h}^g(x),y) \mid x \in g ] - \epsilon_n(g)$.
\IF{$\mathrm{err}_g \geq 0$}
\STATE Set $f^g := \hat{h}^g$.
\ELSE
\STATE Set $f^g := f^{\parent(g)}$, where $\parent(g)$ denotes the parent node of $g.$
\ENDIF
\ENDFOR
\STATE \textbf{return} $f: \mathcal{X} \rightarrow \cZ$, a decision tree predictor.
\end{algorithmic}
\end{algorithm}

\topic{\mgltree gradually refines the ``decision tree'' predictor $f$ by updating its behavior on children groups when a child's predictor is significantly better than its parent's by margin $\epsilon_n(g).$} Each node $g$ in $\gtree$ has an associated predictor $f^g$. To evaluate $f$ on the input $x$, we find the ``deepest'' $g \in \gtree$ that contains $x$ by following a path starting from the root $\cX$ and moving from parent to child whenever the child contains $x$. The prediction $f(x)$ is taken to be $f^g(x).$ 

The main tension that the algorithm resolves is whether to use a predictor for a coarser-grained group (say, the predictor for just $\texttt{R6+}$) vs. a predictor for a finer-grained group (say, the predictor for $\texttt{R6+} \land \texttt{male} \land \texttt{age} < 35$). The finer-grained predictor may be more specific to the finer-grained group, but the finer-grained group necessarily had less samples to train on, and, hence, higher variance. The coarser-grained predictor may be less specific to the finer-grained group (and might, therefore, not pick up on the group-specific signal for the labels), but it had more samples to train on. By the hierarchical structure, coarser predictors are always the parents of finer-grained children. 

The main intuition of the algorithm is to resolve this trade-off by only using the finer-grained group's predictor when it is significantly better than its coarser-grained parent, where the "significance" is $\epsilon_n(g)$ in the Step 7 of the algorithm. Lemma \ref{lemma:fold_main} below establishes a nice "monotonicity" property --- every update operation will never make the overall decision tree violate any error bounds it satisfied earlier in the search, so we always make forward progress towards a decision tree that satisfies the multi-group learning objective.

To prove the correctness of \mgltree, we need to show the main property that the tree's predictions monotonically improve as \mgltree runs. This monotonicity property is similar in spirit to the key idea in the analysis of \prepend \citep{tosh_simple_2022}, but we need to exploit breadth-first search ordering on the hierarchical tree to attain this monotonicity property. We present the statements of the key lemma for proving correctness, and the explicit statement of correctness here in the main body. The full proof is in Appendix \ref{sec:tree_proof}.

\begin{lemma}\label{lemma:fold_main}
Consider any step of Algorithm \ref{alg:tree} where we are considering $g \in \mathcal{G}.$ Let $\fold$ be the decision tree at this step before updating (the state of the tree at line 7). Let $\hat{h}^g \in \argmin_{h \in \mathcal{H}} L_{S}(h \mid g)$ for all $g \in \mathcal{G}.$ Then, for all $x \in g$, $\fold(x) = h^{g'}(x)$ for some $g' \supset g$ already visited by the algorithm.
\end{lemma}

\begin{theorem}[Correctness of \mgltree]\label{thrm:treepend_correctness_main}
Let $\mathcal{G}$ be a hierarchically structured collection of groups, let $S = \{(x_i, y_i)\}_{i = 1}^n$ be $n$ i.i.d. training data drawn from any distribution $\cD$ over $\mathcal{X} \times \mathcal{Y}$, and let $\epsilon_n: \mathcal{G} \rightarrow (0, 1)$ be any error rate function. Then, Algorithm \ref{alg:tree} run on these parameters outputs a predictor $f: \mathcal{X} \rightarrow \mathcal{A}$ satisfying:
\begin{equation*}
    L_{S}(f \mid g) \leq \inf_{h \in \mathcal{H}} L_{S}(h \mid g) + \epsilon_n(g), \quad \text{ for all } g \in \mathcal{G}.
\end{equation*}
\end{theorem}

This predictor also has the bonus that it is easily interpretable. By reading off the nodes of the decision tree predictor, one can interpret the source of model errors for a specific group from its subgroups deeper in the tree. For example, for a hierarchical collection of groups induced by \texttt{race}, \texttt{sex}, and \texttt{age}, one can check if errors on \texttt{race} groups stem from an intersection of $\texttt{race} \land \texttt{sex}$ or the deeper intersection of $\texttt{race} \land \texttt{sex} \land \texttt{age}$ one level down the tree. Because the guarantees from multi-group learning apply to \textit{any} arbitrary bounded loss function, this property may be attractive in, say, fairness scenarios that demand some notion of fairness that can be encoded by a bounded loss function. We leave such applications for future work.

%% file: experiments/intro.tex
\topic{In this section, we perform empirical evaluations of Algorithm \ref{alg:tree} for several classification problems with natural hierarchical group structure.}

%% file: experiments/setup.tex
\topic{\textbf{Models.} For the benchmark hypothesis class $\cH$, we consider the following models:} logistic regression, decision trees (with maximum depth 2, 4, and 8), random forests, and XGBoost. Implementation details are in Appendix \ref{sec:hyperparams}.

\topic{\textbf{Dataset overview.} We conduct our experiments on twelve U.S. Census datasets from the Folktables package of \citet{ding_retiring_2022}.} These datasets include many demographic attributes or other protected features of individuals, and they are large enough to have intersections with substantial data even when subdividing on multiple attributes. 

Nine datasets come from each of the U.S. states of New York (NY), California (CA), and Florida (FL), where we consider the Employment, Income, and Coverage binary classification tasks. The other three nationwide datasets come from concatenating the data from 18 U.S. states, determined by the two most populous states in each U.S. Census designated region of the U.S. \citep{us_census_us_2023} and considering the same three binary classification tasks. The binary classification tasks are as follows:
\begin{itemize}[leftmargin=*,noitemsep,topsep=0pt]
    \item \textbf{Income.} Predict whether employed individuals older than 16 make over \$50,000 USD per year.
    \item \textbf{Employment.} Predict whether individuals older than 16 and younger than 90 are employed.
    \item \textbf{Coverage.} Predict whether individuals younger than 65 (ineligible for Medicare) with income less than \$30,000 USD has coverage from public health insurance.
\end{itemize}

For each of the nine state datasets, we consider two hierarchically structured collections of intersecting groups that arise naturally from subdividing demographic attributes of $\{\texttt{race}, \texttt{sex}, \texttt{age}\}$ or $\{\texttt{race}, \texttt{sex}, \texttt{edu}\}$. For the three nationwide datasets, we divide on $\{\texttt{state}, \texttt{race}, \texttt{age}\}$ or $\{\texttt{state}, \texttt{race}, \texttt{sex}\}.$ In fairness settings, by defining the intersecting groups of interest in this way, we interpolate between a coarse, small collection of pre-defined groups from a single attribute to ensuring fairness for individuals \citep{kearns_preventing_2018}. 
Additional details are provided in Appendix \ref{sec:dataset_details}.

\textbf{Setup.} For each benchmark hypothesis class, we consider four training procedures:
\begin{itemize}[leftmargin=*,noitemsep,topsep=0pt]
    \item \textbf{ERM.} The predictor from minimizing empirical risk over all the data: $f \in \argmin_{f \in \cH} L_S(f).$
    \item \textbf{\prepend.} The decision list predictor output by \prepend. Details are given in Appendix \ref{sec:prepend}.
    \item \textbf{\mgltree.} The decision tree from Algorithm \ref{alg:tree}.
    \item \textbf{Group ERM.} The group-specific ERM predictor $\hat{h}^g \in \argmin_{f \in \cH} L_S(f \mid g)$ (when evaluating performance on group $g$).
\end{itemize}
For all ERM procedures, we optimize the logistic loss, a typical surrogate for the misclassification loss. We note that ``Group ERM'' is a benchmark procedure for if we were \textit{a priori} concerned only about the conditional distribution for group $g \in \cG$ (as discussed in Section \ref{sec:background_mgl}).

We evaluate the multi-group generalization of each predictor by measuring misclassification error on data restricted to each group from a held-out test set of 20\% of the data. We repeat over 10 trials and plot the mean and standard errors.

%% file: experiments/findings.tex
\topic{The main findings of our empirical evaluations are:}
\begin{enumerate}[leftmargin=*,noitemsep,topsep=0pt]
    \item \textbf{\mgltree achieves multi-group learning.}
    \item \textbf{\mgltree performs on par with or better than \prepend on all groups across all datasets.}
    \item \textbf{\mgltree outperforms ``subgroup robust'' baseline methods on certain groups.}
    \item \textbf{Simpler hypothesis classes lead to larger improvements when updating the predictor of \mgltree.}
\end{enumerate}

We note that these findings are consistent across all twelve datasets. Due to space, we relegate most of the plots to Appendix \ref{sec:add_experiments} and present the results on a couple of the datasets (CA Employment and CA Income) in the main body.

\topic{First, we empirically observe that \mgltree indeed achieves multi-group learning.} This supports the upper bound on excess error of Theorem \ref{thrm:tree_finite}. Across all datasets, \mgltree performs on-par with or better than the minimum error achieved by the ERM or Group ERM predictors from the benchmark hypothesis class. For instance, for CA Employment, Figure \ref{fig:lineplot} shows that \mgltree (labeled ``TREE'') has consistently lower error than both ERM and Group ERM for logistic regression, depth 2 decision trees, and random forests. In cases where \mgltree beats ERM by a significant amount, the predictor from \mgltree has identified a high-error subgroup on a specific slice of the population, which may prove useful for further auditing.

Though \mgltree consistently matches or beats ERM and Group ERM on other datasets, the gaps in group generalization are sometimes not as stark; this may suggest that ERM and Group ERM with the benchmark hypothesis class already achieve close to the Bayes optimal risk. Some evidence supports this: on such datasets, even random forest or XGBoost tree ensembles -- which are known to be very effective on tabular datasets \citep{gardner_subgroup_2023} -- have comparable error to simpler benchmark classes.

\begin{figure*}[ht]
\vskip 0.2in
\begin{center}
\centerline{\includegraphics[width=\textwidth]{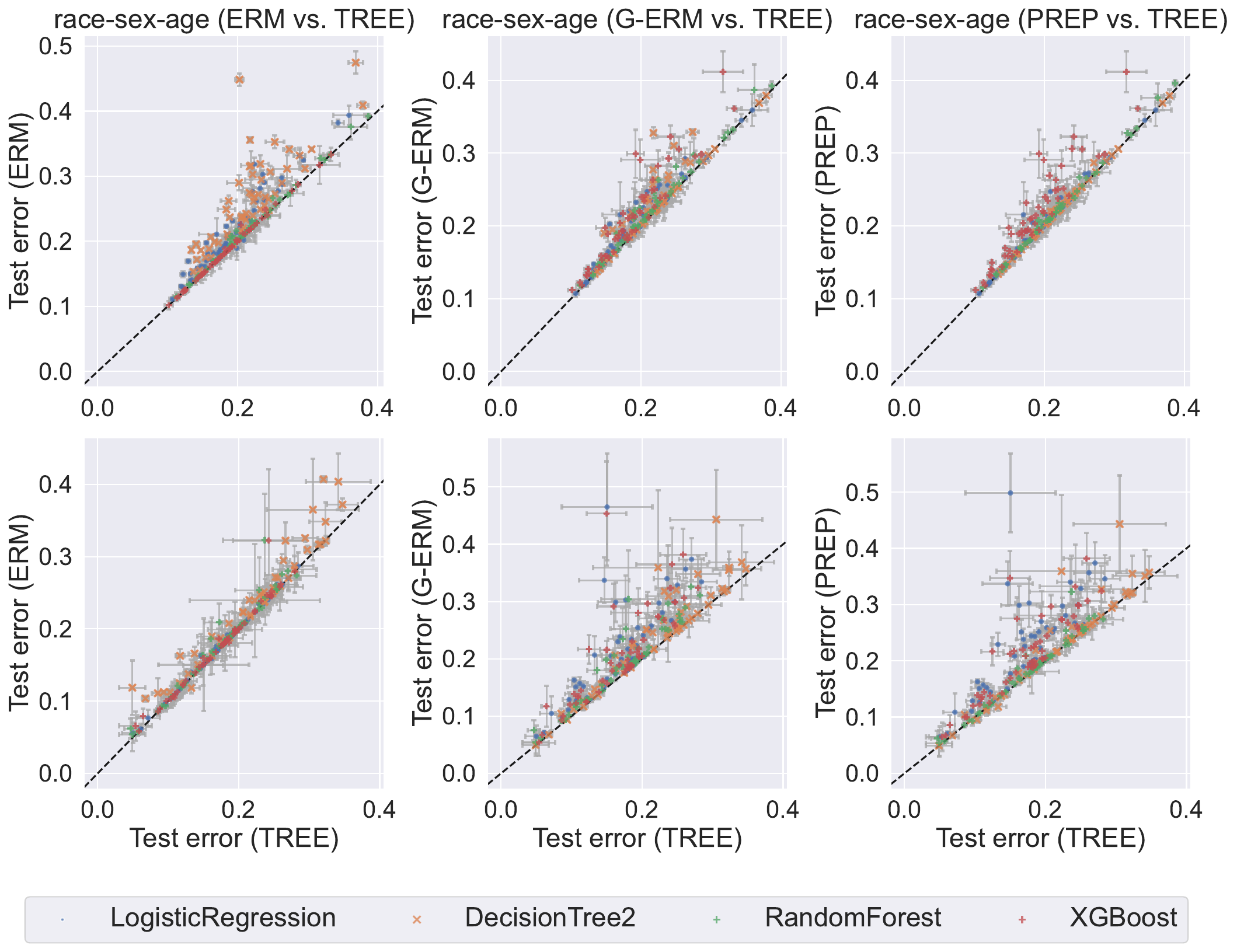}}
\caption{\textbf{Test accuracy on \texttt{race}-\texttt{sex}-\texttt{age} groups for CA Employment (top row) and CA Income (bottom row).} Each point in the plot represents the test error on a specific group. The $y = x$ line represents equal error between our \mgltree and the competing method; points \textit{above} the $y = x$ line are groups where \mgltree exhibits better generalization.} 
\label{fig:lineplot}
\end{center}
\vskip -0.2in
\end{figure*}

\topic{Second, \mgltree performs on par with or better than \prepend across almost all groups across all datasets.} This suggests that, in the case of hierarchially structured groups, \mgltree is better at finding an ordering of group predictors that improves the accuracy of the model. This might not be surprising, as \mgltree directly exploits an inductive bias towards hierarchically structured groups by constructing its decision tree with knowledge of the hierarchial structure. \prepend, on the other hand, constructs a decision list with a possibly arbitrary ordering of groups and group predictors. We observe this in Figure \ref{fig:lineplot} for the CA Income dataset, with further discussion in Section \ref{sec:comparison}.

\topic{Third, \mgltree sometimes outperforms more complex random forest and XGBoost baselines even when using relatively simple benchmark hypothesis classes.} A couple of recent works suggest that simply running ERM with a tree ensemble method such as random forest or XGBoost achieves good multi-group performance \citep{pfohl_comparison_2022, gardner_subgroup_2023}. For example, we observed that \mgltree run with the simple benchmark hypothesis class of linear predictors (specifically, logistic regression models) yields a substantial improvement over the ERM random forest baseline on almost all groups in the CA Employment dataset. See Figure \ref{fig:employment_logreg_bar} in Appendix \ref{sec:add_experiments} for more details. %

\topic{Finally, we also observe that, for simpler model classes, the predictor from \mgltree yields a larger improvement over baseline ERM and Group ERM training compared to more complex model classes.} We see this in both datasets in Figure \ref{fig:lineplot}, where the simplest model class of depth-2 decision trees has the largest gaps from the $y = x$ line. The more complex hypothesis class of linear models from logistic regression yields points closer to the line. The most complex classes of random forests and XGBoost see even less improvement than both logistic regression and depth-2 decision trees from running \mgltree, as evident from how many points hug the $y = x$ line. This suggests that for the predictor from \mgltree, improvements in subgroup generalization compared to the benchmark hypothesis class becomes more difficult as the class becomes more expressive. This behavior is also predicted from the dependence on $\cH$ in our excess error rates in Theorem \ref{thrm:tree_finite}.

%% file: experiments/comparison_small.tex
One main finding in Section \ref{sec:experiment_findings} was that \mgltree generalizes on par with or better than \prepend on all groups across all datasets. Figure \ref{fig:lineplot} shows two examples of this, and Appendix \ref{sec:add_experiments} contains more examples. In this section, we compare the final predictors output by \prepend and \mgltree on two datasets from our empirical study when the benchmark class is logistic regression. We find that, perhaps unsurprisingly, \mgltree achieves better generalization than \prepend because it prefers using the group-specific hypotheses for coarser groups at higher levels in the tree, while \prepend tends to overfit to the finer-grained subgroups at the leaves. We summarize these findings here and refer the reader to Appendix \ref{sec:comparison_full} for the full breakdown of the final predictors from \mgltree and \prepend.

\textbf{CA Employment.} From the upper-right plot in Figure \ref{fig:lineplot}, we observe that, for most groups, \mgltree performs on par with \prepend when both are instantiated with logistic regression (blue dots). There are still several groups where \mgltree performs better, as illustrated by the blue dots above the $y = x$ line, but the vast majority of the groups are on par with \prepend. 

Upon a closer examination of the group-specific predictors used by \prepend or \mgltree, we find that, for CA Employment, \prepend and \mgltree mostly resort to using leaf node predictors. In the case of the \texttt{race}-\texttt{sex}-\texttt{age} groups, the leaves correspond to conjunctions of three attributes such as $\texttt{R6+} \land \texttt{male} \land \texttt{age} < 35$. As such, the final predictor for both algorithms is almost identical, but both show a marked improvement over simply using ERM for most groups, as shown in the upper-left plot in Figure \ref{fig:lineplot}.

\textbf{CA Income.} From the lower-right plot in Figure \ref{fig:lineplot}, on the other hand, we see that \mgltree generalizes better than \prepend on many groups. Upon inspection of the final predictor, we observe that \mgltree employs coarser-grained subgroup predictors further up the tree (e.g. using \texttt{ALL} or \texttt{R1}) on many of the leaves, while, on most groups, \prepend predicts with finer-grained subgroup predictors. 

This suggests that the arbitrary order of the \prepend decision list can overfit to finer-grained subgroups that may not have sufficient samples for the finer-grained predictors to generalize. On the other hand, \mgltree only predicts with a finer-grained subgroup's predictor when its coarser-grained parent has sufficiently high error due to its breadth-first search on the hierarchical tree.

%% file: conclusion.tex
In this paper, we study the multi-group (agnostic PAC) learning framework in the natural case where the groups are hierarchically structured. We give an algorithm that achieves multi-group learning by constructing a decision tree of predictors for each group in the hierarchical collection of groups. This algorithm achieves multi-group learning with a near-optimal group-wise excess error of $O(\sqrt{\log(|\cH||\cG|)/n_g})$ and outputs a simple and deterministic predictor, addressing an issue posed by \citep{tosh_simple_2022} for the case of hierarchically structured groups. We then conduct an extensive empirical evaluation of our algorithm and find that it achieves attractive generalization properties on real datasets with hierarchical group structure, as the theory suggests.

%% file: appendix.tex
\section{Technical Lemmas}
The main technical lemma we will employ in our proofs is a result from \citet{tosh_simple_2022} that establishes the uniform convergence of risks conditional on group membership. This generalizes a lemma from \citet{balsubramani_adaptive_2019} that gives a rate for the uniform convergence of empirical conditional measures.

\begin{lemma}[Theorem 1 from \citet{tosh_simple_2022}]\label{lemma:group_unifconv}
Let $\cH$ be a hypothesis class, let $\cG$ be a collection of groups, and let $\ell: \cZ \times \cY \rightarrow [0, 1]$ be a bounded loss function. Given an i.i.d.~sample of size $n$ drawn from a distribution $\cD$ over $\cX \times \cY$, with probability at least $1 - \delta$,
\begin{equation*}
    \left| L_{\cD}(h \mid g) - L_S(h \mid g) \right| \leq 9\sqrt{\frac{2 \log( S_{2n}(\overline{\cH}) S_{2n}(\cG)) + \log(8/\delta)}{n_g}},
\end{equation*}
for all $h \in \cH$ and $g \in \cG$, where $S_{k}(C)$ is the $k$th shattering coefficient of the collection $C$ of sets, and $\overline{\cH} := \{ (x, y) \mapsto \indic{\ell(h(x), y) \geq t} : h \in \cH, t \in \RR \}$ is the set of (boolean) functions constructed from composing the hypothesis class $\cH$ with the loss $\ell(\cdot, \cdot)$.
\end{lemma}

Two immediate consequences of Lemma \ref{lemma:group_unifconv} are Lemmas \ref{lemma:group_unifconv_finite} and \ref{lemma:group_unifconv_vc}. 

\begin{lemma}\label{lemma:group_unifconv_finite}
In the setting of Lemma \ref{lemma:group_unifconv}, if $\cH$ and $\cG$ are finite, then
\begin{equation*}
    \left| L_{\cD}(h \mid g) - L_S(h \mid g) \right| \leq 9\sqrt{\frac{2 \log(|\cH||\cG|) + \log(8/\delta)}{n_g}}
\end{equation*}
for all $h \in \cH$ and all $g \in \cG.$ 
\begin{proof}
For a finite set, the shattering coefficient is bounded by the size of the set, so $S_{2n}(\overline{\cH}) \leq |\cH|$ and $S_{2n}(\cG) \leq |\cG|.$
\end{proof}
\end{lemma}

\begin{lemma}\label{lemma:group_unifconv_vc}
In the setting of Lemma \ref{lemma:group_unifconv}, if $\cH$ has VC dimension $d > 0$, then
\begin{equation*}
    \left| L_{\cD}(h \mid g) - L_S(h \mid g) \right| \leq 9\sqrt{\frac{2d \log(2|\cG|n) + \log(8/\delta)}{n_g}}.
\end{equation*}
\begin{proof}
By Sauer's lemma, $S_{2n}(\overline{\cH}) \leq \sum_{i = 0}^d \binom{n}{i} \leq (2n)^d$.
\end{proof}
\end{lemma}

We will also need the following two lemmas for decomposing group conditional risks for hierarchically structured $\cG.$

\begin{lemma}\label{lemma:disjoint_decomp_emp}
Let $g_1, \dots g_N$ be $N$ disjoint groups, and let $S = \{(x_i, y_i)\}_{i=1}^n$ be a dataset of $n$ i.i.d. examples. Let $n_{\cup g}$ be the number of examples in $\bigcup_{k = 1}^N g_k$, and let $n_{g_k}$ be the number of examples in $g_k$. Then, the group conditional empirical risk for $\bigcup_{k = 1}^K g_k$ decomposes as follows:
\begin{equation}
    L_S\left(f \bigmid \bigcup_{k = 1}^N g_k\right) = \sum_{k = 1}^N \frac{n_{g_k}}{n_{\cup g}} L_S(f \mid g_k).
\end{equation}
\end{lemma}
\begin{proof}
By definition of group conditional empirical risk,
$$L_S\left(f \bigmid \bigcup_{k = 1}^N g_k\right) = \frac{1}{n} \sum_{i = 1}^n \indic{x_i \in \bigcup_{k = 1}^N g_k} \ell(f(x_i), y_i).$$
We note that $g_1, \dots, g_N$ are disjoint, so:
\begin{align*}
    \frac{1}{n} \sum_{i = 1}^n \indic{x_i \in \bigcup_{k = 1}^N g_k} \ell(f(x_i), y_i) &=  \frac{1}{n_{\cup g}} \sum_{i = 1}^n \left(\sum_{k=1}^N g_k(x_i) \ell(f(x_i), y_i)\right) \\
    &= \frac{1}{n_{\cup g}} \sum_{k=1}^N \left(\sum_{i = 1}^n  g_k(x_i) \ell(f(x_i), y_i)\right) \\
    &= \frac{1}{n_{\cup g}} \sum_{k=1}^N n_{g_k} L_S\left(f \mid g_k\right) = \sum_{k = 1}^N \frac{n_{g_k}}{n_{\cup g}} L_S\left(f \mid g_k\right).
\end{align*}
In the first equality, the $g_k(x)$ are just indicators for disjoint groups, so that immediately follows from boolean algebra on $\indic{x \in \bigcup_{k = 1}^N g_k}$. The second equality just switches order of summation, and the third is the definition of group conditional empirical risk again.
\end{proof}

\begin{lemma} \label{lemma:disjoint_decomp}
Let $g_1, \dots g_N$ be $N$ disjoint groups. Then, the group conditional risk for $\bigcup_{k = 1}^N g_k$ decomposes:
\begin{equation}
    L_{\cD}\left(f \bigmid \bigcup_{k = 1}^N g_k\right) = \sum_{k = 1}^N \frac{\PP[x \in g_k]}{\sum_{j = 1}^K \PP[x \in g_j]} L_{\cD}\left(f \mid g_k\right).
\end{equation}
\end{lemma}
\begin{proof}
By definition of group conditional risk,
$$L_{\cD}(f \mid g) = \EE\left[\ell(f(x), y) \mid x \in g\right].$$
We first claim that $L_{\cD}(f \mid g) = \frac{1}{\PP[x \in g]}\EE\left[g(x) \ell(f(x), y)\right].$ This follows from:
\begin{align*}
    \EE\left[g(x) \ell(f(x), y)\right] 
    &= \EE\left[\EE\left[g(x) \ell(f(x), y) \mid x \in g\right]\right] \\
    &= \EE\left[g(x) \EE\left[\ell(f(x), y) \mid x \in g\right]\right] \\
    &= \PP[x \in g] \EE\left[\ell(f(x), y) \mid x \in g\right] \\
    &= \PP[x \in g ] L_{\cD}(f \mid g).
\end{align*}
Using this fact, we can re-write the group conditional risk as:
$$L_{\cD}\left(f \bigmid \bigcup_{k = 1}^N g_k\right) = \frac{1}{\PP\left[x \in \bigcup_{j = 1}^N g_j\right]} \EE\left[\indic{x \in \bigcup_{k = 1}^N g_k}\ell(f(x), y)\right].$$
Because $g_1, \dots, g_N$ are disjoint, we can use additivity of $\PP(\cdot)$:
\begin{align*}
    &= \frac{1}{\sum_{j=1}^N \PP\left[x \in g_j\right]} \EE\left[\sum_{k = 1}^N g_k(x) \ell(f(x), y)\right] \\
    &= \frac{1}{\sum_{j=1}^N \PP[x \in g_j]} \sum_{k = 1}^N \EE\left[g_k(x) \ell(f(x), y)\right]
\end{align*}
Using the same fact that $\PP\left[x \in g_k\right] L_{\cD}\left(f \mid g_k\right) = \EE\left[g_k(x) \ell(f(x), y)\right]$, we get the desired result:
\begin{align*}
    &= \sum_{k = 1}^N \frac{\PP\left[x \in g_k\right]}{\sum_{j = 1}^N \PP\left[x \in g_j\right]} L_{\cD}(f \mid g_k).
\end{align*}
\end{proof}

\section{Proof of Theorem \ref{thrm:g_erm}}
\label{sec:g_erm_proof}
We give the full version of Theorem \ref{thrm:g_erm} here, including the case where $\cH$ has finite VC dimension.

\begin{theorem}\label{thrm:g_erm_vc}
Let $\cH$ be a hypothesis class with VC dimension $d > 0$ and let $\cG$ be a collection of hierarchically structured groups with leaf nodes $g_1, \dots, g_N$ partitioning the input space $\cX.$ Let $\ell(\cdot, \cdot) \in [0, 1]$ be any bounded loss function. Then, with probability $1 - \delta$ over $n$ i.i.d.~examples $(x,y) \sim \cD$ over $\cX \times \cY$, $f$ in Equation \eqref{alg:g-erm} satisfies the multi-group learning property with
\begin{equation}
    L_{\cD}\left(f \mid g \right) - \inf_{h \in \cH} L_{\cD}\left(h \mid g \right) \leq 9\sum_{i = 1}^k \frac{\PP[x \in g_i]}{\sum_{j=1}^k \PP[x \in g_j]} \sqrt{\frac{2d\log(16|\cG|n/\delta)}{n_{g_i}}}
\end{equation}
for any $g \in \cG$ such that $g = \bigcup_{i = 1}^k g_i$. For finite $\cH$,
\begin{equation}
    L_{\cD}\left(f \mid g \right) - \inf_{h \in \cH} L_{\cD}\left(h \mid g \right) \leq 9\sum_{i = 1}^k \frac{\PP[x \in g_i]}{\sum_{j=1}^k \PP[x \in g_j]} \sqrt{\frac{\log(8|\cG||\cH|/\delta)}{n_{g_i}}}.
\end{equation}
\end{theorem}

\begin{proof}
We first show the proof for $\cH$ with finite VC dimension $d.$

For each disjoint atomic group $g_i$, for $i \in [N],$ the behavior of $f$ is simply to use $\hat{h}_{g_i}.$ We know from uniform convergence of conditional risks (Lemma \ref{lemma:group_unifconv_vc}) that
\begin{equation}\label{eq:indiv_erm}
L_\mathcal{D}(f \mid g_i) = L_\mathcal{D}(\hat{h}_{g_i} \mid g_i) \leq \inf_{h \in \mathcal{H}} L_\mathcal{D}(h \mid g_i) + 9\sqrt{\frac{2d\log(16|\cG|n/\delta)}{n_{g_i}}},
\end{equation}
where $n_{g_i}$ is the number of examples, all from group $g_i.$ Consider the union of $k$ of the $N$ disjoint atomic groups, $\bigcup_{i = 1}^k g_{i}$ (without loss of generality for index $i$, let $g_i$ be any one of the $N$ groups). By Lemma \ref{lemma:disjoint_decomp}, 
$$
L_\mathcal{D}\left(f \bigmid \bigcup_{i = 1}^k g_i\right) = \sum_{i = 1}^k \frac{\PP(g_i)}{\sum_{j=1}^k \PP(g_j)} L_\mathcal{D}(f \mid g_i).
$$
By Equation \eqref{eq:indiv_erm} and the definition of the predictor $f$ from Section \ref{sec:group_erm},
\begin{align*}
    L_\mathcal{D}\left(f \bigmid \bigcup_{i = 1}^k g_i\right) &= \sum_{i = 1}^k \frac{\PP(g_i)}{\sum_{j=1}^k \PP(g_j)} L_\mathcal{D}(f \mid g_i) = \sum_{i = 1}^k \frac{\PP(g_i)}{\sum_{j=1}^k \PP(g_j)} L_\mathcal{D}(\hat{h}_{g_i} \mid g_i) \\
    &\leq \sum_{i = 1}^k \frac{\PP(g_i)}{\sum_{j=1}^k \PP(g_j)} \left(L_\mathcal{D}(h^*_{g_i} \mid g_i) + 9\sqrt{\frac{2d\log(16|\cG|n/\delta)}{n_{g_i}}} \right),
\end{align*}
where $h^*_{g_i} \in \argmin_{h \in \mathcal{H}} L_\mathcal{D}(h \mid g_i).$ Now, let $h^*_{\cup g_i} \in \argmin_{h^* \in \mathcal{H}} L_\mathcal{D}(h^* \mid \bigcup_{i = 1}^k g_i).$ Because each $h^*_{g_i}$ is optimal for their conditional risks on their respective group $g_i$, 
\begin{align*}
    &\leq \sum_{i = 1}^k \frac{\PP(g_i)}{\sum_{j=1}^k \PP(g_j)} \left(L_\mathcal{D}(h^*_{\cup g_i} \mid g_i) + 9\sqrt{\frac{2d\log(16|\cG|n/\delta)}{n_{g_i}}} \right) \\
    &= \sum_{i = 1}^k \frac{\PP(g_i)}{\sum_{j=1}^k \PP(g_j)} L_\mathcal{D}(h^*_{\cup g_i} \mid g_i) + \sum_{i = 1}^k \frac{\PP(g_i)}{\sum_{j=1}^k \PP(g_j)} 9\sqrt{\frac{2d\log(16|\cG|n/\delta)}{n_{g_i}}} \\
    &= L_\mathcal{D}\left(h^*_{\cup g_i} \bigmid \bigcup_{i=1}^k g_i \right) + \sum_{i = 1}^k \frac{\PP(g_i)}{\sum_{j=1}^k \PP(g_j)} 9\sqrt{\frac{2d\log(16|\cG|n/\delta)}{n_{g_i}}},
\end{align*}
whcih is our our result. The final equality is from applying Lemma \ref{lemma:disjoint_decomp} again to combine the first term.

The proof for finite $\mathcal{H}$ is identical, but use the finite $\mathcal{H}$ uniform convergence bound for conditional risks in Lemma \ref{lemma:group_unifconv_finite} instead of Equation \eqref{eq:indiv_erm}.
\end{proof}

\section{Proof of Theorems \ref{thrm:tree_finite} and \ref{thrm:tree_infinite}}\label{sec:tree_proof}

We now present the proof of \mgltree (Algorithm \ref{alg:tree}). This algorithm outputs a final decision tree predictor, $f: \mathcal{X} \rightarrow \mathcal{Z}$. Each node of the decision tree is a group $g \in \mathcal{G}$ with an associated \textit{working predictor} $f^g: \mathcal{X} \rightarrow \mathcal{Z}.$ The goal of the algorithm is to determine a good working predictor for each node of the tree. For each group $g \in \mathcal{G}$ and an i.i.d.~sample $S$, we'll denote $\hat{h}^g$ to be the ERM minimizer of group conditional empirical risk:
\begin{equation*}
    \hat{h}^g := \argmin_{h \in \mathcal{H}} L_{S}(h \mid g)
\end{equation*}
We construct the decision tree as follows. First, we generate the hierarchical tree $\gtree$ (Definition \ref{def:tree}) from the hierarchically structured collection of groups $\mathcal{G}$. For simplicity, the root is $\mathcal{X}$. In this tree, every node $g$ is a subset of its ancestors. We begin by initializing the root node's working predictor $f^{\mathcal{X}} := \hat{h}^\mathcal{X},$ the ERM minimizer of group conditional empirical risk for all of $\mathcal{X}$ --- this is just standard unconditional empirical risk $L_{S}(h).$ 

To assign working predictors $f^g$ to each node, we start from the root of the tree where $g = \mathcal{X}$ and visit all $|\mathcal{G}|$ nodes in the tree in breadth-first order. Let $f^{\parent(g)}$ denote the working predictor for the parent of node $g$. The main idea is to set the working predictor at node $g$ to $f^g := \hat{h}^g$ only if its parent is insufficient for achieving the desired margin of error $\epsilon_n(g)$. Otherwise, node $g$ inherits its working predictor from its parent: $f^g := f^{\parent(g)}$. To show that Algorithm \ref{alg:tree} is correct, the key is to prove a ``monotonicity'' property: at each update operation, the algorithm does not violate any error bounds for groups further up the tree.

Let $\fold$ denote the state of the decision tree \textit{before} an update operation. In Algorithm \ref{alg:tree}, this corresponds to the state of the decision tree at line 7 in each iteration of the main BFS loop. We will analyze $\fold$ in the proofs of Theorem \ref{thrm:treepend_correctness} and our main Theorems \ref{thrm:tree_finite} and \ref{thrm:tree_infinite}.

We state one more obvious lemma concerning the behavior of $f$ when we choose to inherit the parent's working predictor, $f^{\parent(g)}$. When we do this, $f$ is functionally equivalent to $\fold$ on group $g$ and all the nodes on the path from $g$ back up to the root. This is referred to as \ref{lemma:fold_main} in the main body, and is restated here as Lemma \ref{lemma:fold} for convenience.
\begin{lemma}[Behavior of $\fold$]\label{lemma:fold}
Consider any step of Algorithm \ref{alg:tree} where we are considering $g \in \mathcal{G}.$ Let $\fold$ be the decision tree at this step before updating (the state of the tree at line 7). Let $\hat{h}^g \in \argmin_{h \in \mathcal{H}} L_{S}(h \mid g)$ for all $g \in \mathcal{G}.$ Then, for all $x \in g$, $\fold(x) = h^{g'}(x)$ for some $g' \supset g$ already visited by the algorithm.
\end{lemma}
\begin{proof}
This just follows by induction. For the first step of Algorithm \ref{alg:tree}, $\fold$ is simply $h \in \argmin_{h \in \mathcal{H}} L_{S}(h)$, the ERM predictor over all of $\mathcal{X}$. Of course, $g \subset \mathcal{X}$ for any $g.$ Assume the lemma for all $g' \in \mathcal{G}'$, the set of already visited nodes. Suppose we are on step $g \in \mathcal{G}$ in our BFS. Then, if $x \in g,$ by hierarchical structure and BFS, $x \in g'$ for some $g \subset g'$ because we've visited all parents before their children. Therefore, $\fold$ uses $h^{g'}$ for some $g' \supset g.$ 
\end{proof}

Lemma \ref{lemma:fold} allows us to apply Lemma \ref{lemma:group_unifconv}, our conditional uniform convergence bound, on $\fold$, as it is functionally equivalent to some $h \in \mathcal{H},$ our benchmark hypothesis class.

The key to Algorithm \ref{alg:tree} is that each update operation at any node $g$ does not make $f$ violate the error bounds it satisfied further up the tree. We observe that, at an update iteration (when $\mathrm{err}_g \geq 0$), either we accept the (conditional) ERM predictor $f^g := \hat{h}^g$ or we inherit the parent's working predictor $f^g := f^{\parent(g)}$. See Figure \ref{fig:treepend_pf}.

\begin{figure}[H]
    \centering
	\includegraphics[width=0.5\columnwidth]{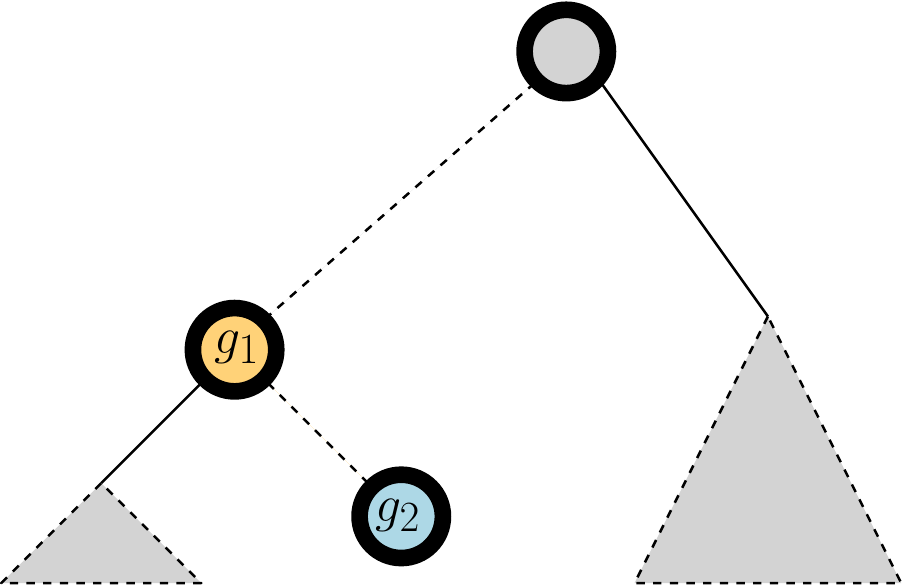}
	\caption{Let $g_1$, the yellow node, be a group that Algorithm \ref{alg:tree} has already seen. Suppose $f$ updates on $g_2.$ We see that $g_1$ is on the path from $g_2$ to the root. We need to show that the inequality for $g_1$ is not violated after the update.}
	\label{fig:treepend_pf}
\end{figure}

The next theorem, Theorem \ref{thrm:treepend_correctness}, establishes the correctness of the algorithm when it terminates. This is stated in the main body as Theorem \ref{thrm:treepend_correctness_main}, and we state it here as Theorem \ref{thrm:treepend_correctness}.

\begin{theorem}[Correctness of \mgltree]\label{thrm:treepend_correctness}
Let $\mathcal{G}$ be a hierarchically structured collection of groups, let $S = \{(x_i, y_i)\}_{i = 1}^n$ be $n$ i.i.d. training data drawn from any distribution $\cD$ over $\mathcal{X} \times \mathcal{Y}$, and let $\epsilon_n: \mathcal{G} \rightarrow (0, 1)$ be any error rate function. Then, Algorithm \ref{alg:tree} run on these parameters outputs a predictor $f: \mathcal{X} \rightarrow \mathcal{A}$ satisfying:
\begin{equation}\label{eqn:treepend_ineq}
    L_{S}(f \mid g) \leq \inf_{h \in \mathcal{H}} L_{S}(h \mid g) + \epsilon_n(g), \quad \text{ for all } g \in \mathcal{G}.
\end{equation}
\end{theorem}
\begin{proof}
Consider any $g^* \in \mathcal{G}$, which corresponds to a node in the decision tree, $f$. We analyze the step of the algorithm's breadth-first search concerned with $g^*$ and argue that this step does not violate any inequalities of the form \eqref{eqn:treepend_ineq} satisfied up until this step. This is sufficient to prove correctness because $g^*$ is an arbitrary node, and the tree traversal will visit every node, so \eqref{eqn:treepend_ineq} will be satisfied for all nodes in the tree.

On the current step for $g^*$, the algorithm can either decide to update or not. If the if-condition is not satisfied and it does not update, then we keep $f := f_\textrm{old}$ from the previous round, and we are done. All inequalities previously satisfied must continue to be satisfied because $f$ did not change.

Suppose, then, that the algorithm \textit{did} update. Let $P_{g^*} := \{\hat{g}_1, \dots, \hat{g}_k\}$ be the set of nodes on the path from $g^*$ to the root of the tree, including the root. Then, for all nodes $g \not \in P_{g^*},$ $f$ continues to satisfy $\eqref{eqn:treepend_ineq}$ because $g \cap g^* = \emptyset$, so for $x \in g,$ $f(x) = \fold(x)$, as before. This is due to the hierarchical structure --- any node \textit{not} on the path from $g^*$ back up to the root must be disjoint from $g^*.$

Now, consider our final case: any $\hat{g}_j \in P_{g^*}$ for $j \in [k],$ a node back up to the root from $g^*.$ Again, we are in the case where we updated, so $f$ has changed. By the hierarchical structure, if $\hat{g}_j$ is further up the tree from $g^*$, then $g^* \subset \hat{g}_j$. We need to show that $L_{S}(f \mid \hat{g}_j) \leq L_{S}(f_\mathrm{old} \mid \hat{g}_j).$ Denote $\overline{g}$ as the complement of $g$, and apply Lemma \ref{lemma:disjoint_decomp_emp}:
\begin{align}
    L_{S}(f \mid \hat{g}_j) &= L_{S}(f \mid (\hat{g}_j \cap g^*) \cup (\hat{g}_j \cap \overline{g^*})) \nonumber \\
    &= L_{S}(f \mid \hat{g}_j \cap g^*)\frac{\Pr[x \in \hat{g}_j \cap g^*]}{\Pr[x \in \hat{g}_j]} + L_{S}(f \mid \hat{g}_j \cap \overline{g^*})\frac{\Pr[x \in (\hat{g}_j \cap \overline{g^*})]}{\Pr[x \in \hat{g}_j]} \nonumber \\
    &= L_{S}(f \mid g^*) \Pr[x \in g^* \mid x \in \hat{g}_j] + L_{S}(f \mid \hat{g}_j \cap \overline{g^*})\Pr[x \in \overline{g^*} \mid x \in \hat{g}_j] \nonumber \\
    &= L_{S}\Pr[x \in g^* \mid x \in \hat{g}_j] + L_{S}(f_\mathrm{old}|\hat{g}_j \cap \overline{g^*})\Pr[x \in \overline{g^*} \mid x \in \hat{g}_j]. \label{eq:hatf_decomp}
\end{align}
The last equality is a result of how the $f$ operates as a decision tree. Observe that, on $x \in g^*$, our updated decision tree $f$ uses $h^* \in \argmin_{h \in \mathcal{H}} L_{S}(h \mid g^*).$ On $x \in \hat{g}_j \cap \overline{g^*},$ our decision tree $f$ simply operates as it did before the update: $f(x) = f_{\mathrm{old}}(x).$ By definition of $h^*$ as ERM predictor,
\begin{equation}\label{eq:erm_hstar}
    L_{S}(h^* \mid g^*) \leq L_{S}(f_\mathrm{old} \mid g^*),
\end{equation}
so combining \eqref{eq:hatf_decomp} and \eqref{eq:erm_hstar},
\begin{align}
    L_{S}(f \mid \hat{g}_j) 
    &= L_{S}(h^* \mid g^*)\Pr[x \in g^* \mid x \in \hat{g}_j] + L_{S}(f_\mathrm{old} \mid \hat{g}_j \cap \overline{g^*})\Pr[x \in \overline{g^*} \mid x \in \hat{g}_j]\nonumber \\
    &\leq L_{S}(f_\mathrm{old} \mid g^*)\Pr[x \in g^* \mid x \in \hat{g}_j] + L_{S}(f_\mathrm{old} \mid \hat{g}_j \cap \overline{g^*})\Pr[x \in \overline{g^*} \mid x \in \hat{g}_j] \nonumber \\
    &= L_{S}(f_\mathrm{old} \mid \hat{g}_j) \leq \min_{h \in \mathcal{H}} L_{S}(h \mid \hat{g}_j) + \epsilon_n(\hat{g}_j), \label{eq:treepend_correct}
\end{align}
where the final equality in \eqref{eq:treepend_correct} follows from another application of Lemma \ref{lemma:disjoint_decomp_emp}. Because 
\begin{equation}\label{eq:treepend_correct_final}
    L_{S}(f \mid \hat{g}_j) \leq \min_{h \in \mathcal{H}} L_{S}(h \mid \hat{g}_j) + \epsilon_n(\hat{g}_j),
\end{equation}
where $f$ is the updated decision list, we see that $f$ does not violate any of the inequalities for the nodes $\hat{g}_j$ in the path up to the root, finishing our proof.
\end{proof}

We state the result for $\cH$ with finite VC dimension to accompany Theorem \ref{thrm:tree_finite} here.

\begin{theorem}\label{thrm:tree_infinite}
For $\cH$ with VC dimension bounded by $d > 0$ and the other conditions of Theorem \ref{thrm:tree_finite}, running Algorithm \ref{alg:tree} with input
$$
\epsilon_n(g) := 18\sqrt{\frac{2d\log(16|\cG|n/\delta)}{n_g}}
$$
outputs a predictor $f$ that achieves
\begin{equation}
    L_{\cD}(f \mid g) \leq \inf_{h \in \cH} L(h \mid g) + 36 \sqrt{\frac{2d\log(16|\cG|n/\delta)}{n_g}} \quad \text{for all } g \in \cG.
\end{equation}
\end{theorem}

We may now prove Theorem \ref{thrm:tree_finite} and Theorem \ref{thrm:tree_infinite} for the multi-group learning guarantee of Algorithm \ref{alg:tree}, \mgltree.

\begin{proof}
We will show this by induction on each iteration (visited group $g$) of Algorithm \ref{alg:tree} for finite $\cG$ and $\cH$. Condition on the event that we drew our i.i.d. dataset of size $n$ and Lemma \ref{lemma:group_unifconv_finite} holds. For ease of notation, we will denote
$$UC(g) := 9\sqrt{\frac{2\log(|\mathcal{G}||\mathcal{H}|) + \log(8/\delta)}{n_g}},$$
so, for all groups $g \in \mathcal{G}$, we have, uniformly over $h \in \mathcal{H}$:
\begin{equation*}
    |L_{\cD}(h \mid g) - L_{S}(h \mid g)| \leq UC(g).
\end{equation*}
Note that we choose $\epsilon_n(g) = 2UC(g)$. Our goal will thus be to show that, for all $g \in \mathcal{G}$, the decision tree $f$ satisfies
\begin{equation}\label{eq:induct_goal}
    L_{\cD}(f \mid g) \leq \min_{h \in \mathcal{H}} L_{\cD}(h \mid g) + 4UC(g).
\end{equation}
For the base case, we just need to show that our starting decision tree $f$, where each node is initialized with $h_0 \in \argmin_{h \in \mathcal{H}} L_{S}(h)$, satisfies inequality \eqref{eq:induct_goal} for $g = \mathcal{X}.$ Note that this is just standard unconditional risk. This is immediately true from standard uniform convergence and $f$ using the ERM predictor $h_0$, so
\begin{equation*}
    L_{\cD}(f) \leq \min_{h \in \mathcal{H}} L_{\cD}(h) + 4UC(g).
\end{equation*}
For the inductive hypothesis, assume that we are on the step of the BFS in Algorithm \ref{alg:tree} concerned with a group $g \in \mathcal{G}.$ Denote $\mathcal{G}'$ as the nodes that we already visited in our BFS, and $\fold$ as the decision list \textit{before} the possible update, as before. Then,
\begin{equation}\label{eq:induct_hyp}
    L_{\cD}(\fold \mid g') \leq \min_{h \in \mathcal{H}} L_{\cD}(h \mid g') + 4UC(g')
\end{equation}
holds for all $g' \in \mathcal{G}'$.

To prove the induction, we aim to show that \eqref{eq:induct_goal} is true for our current iteration's node $g$ as well, regardless of whether we updated $f$ or not. Let $f$ be the decision list \textit{after} the update step of Algorithm \ref{alg:tree}, and let $\hat{h}^g \in \argmin_{h \in \mathcal{H}} L_{S}(h \mid g).$ We want to show, for all $h \in \mathcal{H}$:
\begin{align}
    &L_{\cD}(f \mid g) - L_{\cD}(h \mid g) \leq 4UC(g) \label{eq:g1}\\
    &L_{\cD}(f \mid g') - L_{\cD}(h \mid g') \leq 4UC(g'), \text{  for all } g' \in \mathcal{G'}. \label{eq:g2}
\end{align}
So, showing \eqref{eq:g1} and \eqref{eq:g2} are our goals for each of our two cases: whether we update or not. These two cases depend on whether $L_{S}(\fold \mid g) - L_{S}(\hat{h}^g \mid g) \leq \epsilon_n(g)$ or $L_{S}(\fold \mid g) - L_{S}(\hat{h}^g \mid g) > \epsilon_n(g),$ the central comparison in our algorithm.

Suppose we are in the first case, when we \textit{do not} update. Because $\mathrm{err}_g \leq 0$, we have $L_{S}(\fold \mid g) - L_{S}(\hat{h}^g \mid g) \leq \epsilon_n(g)$. Then, because we do not update, $f = \fold$, so \eqref{eq:g2} is immediately fulfilled because $f$ is functionally equivalent to $\fold,$ which, by induction satisfied the inequality already for all $g' \in \mathcal{G}'.$ It suffices to show \eqref{eq:g1} for this case. Fix any $h \in \mathcal{H}.$ First, with Lemma \ref{lemma:fold}, we can apply conditional uniform convergence on $f = \fold$. Then,
\begin{align*}
    L_{\cD}(f \mid g) - L_{\cD}(h \mid g) &= L_{\cD}(\fold \mid g) - L_{\cD}(h \mid g) \\
    &\leq L_{S}(\fold \mid g) - L_{\cD}(h \mid g) + UC(g) \\
    &\leq L_{S}(\fold \mid g) - L_{S}(h \mid g) + 2UC(g) \\
    &\leq L_{S}(\fold \mid g) - L_{S}(\hat{h}^g \mid g) + 2UC(g) \\
    &\leq \epsilon_n(g) + 2UC(g) = 4UC(g).
\end{align*}
The first inequality is from Lemma \ref{lemma:fold} and Lemma \ref{lemma:group_unifconv_finite}. The third inequality is from the fact that $\hat{h}^g$ is the optimal ERM predictor conditioned on $x \in g.$ This proves \eqref{eq:g1} for the first case where we do not update $f$.

Suppose we are in the second case. In this case, we \textit{do} update and we have that $L_{S}(\fold \mid g) - L_{S}(\hat{h}^g \mid g) > \epsilon_n(g)$. In this case, $\fold$ is the decision tree before the update, and $f$ is the tree after the update. Its \textit{working predictor} has been updated to $\hat{h}^g$. Immediately, we have $L_{S}(f \mid g) - L_{S}(\hat{h}^g \mid g) = 0$ for the current node $g$, so, for any $h \in \mathcal{H}$,
\begin{align*}
    L_{\cD}(f \mid g) - L_{\cD}(h \mid g) &= L_{\cD}(\hat{h}^g \mid g) - L_{\cD}(h \mid g) \\
    &\leq L_{S}(\hat{h}^g \mid g) - L_{S}(h \mid g) + 2UC(g) \\
    &\leq 2UC(g) \leq 4UC(g).
\end{align*}
The first equality is because $f$ is functionally equivalent to $\hat{h}^g$ on all $x \in g$, and the first inequality comes from applying Lemma \ref{lemma:group_unifconv_finite} twice to get sample risk for $h$ and $\hat{h}^g.$ This proves \eqref{eq:g1}. It suffices to prove \eqref{eq:g2}. Consider any $g' \in \mathcal{G}'$, the set of already visited groups. There are two types nodes in $G'$: $g'_p$, the nodes on the path back up to the root from $g$, and $g'_{np}$, the nodes \textit{not} on the path back up to the root from $g.$

For any $g'_{np},$ by hierarchical structure, $g'_{np} \cap g = \emptyset$. So, for all $x \in g'_{np}$, the predictor just outputs as it did before the update: $f = \fold$. Then, for any $h \in \mathcal{H}$, we maintain the same guarantee we had before, fulfilling \eqref{eq:g2} for all $g'_{np}$:

\begin{equation*}
    L_{\cD}(f \mid g_{np}') - L_{\cD}(h \mid g'_{np}) = L_{\cD}(\fold \mid g_{np}') - L_{\cD}(h \mid g'_{np}) \leq 4UC(g'_{np}).
\end{equation*}

To finish the proof, it suffices to show that \eqref{eq:g2} is still fulfilled for all $g'_p$, the nodes on a path back up to the root from $g.$ Again, fix some $h \in \mathcal{H}.$ Using Lemma \ref{lemma:disjoint_decomp}:
\begin{align*}
    L_{\cD}(f \mid g_p') &= L_{\cD}(f \mid (g_p' \cap g) \cup (g_p' \setminus g)) \\
    &= \frac{\Pr[x \in (g_p' \cap g)]}{\Pr[x \in g_p']} L_{\cD}(f \mid g_p' \cap g) + \frac{\Pr[x \in (g_p' \setminus g)]}{\Pr[x \in g_p']} L_{\cD}(f \mid g_p' \setminus g) \\
    &= \frac{\Pr[x \in g]}{\Pr[x \in g_p']} L_{\cD}(f \mid g) + \frac{\Pr[x \in (g_p' \setminus g)]}{\Pr[x \in g_p']} L_{\cD}(f \mid g_p' \setminus g).
\end{align*}
The last equality comes from $g \subseteq g'_p$ because all nodes are contained in their ancestors. The updated $f$ now uses $\hat{h}^g$ for all $x \in g.$ Therefore:
\begin{equation*}
    L_{\cD}(f \mid g_p') = \frac{\Pr[x \in g]}{\Pr[x \in g_p']} L_{\cD}(\hat{h}^g \mid g) + \frac{\Pr[x \in (g_p' \setminus g)]}{\Pr[x \in g_p']} L_{\cD}(\fold \mid g_p' \setminus g).
\end{equation*}
Apply Lemma \ref{lemma:group_unifconv_finite} for conditional uniform convergence on $\hat{h}^g$:
\begin{equation*}
    L_{\cD}(f \mid g_p') \leq \frac{\Pr[x \in g]}{\Pr[x \in g_p']} L_{S}(\hat{h}^g \mid g) + \frac{\Pr[x \in (g_p' \setminus g)]}{\Pr[x \in g_p']} L_{\cD}(\fold \mid g_p' \setminus g) + \frac{\Pr[x \in g]}{\Pr[x \in g_p']} UC(g).
\end{equation*}
Adding and subtracting $\epsilon_n(g)$ to use the fact that we are in the update case where $L_{S}(\fold \mid g) >  L_{S}(\hat{h}^g \mid g) + \epsilon_n(g)$:
\begin{align*}
    &= \frac{\Pr[x \in g]}{\Pr[x \in g_p']} L_{S}(\hat{h}^g \mid g) + \frac{\Pr[x \in g]}{\Pr[x \in g_p']} \epsilon_n(g) \\
    &\qquad{} + \frac{\Pr[x \in (g_p' \setminus g)]}{\Pr[x \in g_p']} L_{\cD}(\fold \mid g_p' \setminus g) + \frac{\Pr[x \in g]}{\Pr[x \in g_p']} UC(g) - \frac{\Pr[x \in g]}{\Pr[x \in g_p']} \epsilon_n(g) \\
    &\leq \frac{\Pr[x \in g]}{\Pr[x \in g_p']} L_{S}(\fold \mid g) + \frac{\Pr[x \in (g_p' \setminus g)]}{\Pr[x \in g_p']} L_{\cD}(\fold \mid g_p' \setminus g) + \frac{\Pr[x \in g]}{\Pr[x \in g_p']} UC(g) - \frac{\Pr[x \in g]}{\Pr[x \in g_p']} \epsilon_n(g).
\end{align*}
Finally, using Lemma \ref{lemma:fold} on $\fold$ on $x \in g$, applying Lemma \ref{lemma:group_unifconv_finite} again, and recombining terms with Lemma \ref{lemma:disjoint_decomp},
\begin{align*}
    &\leq \frac{\Pr[x \in g]}{\Pr[x \in g_p']} L_{\cD}(\fold \mid g) + \frac{\Pr[x \in (g_p' \setminus g)]}{\Pr[x \in g_p']} L_{\cD}(\fold \mid g_p' \setminus g) + \frac{2\Pr[x \in g]}{\Pr[x \in g_p']} UC(g) - \frac{\Pr[x \in g]}{\Pr[x \in g_p']} \epsilon_n(g) \\
    &= L_{\cD}(\fold \mid g'_p) + \frac{2\Pr[x \in g]}{\Pr[x \in g_p']} UC(g) - \frac{\Pr[x \in g]}{\Pr[x \in g_p']} \epsilon_n(g) \leq L_{\cD}(h \mid g'_p) + 4UC(g'_p).
\end{align*}
The final line follows by our choice of $\epsilon_n(g) = 2UC(g)$ and the inductive hypothesis on $g'_p.$ This shows \eqref{eq:g2} for the second case where we update $f,$ and thus completes our proof.

The proof for $\mathcal{H}$ with VC dimension $d > 0$ is identical, but with
$$UC(g) = 9\sqrt{\frac{2d\log(16|\cG|n/\delta)}{n_g}}$$
and $\epsilon_n(g) = 2UC(g)$.
\end{proof}

\section{Description of \prepend (Algorithm 1 of \citet{tosh_simple_2022})}\label{sec:prepend}
In this section, we give a brief description of the \prepend Algorithm of \citet{tosh_simple_2022}, which is closely related to the concurrent decision list algorithm of \citet{globus-harris_algorithmic_2022}. This algorithm outputs a decision list of predictors $h \in \cH$, where the decision nodes are groups $g \in \cG$. Quoting from \citet{tosh_simple_2022}, such a decision list of length $T$ predicts as follows on input $x \in \cX$ by following:
\begin{center}
    \textbf{if} $g_T(x) = 1$ \textbf{then }\text{return} $h_T(x)$ \textbf{else if} $g_{T-1}(x) = 1$ \textbf{then }\text{return} $h_{T-1}(x)$ \textbf{else if} \dots \textbf{else}\text{ return } $h_0(x).$ 
\end{center}
We shall represent such a decision list by alternating groups and hypotheses; a decision list of length $T$ can be represented as: $f_T := [g_T, h_T, g_{T-1}, h_{T-1}, \dots, g_1, h_1, h_0].$ To construct such a predictor, the \prepend algorithm maintains a decision list $f_t$ and proceeds in rounds $t = 1, 2, 3, \dots$ At round $t$, the algorithm checks whether there exists a group-hypothesis pair $(g,h)$ that violates a desired empirical error bound. If such a pair exists, the algorithm ``prepends'' the group and hypothesis to the front of the list; if no such pair exists, the algorithm terminates. 

Algorithm \ref{alg:prepend} displays the algorithm as presented in \citet{tosh_simple_2022}. 

\begin{algorithm}
\caption{\prepend}\label{alg:prepend}
\begin{algorithmic}[1]
\REQUIRE 
\STATE $S$, a training dataset. 
\STATE Collection of groups $\mathcal{G} \subseteq 2^{\mathcal{X}}$. 
\STATE Error rates $\epsilon_n(g) \in (0, 1)$ for all $g \in \mathcal{G}$
\ENSURE Decision list $f_T: \mathcal{X} \rightarrow \mathcal{Z}$.
\STATE Compute $h_0 \in \argmin_{h \in \cH} L_S(h)$
\FOR {t = 0, 1, 2, \dots,}
\STATE Compute:
$$
(g_{t + 1}, h_{t + 1}) \in \argmax_{(g,h) \in \cG \times \cH} L_S(f_t \mid g) - L_S(h \mid g) - \epsilon_n(g).
$$
\IF{$L_S(f_t \mid g_{t + 1}) - L_S(h_{t + 1} \mid g) \geq \epsilon_n(g_{t + 1})$}
\STATE Prepend $(g_{t + 1}, h_{t + 1})$ to $f_t$ to obtain
$$
f_{t + 1} := [g_{t + 1}, h_{t + 1}, g_t, h_t, \dots, g_1, h_1, h_0].
$$
\ELSE
\STATE \textbf{return} $f_t: \mathcal{X} \rightarrow \cZ$, a decision list predictor.
\ENDIF
\ENDFOR
\end{algorithmic}
\end{algorithm}

\section{Experiment hyperparameters}\label{sec:hyperparams}
In this section, we include the hyperparameters used for training each model class. Logistic regression, decision trees, and random forests all used the implementation from \texttt{scikit-learn} \footnote{https://scikit-learn.org/stable/}. XGBoost was implemented using the open-source \texttt{xgboost} implementation of \citet{chen_xgboost_2016} \footnote{https://xgboost.readthedocs.io/en/stable/}. All the hyperparameters not listed were set to the default values in \texttt{scikit-learn} or \texttt{xgboost}. 

\begin{table}[h!]
\centering
\begin{tabular}{|c c|} 
 \hline
 Model & Hyperparameters\\
 \hline
 Logistic Regression & \texttt{loss = log\_loss}, \texttt{dual=False}, \texttt{solver=lbfgs} \\ 
 Decision Tree & \texttt{criterion = log\_loss}, $\texttt{max\_depth = \{2, 4, 8\}}$ \\
 Random Forest & \texttt{criterion = log\_loss}  \\
 XGBoost & \texttt{objective = binary:logistic}  \\
 \hline
\end{tabular}
\caption{Hyperparamter settings for each choice of benchmark hypothesis class.}
\label{table:hyperparams}
\end{table}

\section{Additional dataset details}\label{sec:dataset_details}
Our experiments span twelve datasets from the Folktables package of \citet{ding_retiring_2022}. These are comprised of nine \textit{statewide} datasets and three \textit{nationwide} datasets.

\textbf{Statewide datasets.} We consider nine different datasets corresponding to one of the three U.S. states of New York (NY), California (CA), and Florida (FL) and one of the three binary prediction tasks of Employment, Income, and Coverage. The description of each task is found in Section \ref{sec:experiment_setup}. 

We consider a couple of hierarchical group structures from subdividing the following categorical demographic attributes:
\begin{itemize}
    \item \texttt{race}. 6 total categories: $\texttt{R1}$ (White alone), $\texttt{R2}$ (Black or African American alone), $\texttt{R3+}$ (Native), $\texttt{R6+}$ (Asian or Pacific Islander), $\texttt{R7}$ (other race alone), and $\texttt{R8}$ (two or more races).
    \item \texttt{sex}. 2 total categories: $\texttt{M}$ (Male) and $\texttt{F}$ (Female).
    \item \texttt{age}. 3 total categories. For Income and Employment: $\texttt{Ya}$ ($\texttt{age} < 35$), $\texttt{Ma}$ ($35 \leq \texttt{age} < 60$), and $\texttt{Oa}$ ($\texttt{age} \geq 60$). For Coverage, $\texttt{Ma}$ is instead defined as $35 \leq \texttt{age} < 50$ and $\texttt{Oa}$ is instead defined as $\texttt{age} \geq 50$ because Coverage only concerns individuals whose age is less than 65.
    \item \texttt{edu}. 4 total categories: $\texttt{HS-}$ (no high school diploma), $\texttt{HS}$ (high school diploma or equivalent, but no college degree), $\texttt{COL}$ (associates or bachelor's degree), and $\texttt{COL+}$ (master's, professional, or doctorate degree beyond bachelor's).  
\end{itemize}

For each of the nine statewide datasets, we subdivide on each of the demographic attributes to attain hierarchically structured groups. We focus on two collections of groups:
\begin{itemize}
    \item \textbf{Attributes:} $\{ \texttt{race}, \texttt{sex}, \texttt{age} \}.$ Total of 54 subgroups (6 groups from \texttt{race}, 12 groups from $\texttt{race} \land \texttt{sex}$, and 36 groups from $\texttt{race} \land \texttt{sex} \land \texttt{age}$). 
    \item \textbf{Attributes:} $\{ \texttt{race}, \texttt{sex}, \texttt{edu} \}.$ Total of 66 subgroups (6 groups from $\texttt{race}$, 12 groups from $\texttt{race} \land \texttt{sex}$, and 48 groups from $\texttt{race} \land \texttt{sex} \land \texttt{edu}$). 
\end{itemize}
Tables \ref{table:statewide_datasets1} and \ref{table:statewide_datasets2} lists all the nine statewide datasets and the number of examples in each separate demographic attribute.  

\begin{table}[h!]
\centering
\begin{tabular}{|c c c c c c c c|} 
 \hline
 Dataset & All & \texttt{R1} & \texttt{R2} & \texttt{R3+} & \texttt{R6+} & \texttt{R7} & \texttt{R8}\\
 \hline\hline
 CA Emp. & 376035 & 231232 & 18831 & 3531 & 58147 & 46316 & 17978 \\ 
 NY Emp. & 196104 & 137179 & 25362 & 735 & 16304 & 11089 & 5435 \\
 FL Emp. & 196828 & 155682 & 26223 & 589 & 5356 & 4328 & 4650 \\
 CA Inc. & 190187 & 118212 & 8656 & 1585 & 31039 & 23285 & 7410 \\
 NY Inc. & 101270 & 72655 & 11970 & 340 & 8639 & 5407 & 2259 \\
 FL Inc. & 94507 & 75218 & 11978 & 267 & 2892 & 2314 & 1838 \\
 CA Cov. & 145994 & 82486 & 8652 & 1626 & 21820 & 24526 & 6884 \\
 NY Cov. & 71379 & 44632 & 11262 & 332 & 7244 & 5754 & 2155 \\
 FL Cov. & 73406 & 53841 & 12872 & 274 & 2350 & 2235 & 1834 \\
 \hline
\end{tabular}
\caption{Number of examples for the \texttt{race} attribute, for each of the nine datasets. ``Emp.'' stands for Employment, ``Inc.'' stands for Income, and ``Cov.'' stands for Coverage.}
\label{table:statewide_datasets1}
\end{table}

\begin{table}[h!]
\centering
\begin{tabular}{|c c c c c c c c c c |} 
 \hline
Dataset & \texttt{M} & \texttt{F} & \texttt{Ya} & \texttt{Ma} & \texttt{Oa} & \texttt{HS-} & \texttt{HS} & \texttt{COL} & \texttt{COL+}\\
 \hline\hline
 CA Emp. & 185603 & 190432 & 164324 & 124293 & 87418 & 126469 & 132156 & 81714 & 35696 \\ 
 NY Emp. & 94471 & 101633 & 81007 & 64422 & 50675 & 57694 & 71694 & 43783 & 22933 \\
 FL Emp. & 95281 & 101547 & 70632 & 63329 & 62867 & 53825 & 80479 & 44754 & 17770 \\
 CA Inc. & 101125 & 89062 & 65087 & 97507 & 27593 & 23840 & 80594 & 58824 & 26929 \\
 NY Inc. & 51408 & 49862 & 33349 & 51543 & 16378 & 8417 & 42215 & 33038 & 17600 \\
 FL Inc. & 48623 & 45884 & 28007 & 49451 & 17049 & 8150 & 44394 & 30289 & 11674 \\
 CA Cov. & 65287 & 80707 & 75484 & 33174 & 37336 & 43646 & 70617 & 25871 & 5860 \\
 NY Cov. & 31218 & 40161 & 36683 & 15235 & 19461 & 17649 & 35723 & 14282 & 3725 \\
 FL Cov. & 32686 & 40720 & 33552 & 16604 & 23250 & 17209 & 38354 & 14913 & 2930 \\
 \hline
\end{tabular}
\caption{Number of examples for the \texttt{sex}, \texttt{age}, and \texttt{edu} attributes, for each of the nine datasets. ``Emp.'' stands for Employment, ``Inc.'' stands for Income, and ``Cov.'' stands for Coverage.}
\label{table:statewide_datasets2}
\end{table}

\textbf{Nationwide datasets.} We consider three more datasets constructed from the Folktables package with many more samples than the statewide datasets. For each of the classification tasks (Income, Employment, and Coverage), we gather all examples from a total of 18 U.S. states. 2 states were chosen for each of the 9 federally designated Census geographic divisions of the United States \citep{us_census_us_2023}. 

These states and their corresponding geographic regions are: MA, CT (New England), NY, PA (Mid-Atlantic), IL, OH (East North Central), MO, MN (West North Central), FL, GA, (South Atlantic), TN, AL (East South Central), TX, LA (West South Central), AZ, CO (Mountain), CA, and WA (Pacific).

For each of the three nationwide datasets, we subdivide on each of the demographic attributes to attain hierarchically structured groups. We focus on two collections of groups:
\begin{itemize}
    \item \textbf{Attributes:} $\{ \texttt{state}, \texttt{race}, \texttt{sex} \}.$ Total of 342 subgroups (18 groups from \texttt{state}, 108 groups from $\texttt{state} \land \texttt{race}$, and 216 groups from $\texttt{state} \land \texttt{race} \land \texttt{sex}$). 
    \item \textbf{Attributes:} $\{ \texttt{state}, \texttt{race}, \texttt{age} \}.$ Total of 450 subgroups (18 groups from \texttt{state}, 108 groups from $\texttt{state} \land \texttt{race}$, and 324 groups from $\texttt{state} \land \texttt{race} \land \texttt{age}$).  
\end{itemize}

\section{Comparison of \prepend and \mgltree (Algorithm~\ref{alg:tree}) on CA Employment and CA Income}\label{sec:comparison_full}
In Table \ref{table:leaves_employment} and Table \ref{table:leaves_income}, we include the group-specific predictors of Algorithm \ref{alg:tree} (labeled ``TREE'') and \prepend (labeled ``PREP'') for each leaf node in the hierarchical group structure induced by splitting by \texttt{race}, \texttt{sex}, and \texttt{age} for the CA Employment and CA Income datasets. Because the leaf nodes partition the input space, this accounts for the behavior of the predictors on any possible test example $x \in \cX$. As described in Section \ref{sec:comparison}, we observe that Algorithm \ref{alg:tree} prefers using coarser-grained group-specific predictors higher up the tree, such as \texttt{ALL} or \texttt{R1}, while \prepend often resorts to the finer-grained predictors at the leaves.

\begin{table*}[ht]
\centering
\begin{tabular}{|c| c c c c c c |} 
\hline
Leaf & $\texttt{R1},\texttt{M},\texttt{Ya}$ & $\texttt{R1},\texttt{M},\texttt{Ma}$ & $\texttt{R1},\texttt{M},\texttt{Oa}$ & $\texttt{R1},\texttt{F},\texttt{Ya}$ & $\texttt{R1},\texttt{F},\texttt{Ma}$ & $\texttt{R1},\texttt{F},\texttt{Oa}$ \\
\hline\hline
TREE & $\texttt{R1},\texttt{M},\texttt{Ya}$ & $\texttt{R1},\texttt{M},\texttt{Ma}$ & $\texttt{R1},\texttt{M},\texttt{Oa}$ & $\texttt{R1},\texttt{F},\texttt{Ya}$ & $\texttt{R1},\texttt{F},\texttt{Ma}$ & $\texttt{R1},\texttt{F},\texttt{Oa}$ \\ 
PREP & $\texttt{R1},\texttt{M},\texttt{Ya}$ & $\texttt{R1},\texttt{M},\texttt{Ma}$ & $\texttt{R1},\texttt{M},\texttt{Oa}$ & $\texttt{R1},\texttt{F},\texttt{Ya}$ & $\texttt{R1},\texttt{F},\texttt{Ma}$ & $\texttt{R1},\texttt{F},\texttt{Oa}$ \\
\hline\hline
 
Leaf & $\texttt{R2},\texttt{M},\texttt{Ya}$ & $\texttt{R2},\texttt{M},\texttt{Ma}$ & $\texttt{R2},\texttt{M},\texttt{Oa}$ & $\texttt{R2},\texttt{F},\texttt{Ya}$ & $\texttt{R2},\texttt{F},\texttt{Ma}$ & $\texttt{R2},\texttt{F},\texttt{Oa}$ \\
\hline\hline
TREE & $\texttt{R2},\texttt{M},\texttt{Ya}$ & \textbf{\texttt{ALL}} & \textbf{\texttt{ALL}} & $\texttt{R2},\texttt{F},\texttt{Ya}$ & $\texttt{R2},\texttt{F},\texttt{Ma}$ & $\texttt{R2},\texttt{F},\texttt{Oa}$\\ 
PREP & $\texttt{R2},\texttt{M},\texttt{Ya}$ & \textbf{$\texttt{R2},\texttt{M},\texttt{Ma}$} & \textbf{$\texttt{R2},\texttt{M},\texttt{Oa}$} & $\texttt{R2},\texttt{F},\texttt{Ya}$ & $\texttt{R2},\texttt{F},\texttt{Ma}$ & $\texttt{R2},\texttt{F},\texttt{Oa}$ \\
\hline\hline
 
Leaf & $\texttt{R3+},\texttt{M},\texttt{Ya}$ & $\texttt{R3+},\texttt{M},\texttt{Ma}$ & $\texttt{R3+},\texttt{M},\texttt{Oa}$ & $\texttt{R3+},\texttt{F},\texttt{Ya}$ & $\texttt{R3+},\texttt{F},\texttt{Ma}$ & $\texttt{R3+},\texttt{F},\texttt{Oa}$ \\
\hline\hline
TREE & \textbf{\texttt{ALL}} & \textbf{\texttt{ALL}} & \textbf{\texttt{ALL}} & $\texttt{R3+},\texttt{F},\texttt{Ya}$ & $\texttt{R3+},\texttt{F},\texttt{Ma}$ & \textbf{\texttt{ALL}} \\ 
PREP & \textbf{$\texttt{R3+},\texttt{M},\texttt{Ya}$} & \textbf{$\texttt{R3+},\texttt{M},\texttt{Ma}$} & \textbf{$\texttt{R3+},\texttt{M},\texttt{Oa}$} & $\texttt{R3+},\texttt{F},\texttt{Ya}$ & $\texttt{R3+},\texttt{F},\texttt{Ma}$ & \textbf{$\texttt{R3+},\texttt{F},\texttt{Oa}$} \\
\hline\hline
 
Leaf & $\texttt{R6+},\texttt{M},\texttt{Ya}$ & $\texttt{R6+},\texttt{M},\texttt{Ma}$ & $\texttt{R6+},\texttt{M},\texttt{Oa}$ & $\texttt{R6+},\texttt{F},\texttt{Ya}$ & $\texttt{R6+},\texttt{F},\texttt{Ma}$ & $\texttt{R6+},\texttt{F},\texttt{Oa}$ \\
\hline\hline
TREE & $\texttt{R6+},\texttt{M},\texttt{Ya}$ & $\texttt{R6+},\texttt{M},\texttt{Ma}$ & $\texttt{R6+},\texttt{M},\texttt{Oa}$ & $\texttt{R6+},\texttt{F},\texttt{Ya}$ & $\texttt{R6+},\texttt{F},\texttt{Ma}$ & $\texttt{R6+},\texttt{F},\texttt{Oa}$  \\
PREP & $\texttt{R6+},\texttt{M},\texttt{Ya}$ & $\texttt{R6+},\texttt{M},\texttt{Ma}$ & $\texttt{R6+},\texttt{M},\texttt{Oa}$ & $\texttt{R6+},\texttt{F},\texttt{Ya}$ & $\texttt{R6+},\texttt{F},\texttt{Ma}$ & $\texttt{R6+},\texttt{F},\texttt{Oa}$ \\ 
\hline\hline
 
Leaf & $\texttt{R7},\texttt{M},\texttt{Ya}$ & $\texttt{R7},\texttt{M},\texttt{Ma}$ & $\texttt{R7},\texttt{M},\texttt{Oa}$ & $\texttt{R7},\texttt{F},\texttt{Ya}$ & $\texttt{R7},\texttt{F},\texttt{Ma}$ & $\texttt{R7},\texttt{F},\texttt{Oa}$ \\
\hline\hline
TREE & $\texttt{R7},\texttt{M},\texttt{Ya}$ & $\texttt{R7},\texttt{M},\texttt{Ma}$ & $\texttt{R7},\texttt{M},\texttt{Oa}$ & $\texttt{R7},\texttt{F},\texttt{Ya}$ & $\texttt{R7},\texttt{F},\texttt{Ma}$ & \textbf{\texttt{ALL}} \\ 
PREP & $\texttt{R7},\texttt{M},\texttt{Ya}$ & $\texttt{R7},\texttt{M},\texttt{Ma}$ & $\texttt{R7},\texttt{M},\texttt{Oa}$ & $\texttt{R7},\texttt{F},\texttt{Ya}$ & $\texttt{R7},\texttt{F},\texttt{Ma}$ & \textbf{$\texttt{R7},\texttt{F},\texttt{Oa}$} \\
\hline\hline

Leaf & $\texttt{R8},\texttt{M},\texttt{Ya}$ & $\texttt{R8},\texttt{M},\texttt{Ma}$ & $\texttt{R8},\texttt{M},\texttt{Oa}$ & $\texttt{R8},\texttt{F},\texttt{Ya}$ & $\texttt{R8},\texttt{F},\texttt{Ma}$ & $\texttt{R8},\texttt{F},\texttt{Oa}$ \\
\hline\hline
TREE & $\texttt{R8},\texttt{M},\texttt{Ya}$ & \textbf{$\texttt{R8},\texttt{M}$} & \textbf{$\texttt{R8},\texttt{M}$} & \textbf{\texttt{ALL}} & $\texttt{R8},\texttt{F},\texttt{Ma}$ & $\texttt{R8},\texttt{F},\texttt{Oa}$ \\ 
PREP & $\texttt{R8},\texttt{M},\texttt{Ya}$ & \textbf{$\texttt{R8},\texttt{M},\texttt{Ma}$} & \textbf{$\texttt{R8},\texttt{M},\texttt{Oa}$} & \textbf{$\texttt{R8},\texttt{F},\texttt{Ya}$} & $\texttt{R8},\texttt{F},\texttt{Ma}$ & $\texttt{R8},\texttt{F},\texttt{Oa}$ \\
\hline
\end{tabular}
\caption{\label{table:leaves_employment} Predictors for each \texttt{race}-\texttt{sex}-\texttt{age} leaf node in CA Employment dataset. \texttt{ALL} indicates the ERM logistic regression predictor trained on all the data. Bolded entries are leaves where \mgltree and \prepend differ in their predictor.}
\end{table*}

\begin{table*}[ht]
\centering
\begin{tabular}{|c| c c c c c c |} 
\hline
Leaf & $\texttt{R1},\texttt{M},\texttt{Ya}$ & $\texttt{R1},\texttt{M},\texttt{Ma}$ & $\texttt{R1},\texttt{M},\texttt{Oa}$ & $\texttt{R1},\texttt{F},\texttt{Ya}$ & $\texttt{R1},\texttt{F},\texttt{Ma}$ & $\texttt{R1},\texttt{F},\texttt{Oa}$ \\
\hline\hline
TREE & \textbf{\texttt{R1}} & \textbf{\texttt{R1}} & \textbf{\texttt{R1}} & \textbf{\texttt{R1}} & \textbf{\texttt{R1}} & $\texttt{R1},\texttt{F},\texttt{Oa}$ \\ 
PREP & \textbf{$\texttt{R1},\texttt{M},\texttt{Ya}$} & \textbf{$\texttt{R1},\texttt{M},\texttt{Ma}$} & \textbf{$\texttt{R1},\texttt{M},\texttt{Oa}$} & \textbf{$\texttt{R1},\texttt{F},\texttt{Ya}$} & \textbf{$\texttt{R1},\texttt{F},\texttt{Ma}$} & $\texttt{R1},\texttt{F},\texttt{Oa}$ \\
\hline\hline
 
Leaf & $\texttt{R2},\texttt{M},\texttt{Ya}$ & $\texttt{R2},\texttt{M},\texttt{Ma}$ & $\texttt{R2},\texttt{M},\texttt{Oa}$ & $\texttt{R2},\texttt{F},\texttt{Ya}$ & $\texttt{R2},\texttt{F},\texttt{Ma}$ & $\texttt{R2},\texttt{F},\texttt{Oa}$ \\
\hline\hline
TREE & \textbf{\texttt{ALL}} & \textbf{\texttt{ALL}} & \textbf{\texttt{ALL}} & \textbf{\texttt{ALL}} & \texttt{ALL} & \textbf{\texttt{ALL}} \\ 
PREP & \textbf{$\texttt{R2},\texttt{M},\texttt{Ya}$} & \textbf{$\texttt{R2},\texttt{M},\texttt{Ma}$} & \textbf{$\texttt{R2},\texttt{M},\texttt{Oa}$} & \textbf{$\texttt{R2},\texttt{F},\texttt{Ya}$} & \texttt{ALL} & \textbf{$\texttt{R2},\texttt{F},\texttt{Oa}$} \\
\hline\hline
 
Leaf & $\texttt{R3+},\texttt{M},\texttt{Ya}$ & $\texttt{R3+},\texttt{M},\texttt{Ma}$ & $\texttt{R3+},\texttt{M},\texttt{Oa}$ & $\texttt{R3+},\texttt{F},\texttt{Ya}$ & $\texttt{R3+},\texttt{F},\texttt{Ma}$ & $\texttt{R3+},\texttt{F},\texttt{Oa}$ \\
\hline\hline
TREE & \textbf{\texttt{ALL}} & \texttt{ALL} & \textbf{\texttt{ALL}} & \textbf{\texttt{ALL}} & \textbf{\texttt{ALL}} & \textbf{\texttt{ALL}} \\ 
PREP & \textbf{$\texttt{R3+},\texttt{M},\texttt{Ya}$} & \texttt{ALL} & \textbf{$\texttt{R3+},\texttt{M},\texttt{Oa}$} & \textbf{$\texttt{R3+},\texttt{F},\texttt{Ya}$} & \textbf{$\texttt{R3+},\texttt{F},\texttt{Ma}$} & \textbf{$\texttt{R3+},\texttt{F},\texttt{Oa}$} \\
\hline\hline
 
Leaf & $\texttt{R6+},\texttt{M},\texttt{Ya}$ & $\texttt{R6+},\texttt{M},\texttt{Ma}$ & $\texttt{R6+},\texttt{M},\texttt{Oa}$ & $\texttt{R6+},\texttt{F},\texttt{Ya}$ & $\texttt{R6+},\texttt{F},\texttt{Ma}$ & $\texttt{R6+},\texttt{F},\texttt{Oa}$ \\
\hline\hline
TREE & \textbf{\texttt{R6+}} & \textbf{$\texttt{R6+},\texttt{M},\texttt{Ma}$} & \textbf{\texttt{R6+}} & \textbf{\texttt{R6+}} & \textbf{\texttt{R6+}} & \textbf{\texttt{R6+}} \\ 
PREP & \textbf{\texttt{ALL}} & \textbf{\texttt{ALL}} & \textbf{\texttt{ALL}} & \textbf{\texttt{ALL}} & \textbf{$\texttt{R6+},\texttt{F},\texttt{Ma}$} & \textbf{$\texttt{R6+},\texttt{F},\texttt{Oa}$} \\
\hline\hline
 
Leaf & $\texttt{R7},\texttt{M},\texttt{Ya}$ & $\texttt{R7},\texttt{M},\texttt{Ma}$ & $\texttt{R7},\texttt{M},\texttt{Oa}$ & $\texttt{R7},\texttt{F},\texttt{Ya}$ & $\texttt{R7},\texttt{F},\texttt{Ma}$ & $\texttt{R7},\texttt{F},\texttt{Oa}$ \\
\hline\hline
TREE & $\texttt{R7},\texttt{M},\texttt{Ya}$ & $\texttt{R7},\texttt{M},\texttt{Ma}$ & $\texttt{R7},\texttt{M},\texttt{Oa}$ & $\texttt{R7},\texttt{F},\texttt{Ya}$ & \textbf{$\texttt{R7},\texttt{F},\texttt{Ma}$} & \texttt{ALL} \\ 
PREP & $\texttt{R7},\texttt{M},\texttt{Ya}$ & $\texttt{R7},\texttt{M},\texttt{Ma}$ & $\texttt{R7},\texttt{M},\texttt{Oa}$ & $\texttt{R7},\texttt{F},\texttt{Ya}$ & \textbf{\texttt{ALL}} & \texttt{ALL} \\
\hline\hline

Leaf & $\texttt{R8},\texttt{M},\texttt{Ya}$ & $\texttt{R8},\texttt{M},\texttt{Ma}$ & $\texttt{R8},\texttt{M},\texttt{Oa}$ & $\texttt{R8},\texttt{F},\texttt{Ya}$ & $\texttt{R8},\texttt{F},\texttt{Ma}$ & $\texttt{R8},\texttt{F},\texttt{Oa}$ \\
\hline\hline
TREE & $\texttt{R8},\texttt{M},\texttt{Ya}$ & \textbf{\texttt{ALL}} & \textbf{\texttt{ALL}} & \textbf{\texttt{ALL}} & $\texttt{R8},\texttt{F},\texttt{Ma}$ & $\texttt{R8},\texttt{F},\texttt{Oa}$ \\ 
PREP & $\texttt{R8},\texttt{M},\texttt{Ya}$ & \textbf{$\texttt{R8},\texttt{M},\texttt{Ma}$} & \textbf{$\texttt{R8},\texttt{M},\texttt{Oa}$} & \textbf{$\texttt{R8},\texttt{F},\texttt{Ya}$} & $\texttt{R8},\texttt{F},\texttt{Ma}$ & $\texttt{R8},\texttt{F},\texttt{Oa}$  \\
\hline
\end{tabular}
\caption{\label{table:leaves_income} Predictors for each \texttt{race}-\texttt{sex}-\texttt{age} leaf node in CA Income dataset. \texttt{ALL} indicates the ERM logistic regression predictor trained on all the data. Bolded entries are leaves where \mgltree and \prepend differ in their predictor.}
\end{table*}

\section{Additional experimental results}\label{sec:add_experiments}
In this section, we provide additional results for the other datasets not shown in the main body. For each dataset, we consider the two collections of hierarchical groups described in Appendix \ref{sec:dataset_details}, and we compare the generalization performance of all the methods and benchmark hypothesis classes described in Section \ref{sec:experiment_setup}.

For all the datasets, we plot the test error of each method on a held-out test set for each group. Each group corresponds to a point on the plot. The $y = x$ line corresponds to equal test error between the methods; all points above the line indicates that our Algorithm \ref{alg:tree} had lower test error. Error bars are from the standard error from 10 random trials on each group (from a fresh train-test split, retraining each model from scratch). Error bars are omitted for the nationwide datasets for visual presentation.

Throughout all the figures, ``DecisionTree2,'' ``DecisionTree4,'' and ``DecisionTree8'' refer to decision trees trained with the \texttt{max\_depth} parameter set to 2, 4, and 8, respectively. ``ERM'' refers to the model trained on all available data, ``G-ERM'' refers to the group-stratified model trained only on data from the group corresponding to the point on the plot, and ``PREP'' refers to the \prepend algorithm of \citet{tosh_simple_2022}. 

We also include a ``case study'' of a particular dataset's (CA Employment) error rates group-by-group for logistic regression. We see that the bars for Algorithm \ref{alg:tree} are consistently on par with or lower than both the ERM and Group ERM bars, sometimes by a very significant margin. In cases where Algorithm \ref{alg:tree} beats ERM by a significant amount (such as on, say, the group \texttt{R1, M, Oa}), the predictor from Algorithm \ref{alg:tree} has identified an instance of a high-error subgroup on a specific slice of the population, which may prove useful for further auditing. 

\begin{figure}[ht]
\vskip 0.2in
\begin{center}
\scalebox{0.85}{
\centerline{\includegraphics[width=\columnwidth]{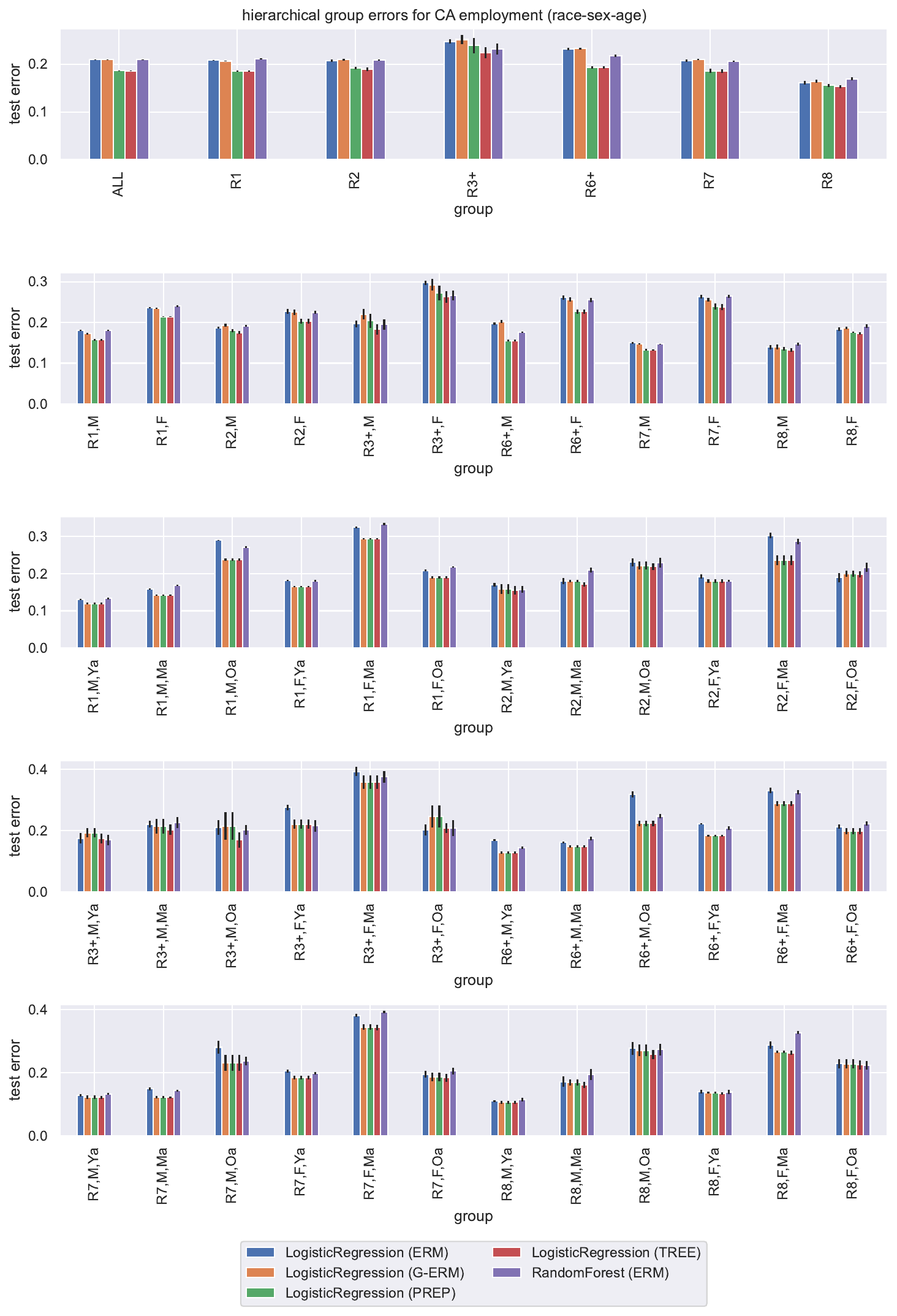}}}
\caption{\textbf{Test accuracy of $\cH =$ Logistic Regression for \texttt{race-sex-age} groups (CA Employment).} Test errors across all $|\cG| = 54$ hierarchically structured groups from \texttt{race-sex-age}. See Appendix \ref{sec:dataset_details} for more information on the specific categories for each group.}
\label{fig:employment_logreg_bar}
\end{center}
\vskip -0.2in
\end{figure}

\begin{figure}[H]
\begin{center}
\centerline{\includegraphics[width=\columnwidth]{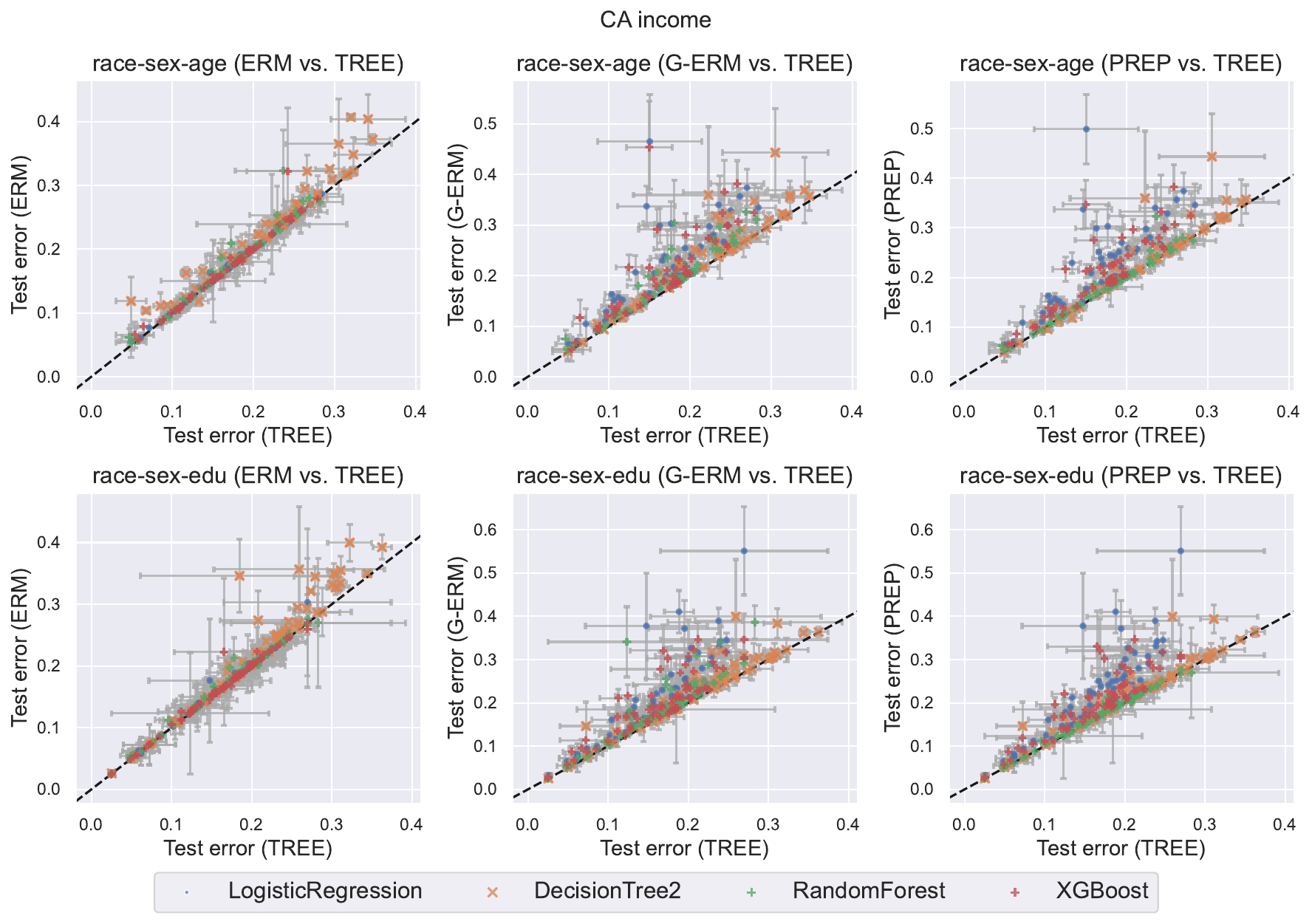}}
\caption{Test error of \mgltree vs.~ERM, Group ERM, and \prepend on \texttt{race}-\texttt{sex}-\texttt{age} and \texttt{race}-\texttt{sex}-\texttt{edu} groups (CA Income). Each point corresponds to a group; points above the $y = x$ line show that \mgltree generalizes better than the competitor method on that particular group. Benchmark hypothesis classes considered: logistic regression, decision trees with \texttt{max\_depth} 2, random forest, and XGBoost.} 
\label{fig:incomeCA}
\end{center}
\end{figure}

\begin{figure}[H]
\vskip 0.2in
\begin{center}
\centerline{\includegraphics[width=\columnwidth]{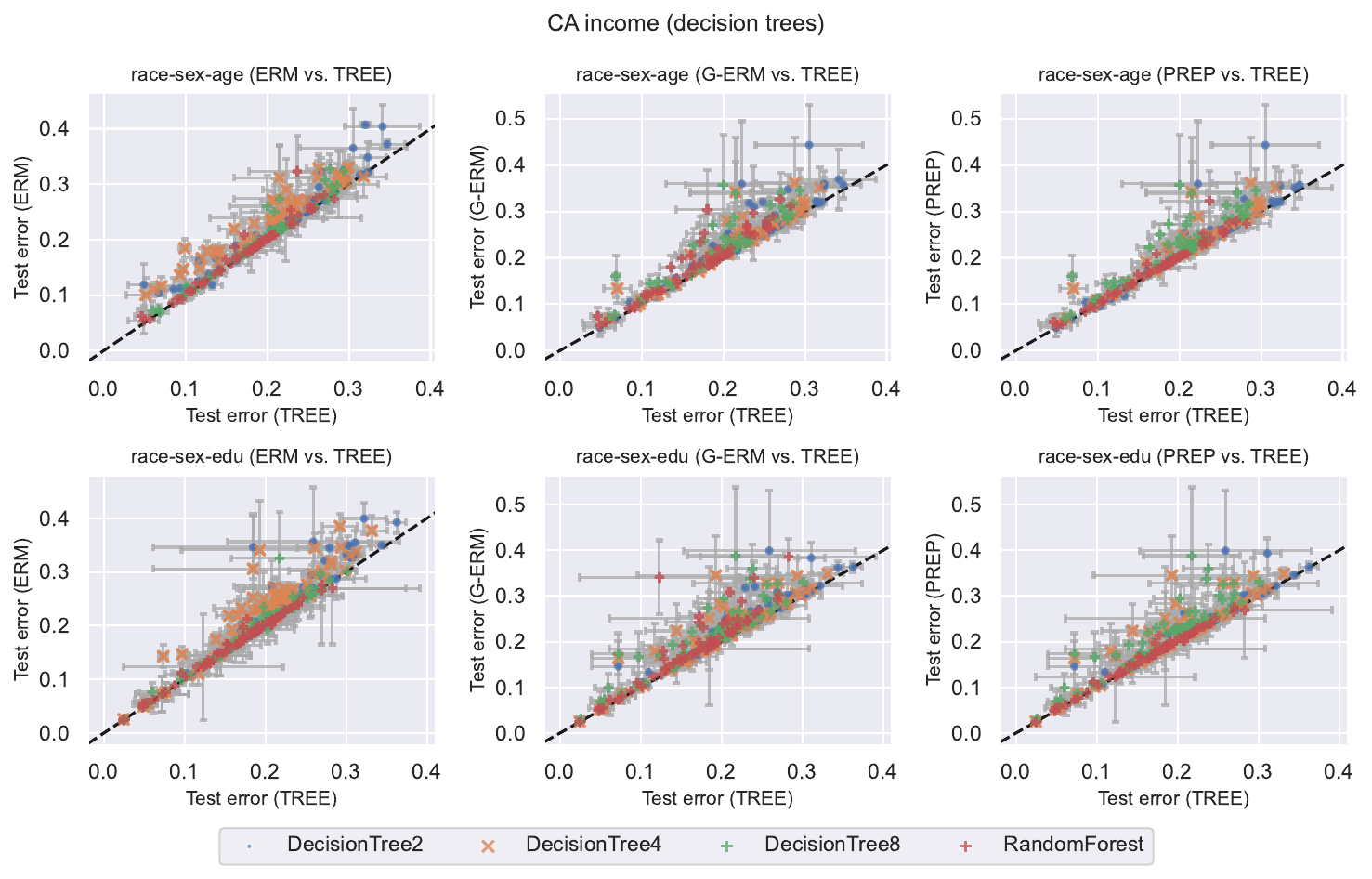}}
\caption{Test error of \mgltree vs.~ERM, Group ERM, and \prepend on \texttt{race}-\texttt{sex}-\texttt{age} and \texttt{race}-\texttt{sex}-\texttt{edu} groups (CA Income). Each point corresponds to a group; points above the $y = x$ line show that \mgltree generalizes better than the competitor method on that particular group. Benchmark hypothesis classes considered: decision trees with \texttt{max\_depth} 2, decision trees with \texttt{max\_depth} 4, decision trees with \texttt{max\_depth} 8, and random forest.} 
\label{fig:incomeCA_dt}
\end{center}
\vskip -0.2in
\end{figure}

\begin{figure}[H]
\vskip 0.2in
\begin{center}
\centerline{\includegraphics[width=\columnwidth]{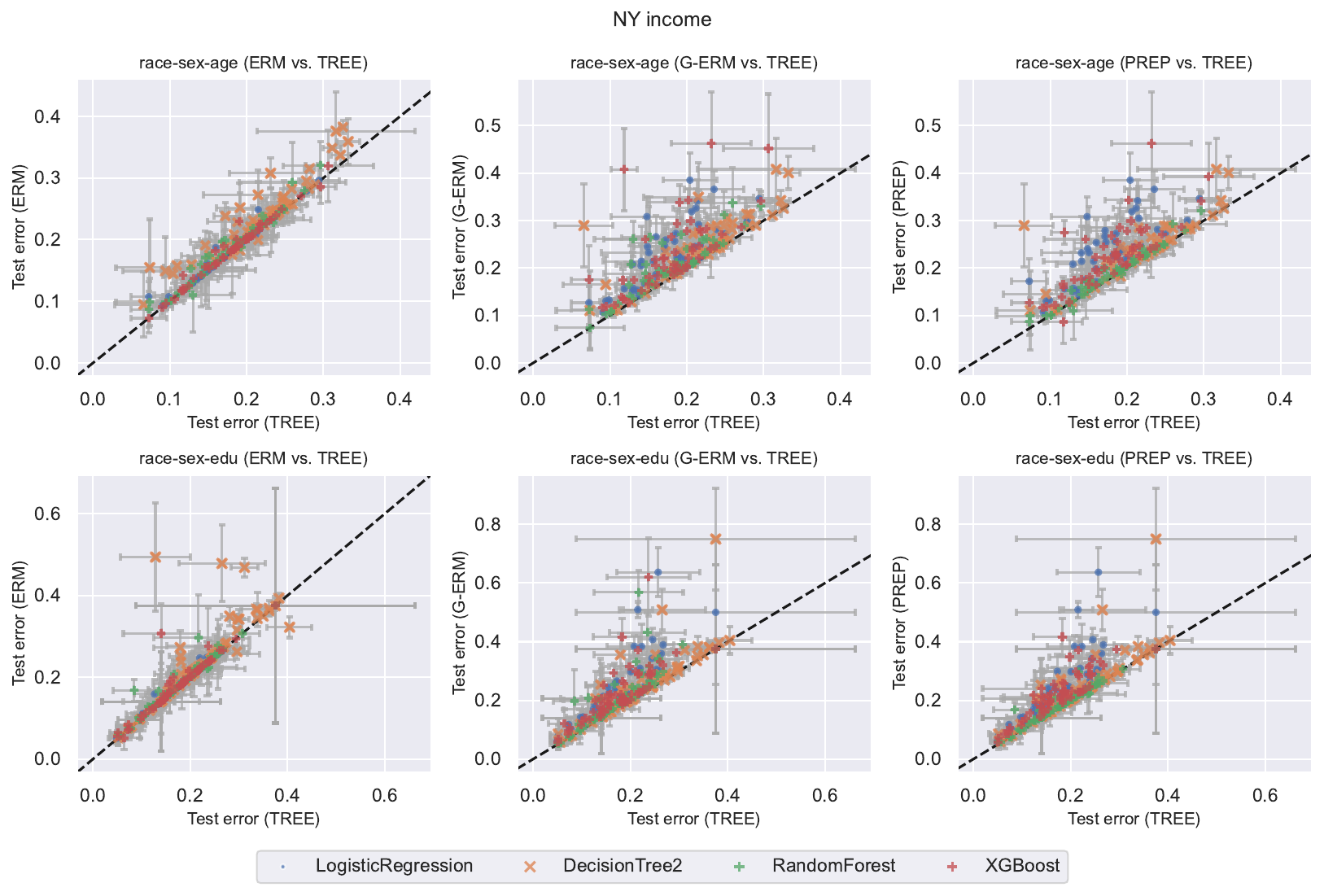}}
\caption{Test error of \mgltree vs.~ERM, Group ERM, and \prepend on \texttt{race}-\texttt{sex}-\texttt{age} and \texttt{race}-\texttt{sex}-\texttt{edu} groups (NY Income). Each point corresponds to a group; points above the $y = x$ line show that \mgltree generalizes better than the competitor method on that particular group. Benchmark hypothesis classes considered: logistic regression, decision trees with \texttt{max\_depth} 2, random forest, and XGBoost.} 
\label{fig:incomeNY}
\end{center}
\vskip -0.2in
\end{figure}

\begin{figure}[H]
\vskip 0.2in
\begin{center}
\centerline{\includegraphics[width=\columnwidth]{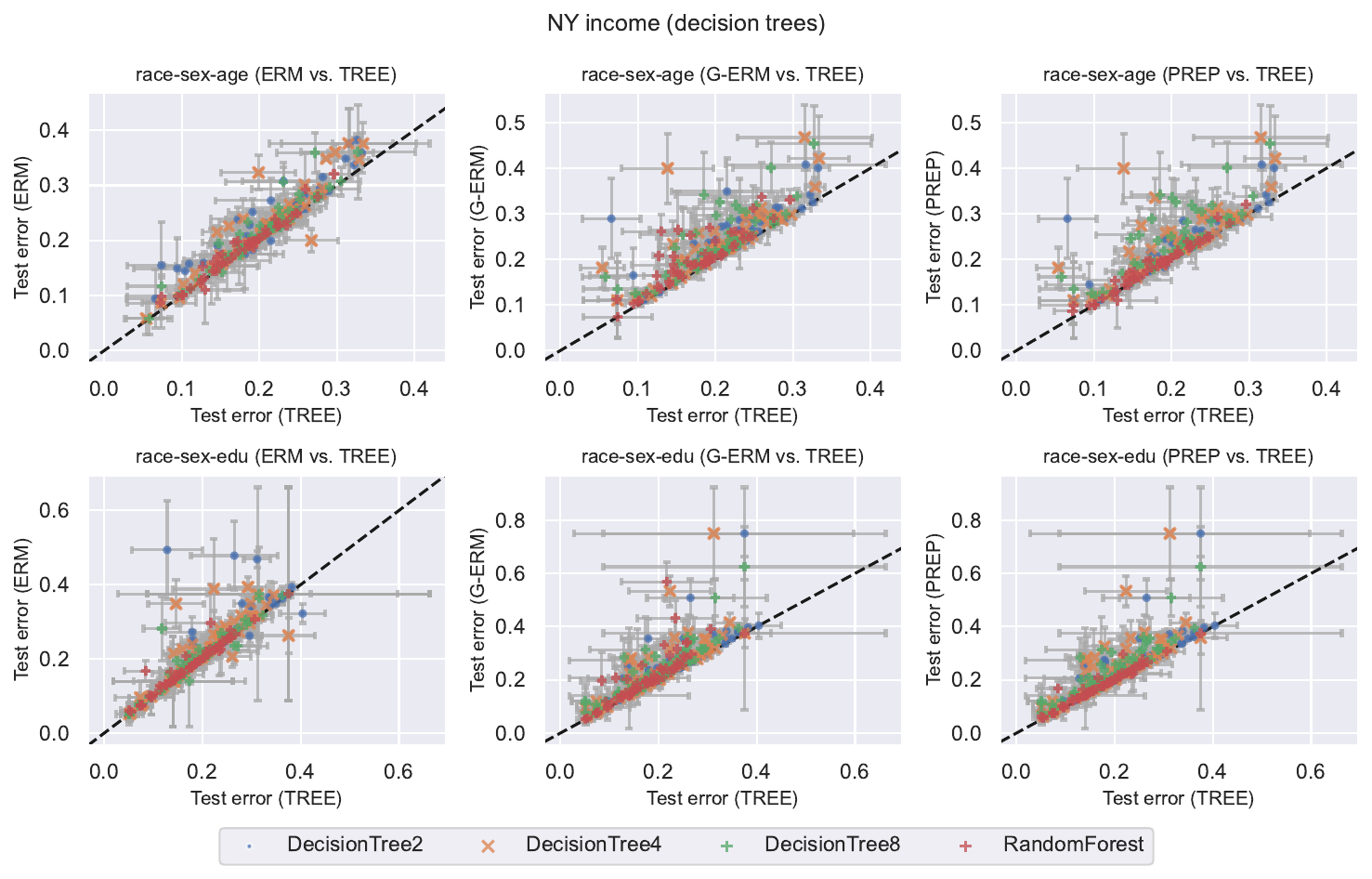}}
\caption{Test error of \mgltree vs.~ERM, Group ERM, and \prepend on \texttt{race}-\texttt{sex}-\texttt{age} and \texttt{race}-\texttt{sex}-\texttt{edu} groups (NY Income). Each point corresponds to a group; points above the $y = x$ line show that \mgltree generalizes better than the competitor method on that particular group. Benchmark hypothesis classes considered: decision trees with \texttt{max\_depth} 2, decision trees with \texttt{max\_depth} 4, decision trees with \texttt{max\_depth} 8, and random forest.} 
\label{fig:incomeNY_dt}
\end{center}
\vskip -0.2in
\end{figure}

\begin{figure}[H]
\vskip 0.2in
\begin{center}
\centerline{\includegraphics[width=\columnwidth]{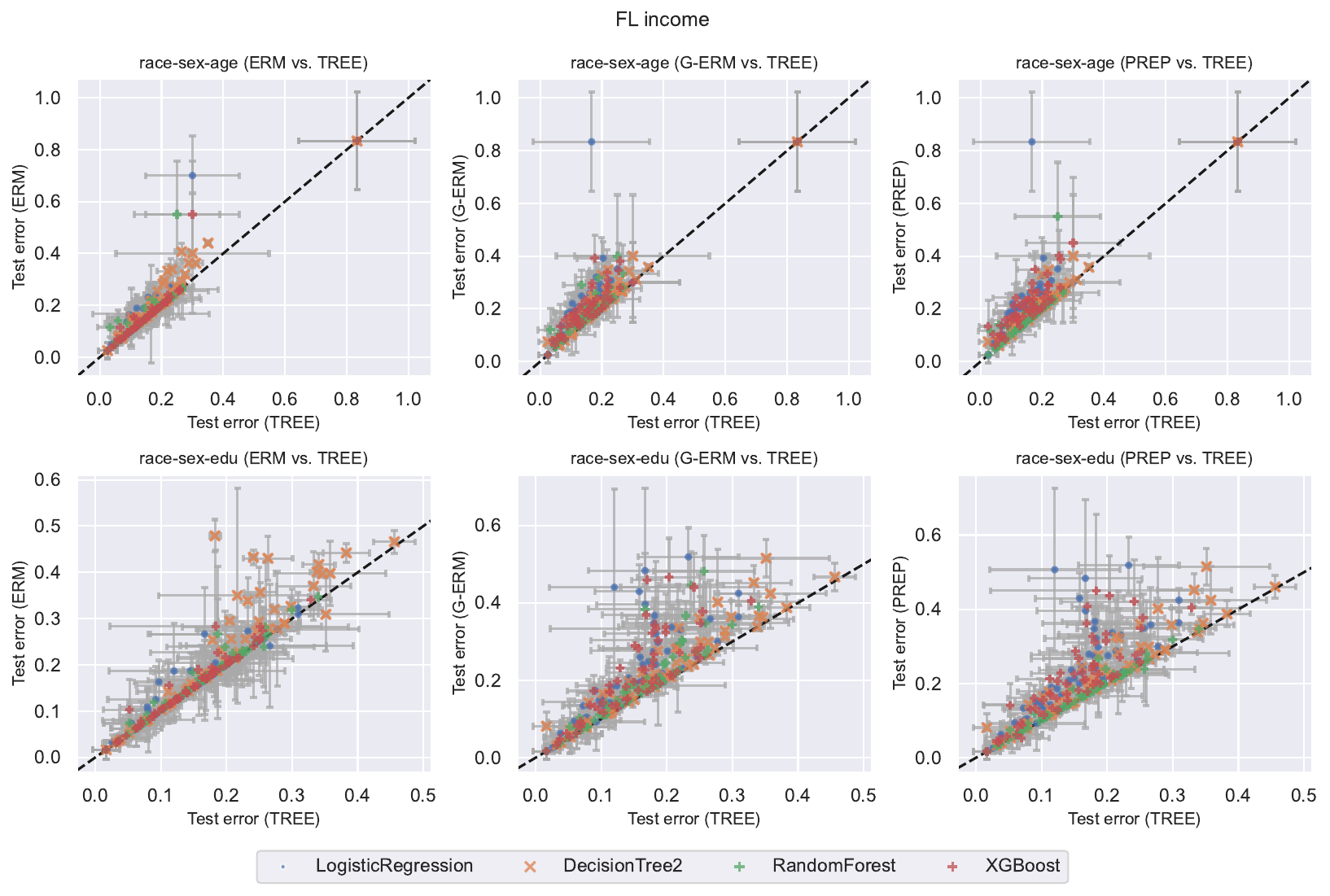}}
\caption{Test error of \mgltree vs.~ERM, Group ERM, and \prepend on \texttt{race}-\texttt{sex}-\texttt{age} and \texttt{race}-\texttt{sex}-\texttt{edu} groups (FL Income). Each point corresponds to a group; points above the $y = x$ line show that \mgltree generalizes better than the competitor method on that particular group. Benchmark hypothesis classes considered: logistic regression, decision trees with \texttt{max\_depth} 2, random forest, and XGBoost.} 
\label{fig:incomeFL}
\end{center}
\vskip -0.2in
\end{figure}

\begin{figure}[H]
\vskip 0.2in
\begin{center}
\centerline{\includegraphics[width=\columnwidth]{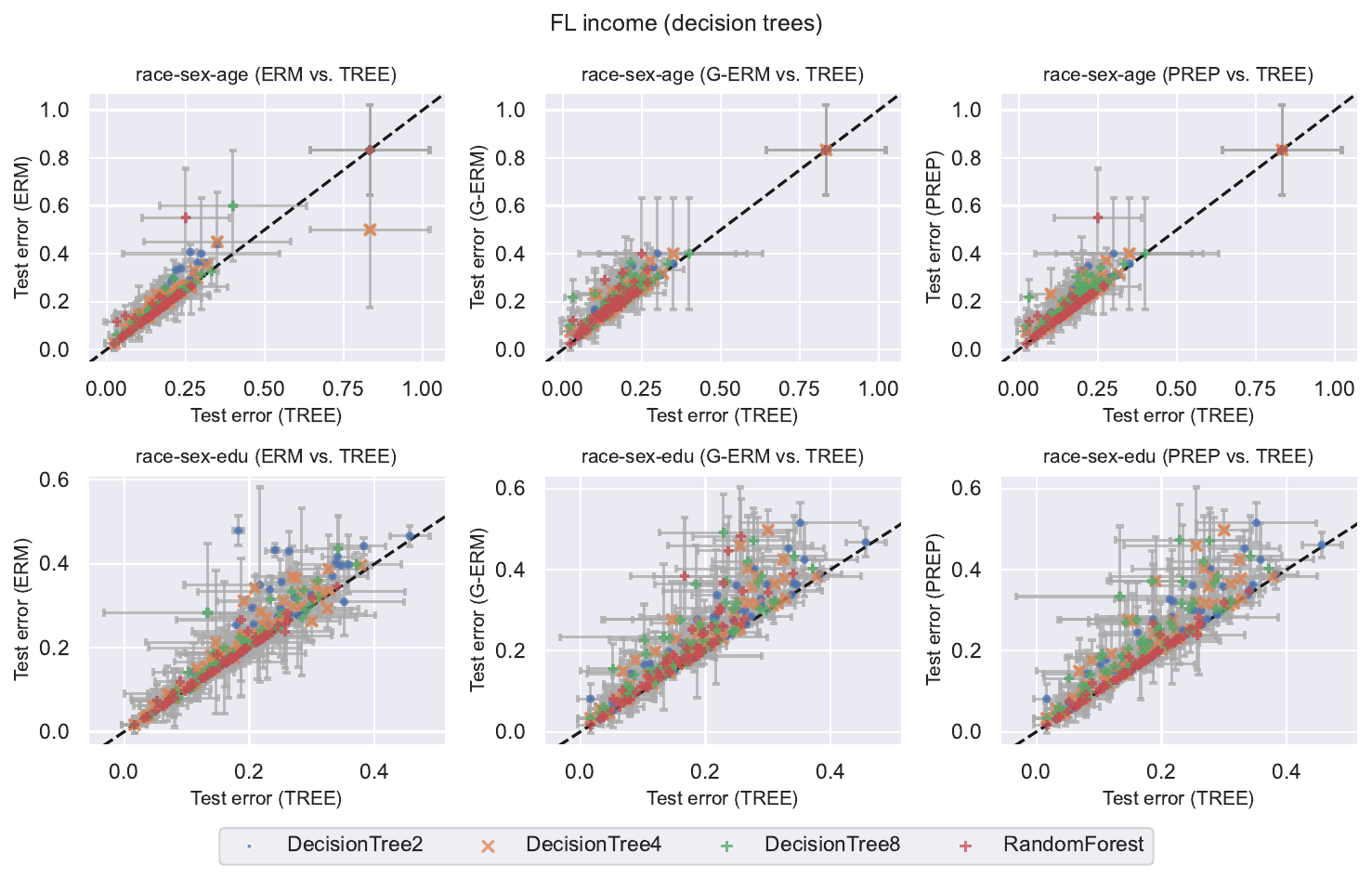}}
\caption{Test error of \mgltree vs.~ERM, Group ERM, and \prepend on \texttt{race}-\texttt{sex}-\texttt{age} and \texttt{race}-\texttt{sex}-\texttt{edu} groups (FL Income). Each point corresponds to a group; points above the $y = x$ line show that \mgltree generalizes better than the competitor method on that particular group. Benchmark hypothesis classes considered: decision trees with \texttt{max\_depth} 2, decision trees with \texttt{max\_depth} 4, decision trees with \texttt{max\_depth} 8, and random forest.} 
\label{fig:incomeFL_dt}
\end{center}
\vskip -0.2in
\end{figure}

\begin{figure}[H]
\vskip 0.2in
\begin{center}
\centerline{\includegraphics[width=\columnwidth]{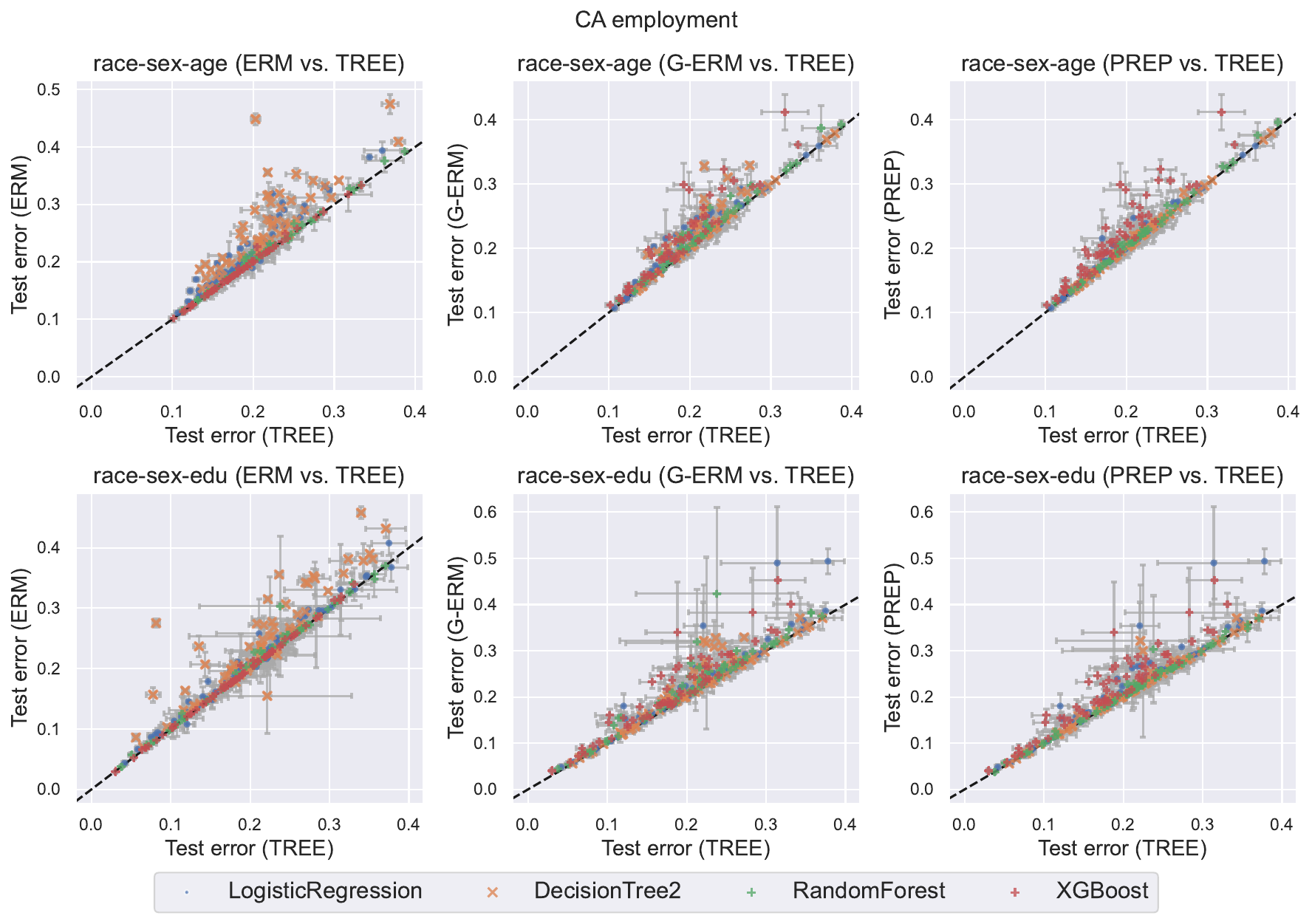}}
\caption{Test error of \mgltree vs.~ERM, Group ERM, and \prepend on \texttt{race}-\texttt{sex}-\texttt{age} and \texttt{race}-\texttt{sex}-\texttt{edu} groups (CA Employment). Each point corresponds to a group; points above the $y = x$ line show that \mgltree generalizes better than the competitor method on that particular group. Benchmark hypothesis classes considered: logistic regression, decision trees with \texttt{max\_depth} 2, random forest, and XGBoost.} 
\label{fig:employmentCA}
\end{center}
\vskip -0.2in
\end{figure}

\begin{figure}[H]
\vskip 0.2in
\begin{center}
\centerline{\includegraphics[width=\columnwidth]{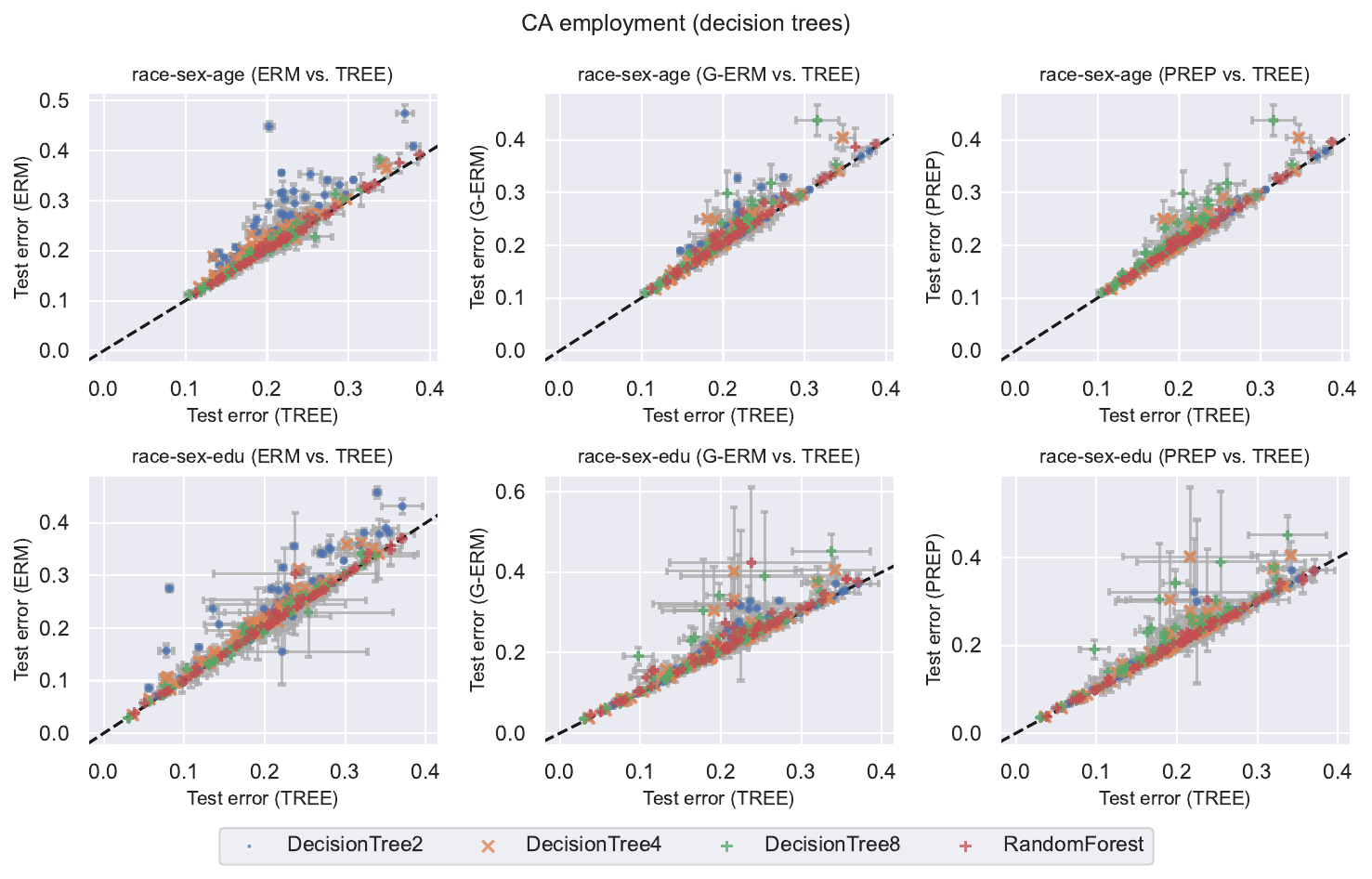}}
\caption{Test error of \mgltree vs.~ERM, Group ERM, and \prepend on \texttt{race}-\texttt{sex}-\texttt{age} and \texttt{race}-\texttt{sex}-\texttt{edu} groups (CA Employment). Each point corresponds to a group; points above the $y = x$ line show that \mgltree generalizes better than the competitor method on that particular group. Benchmark hypothesis classes considered: decision trees with \texttt{max\_depth} 2, decision trees with \texttt{max\_depth} 4, decision trees with \texttt{max\_depth} 8, and random forest.} 
\label{fig:employmentCA_dt}
\end{center}
\vskip -0.2in
\end{figure}

\begin{figure}[H]
\vskip 0.2in
\begin{center}
\centerline{\includegraphics[width=\columnwidth]{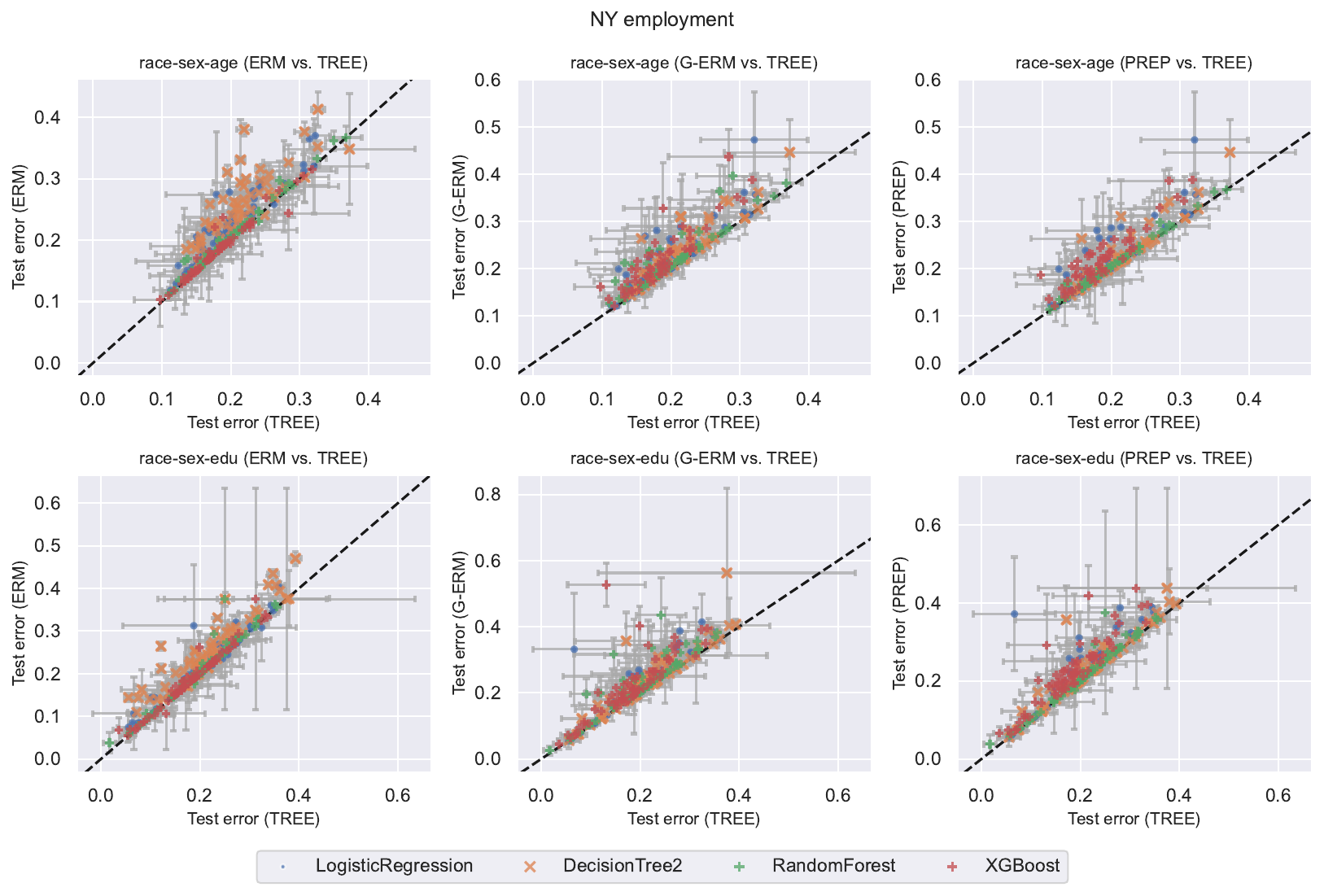}}
\caption{Test error of \mgltree vs.~ERM, Group ERM, and \prepend on \texttt{race}-\texttt{sex}-\texttt{age} and \texttt{race}-\texttt{sex}-\texttt{edu} groups (NY Employment). Each point corresponds to a group; points above the $y = x$ line show that \mgltree generalizes better than the competitor method on that particular group. Benchmark hypothesis classes considered: logistic regression, decision trees with \texttt{max\_depth} 2, random forest, and XGBoost.} 
\label{fig:employmentNY}
\end{center}
\vskip -0.2in
\end{figure}

\begin{figure}[H]
\vskip 0.2in
\begin{center}
\centerline{\includegraphics[width=\columnwidth]{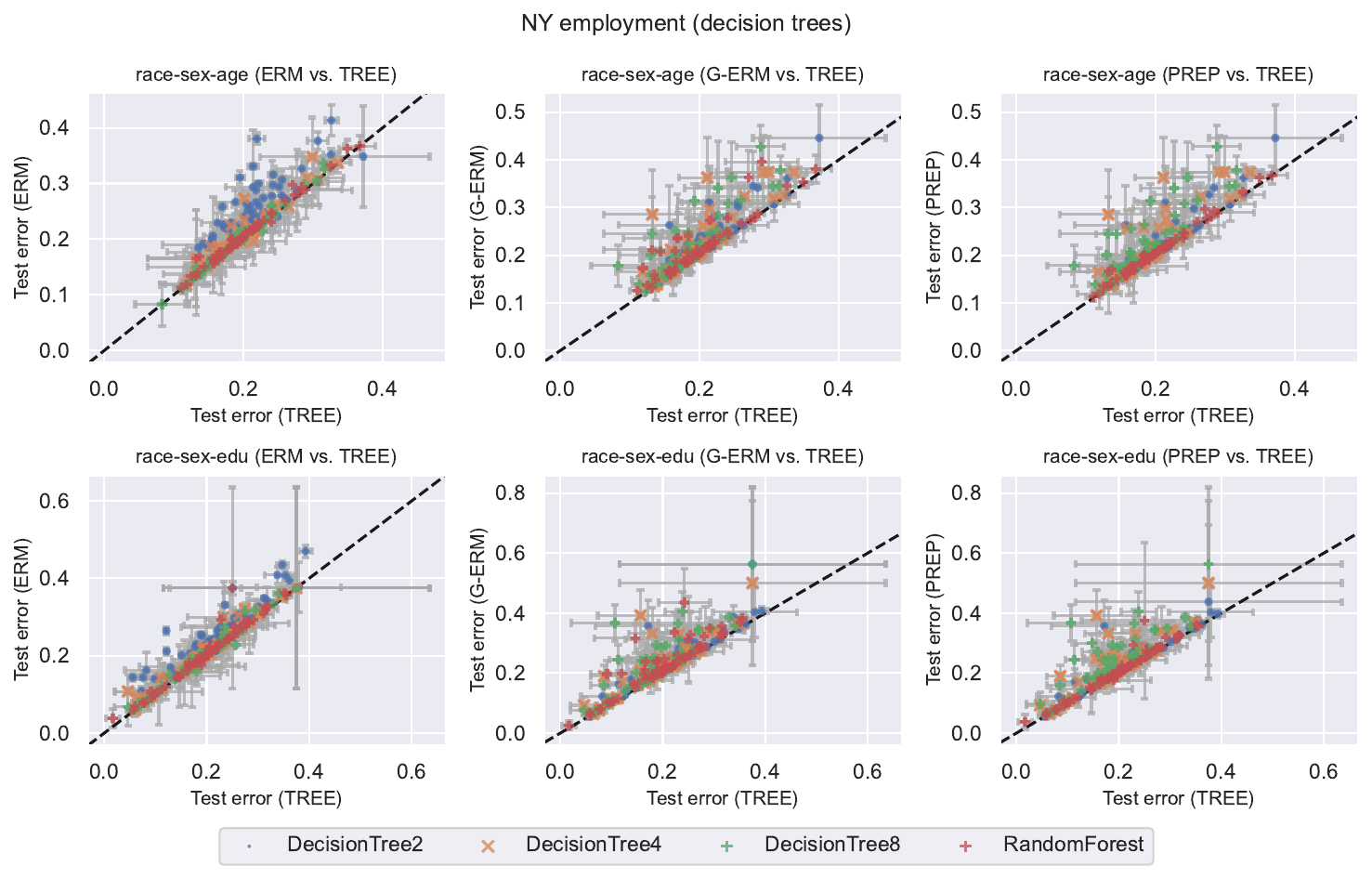}}
\caption{Test error of \mgltree vs.~ERM, Group ERM, and \prepend on \texttt{race}-\texttt{sex}-\texttt{age} and \texttt{race}-\texttt{sex}-\texttt{edu} groups (NY Employment). Each point corresponds to a group; points above the $y = x$ line show that \mgltree generalizes better than the competitor method on that particular group. Benchmark hypothesis classes considered: decision trees with \texttt{max\_depth} 2, decision trees with \texttt{max\_depth} 4, decision trees with \texttt{max\_depth} 8, and random forest.} 
\label{fig:employmentNY_dt}
\end{center}
\vskip -0.2in
\end{figure}

\begin{figure}[H]
\vskip 0.2in
\begin{center}
\centerline{\includegraphics[width=\columnwidth]{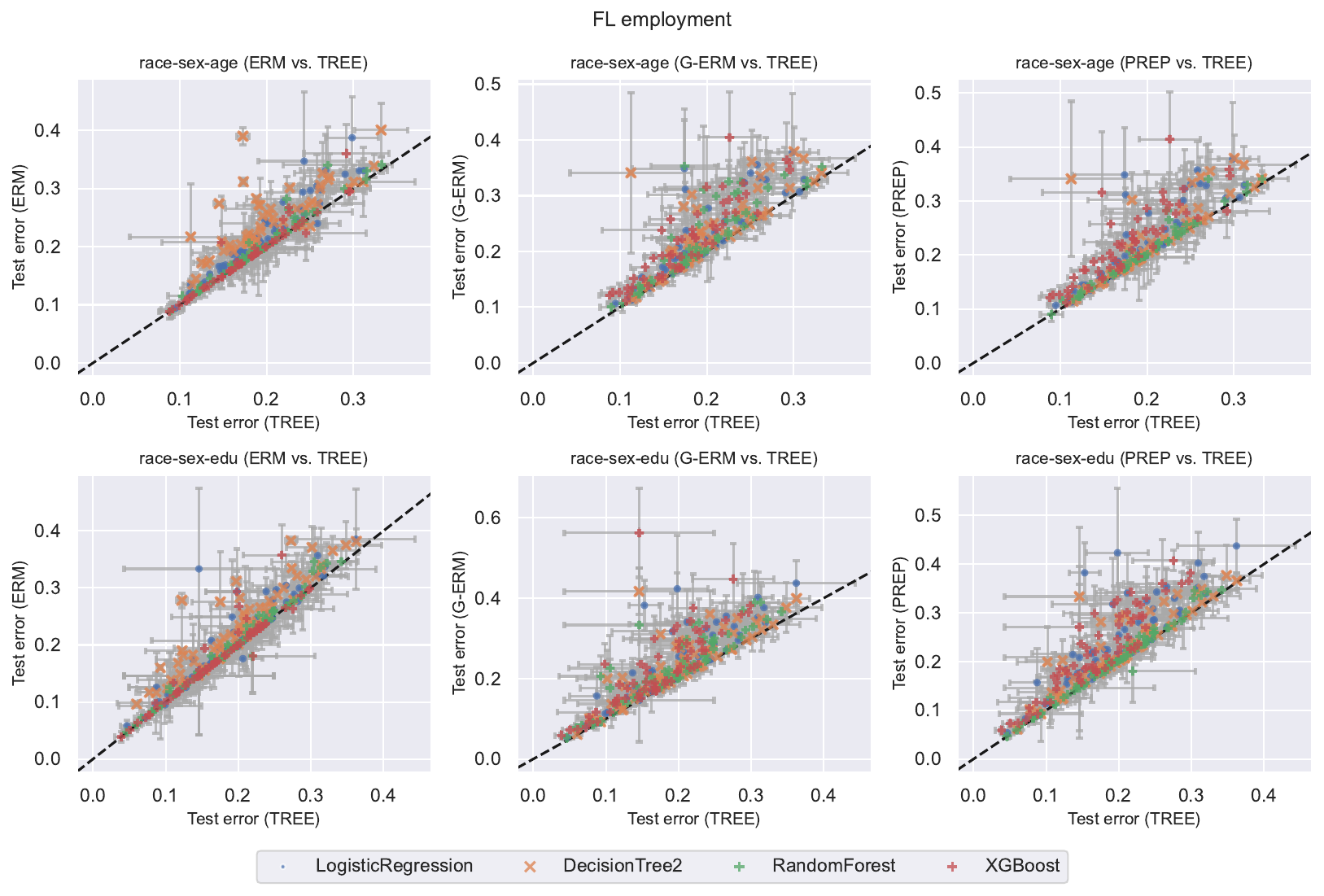}}
\caption{Test error of \mgltree vs.~ERM, Group ERM, and \prepend on \texttt{race}-\texttt{sex}-\texttt{age} and \texttt{race}-\texttt{sex}-\texttt{edu} groups (FL Employment). Each point corresponds to a group; points above the $y = x$ line show that \mgltree generalizes better than the competitor method on that particular group. Benchmark hypothesis classes considered: logistic regression, decision trees with \texttt{max\_depth} 2, random forest, and XGBoost.} 
\label{fig:employmentFL}
\end{center}
\vskip -0.2in
\end{figure}

\begin{figure}[H]
\vskip 0.2in
\begin{center}
\centerline{\includegraphics[width=\columnwidth]{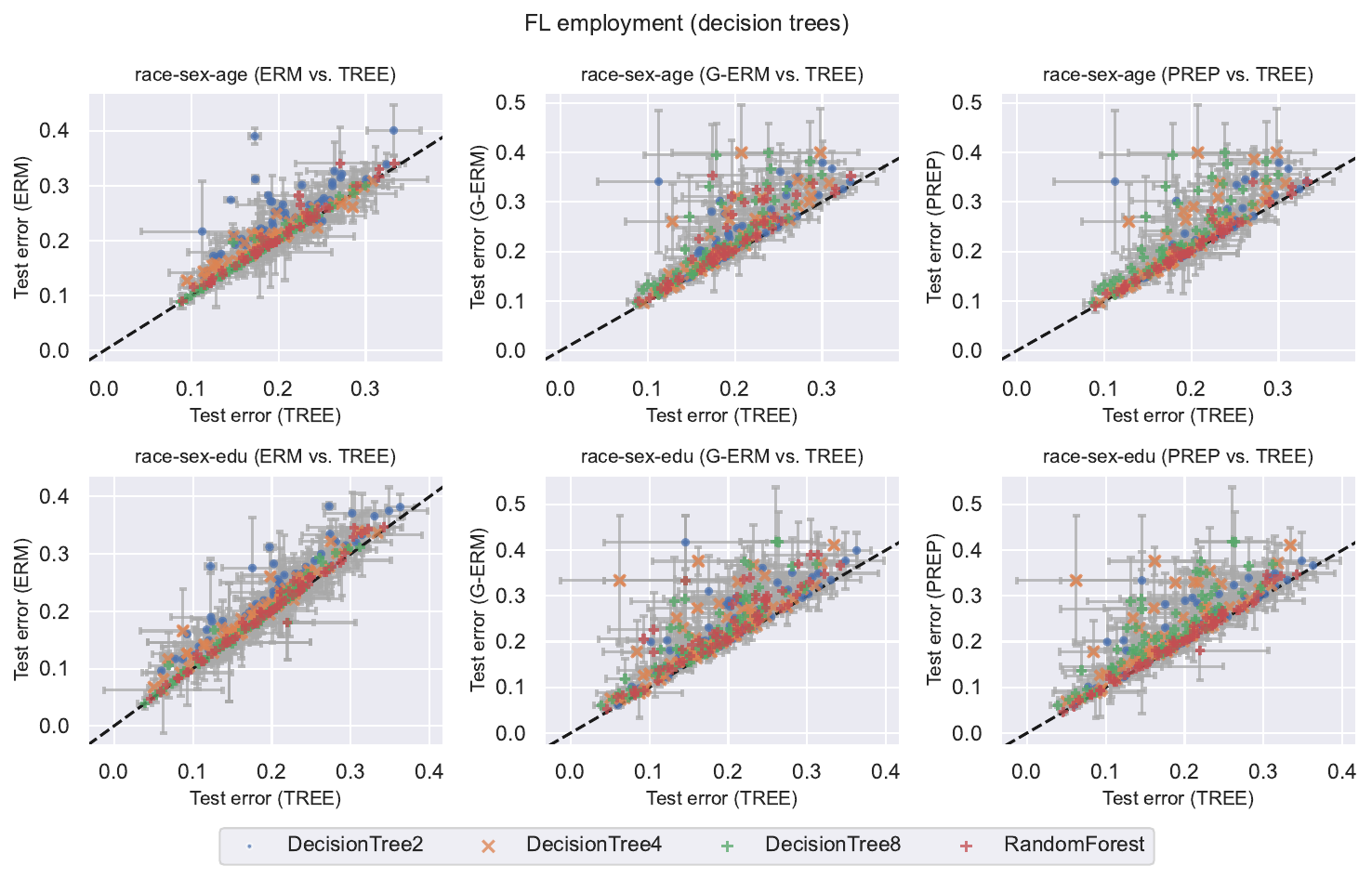}}
\caption{Test error of \mgltree vs.~ERM, Group ERM, and \prepend on \texttt{race}-\texttt{sex}-\texttt{age} and \texttt{race}-\texttt{sex}-\texttt{edu} groups (FL Employment). Each point corresponds to a group; points above the $y = x$ line show that \mgltree generalizes better than the competitor method on that particular group. Benchmark hypothesis classes considered: decision trees with \texttt{max\_depth} 2, decision trees with \texttt{max\_depth} 4, decision trees with \texttt{max\_depth} 8, and random forest.} 
\label{fig:employmentFL_dt}
\end{center}
\vskip -0.2in
\end{figure}

\begin{figure}[H]
\vskip 0.2in
\begin{center}
\centerline{\includegraphics[width=\columnwidth]{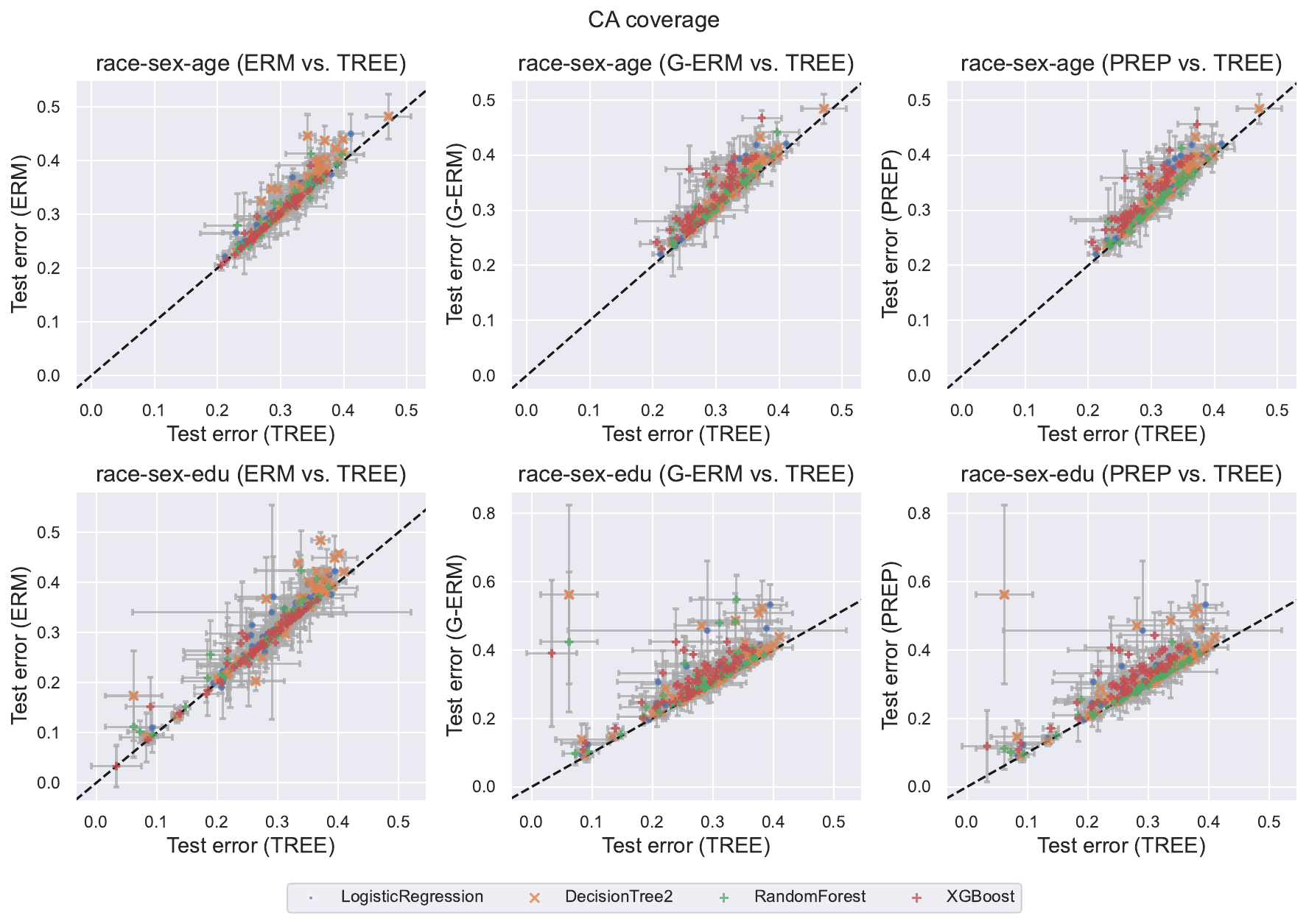}}
\caption{Test error of \mgltree vs.~ERM, Group ERM, and \prepend on \texttt{race}-\texttt{sex}-\texttt{age} and \texttt{race}-\texttt{sex}-\texttt{edu} groups (CA Coverage). Each point corresponds to a group; points above the $y = x$ line show that \mgltree generalizes better than the competitor method on that particular group. Benchmark hypothesis classes considered: logistic regression, decision trees with \texttt{max\_depth} 2, random forest, and XGBoost.} 
\label{fig:coverageCA}
\end{center}
\vskip -0.2in
\end{figure}

\begin{figure}[H]
\vskip 0.2in
\begin{center}
\centerline{\includegraphics[width=\columnwidth]{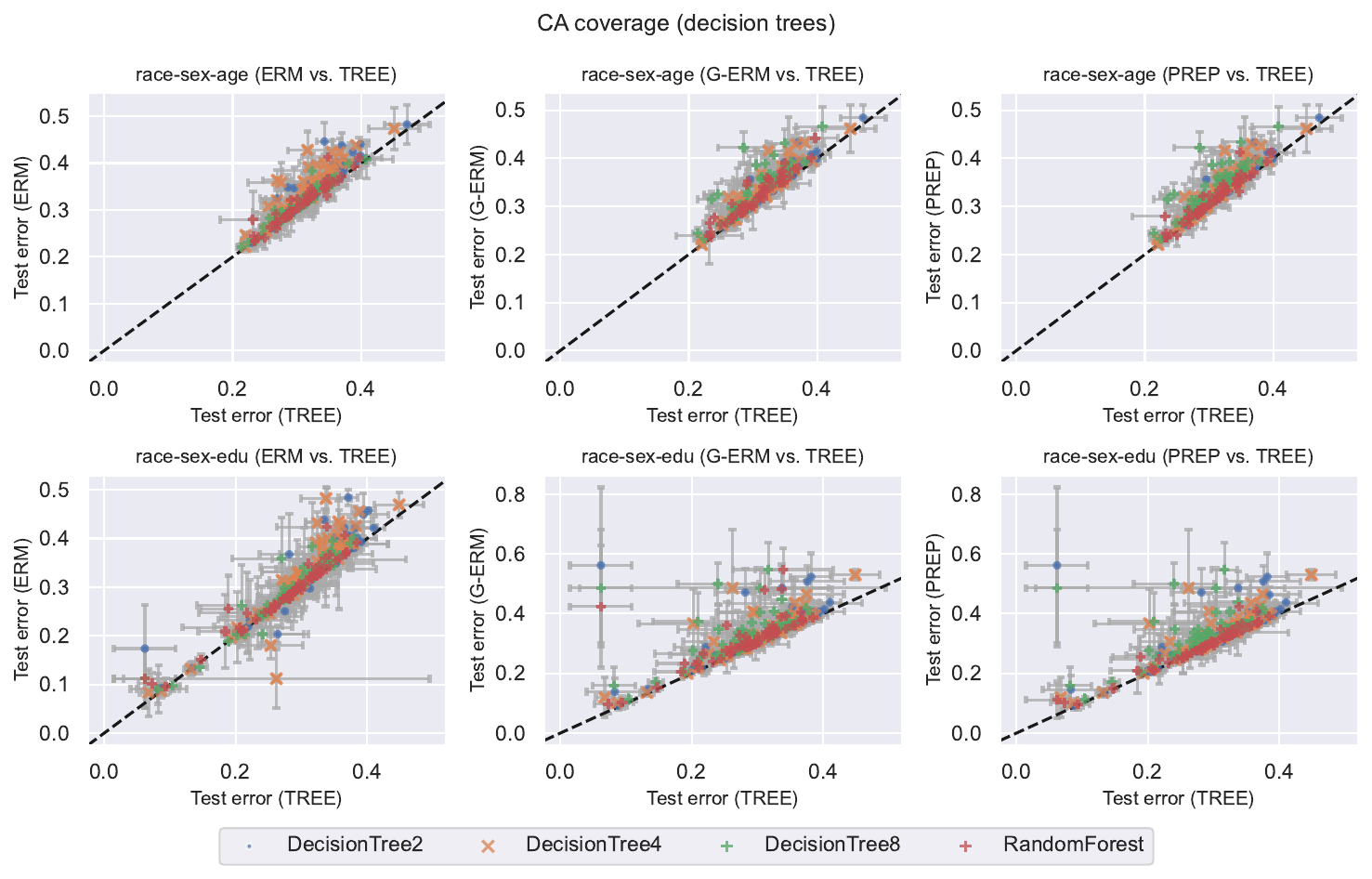}}
\caption{Test error of \mgltree vs.~ERM, Group ERM, and \prepend on \texttt{race}-\texttt{sex}-\texttt{age} and \texttt{race}-\texttt{sex}-\texttt{edu} groups (CA Coverage). Each point corresponds to a group; points above the $y = x$ line show that \mgltree generalizes better than the competitor method on that particular group. Benchmark hypothesis classes considered: decision trees with \texttt{max\_depth} 2, decision trees with \texttt{max\_depth} 4, decision trees with \texttt{max\_depth} 8, and random forest.} 
\label{fig:coverageCA_dt}
\end{center}
\vskip -0.2in
\end{figure}

\begin{figure}[H]
\vskip 0.2in
\begin{center}
\centerline{\includegraphics[width=\columnwidth]{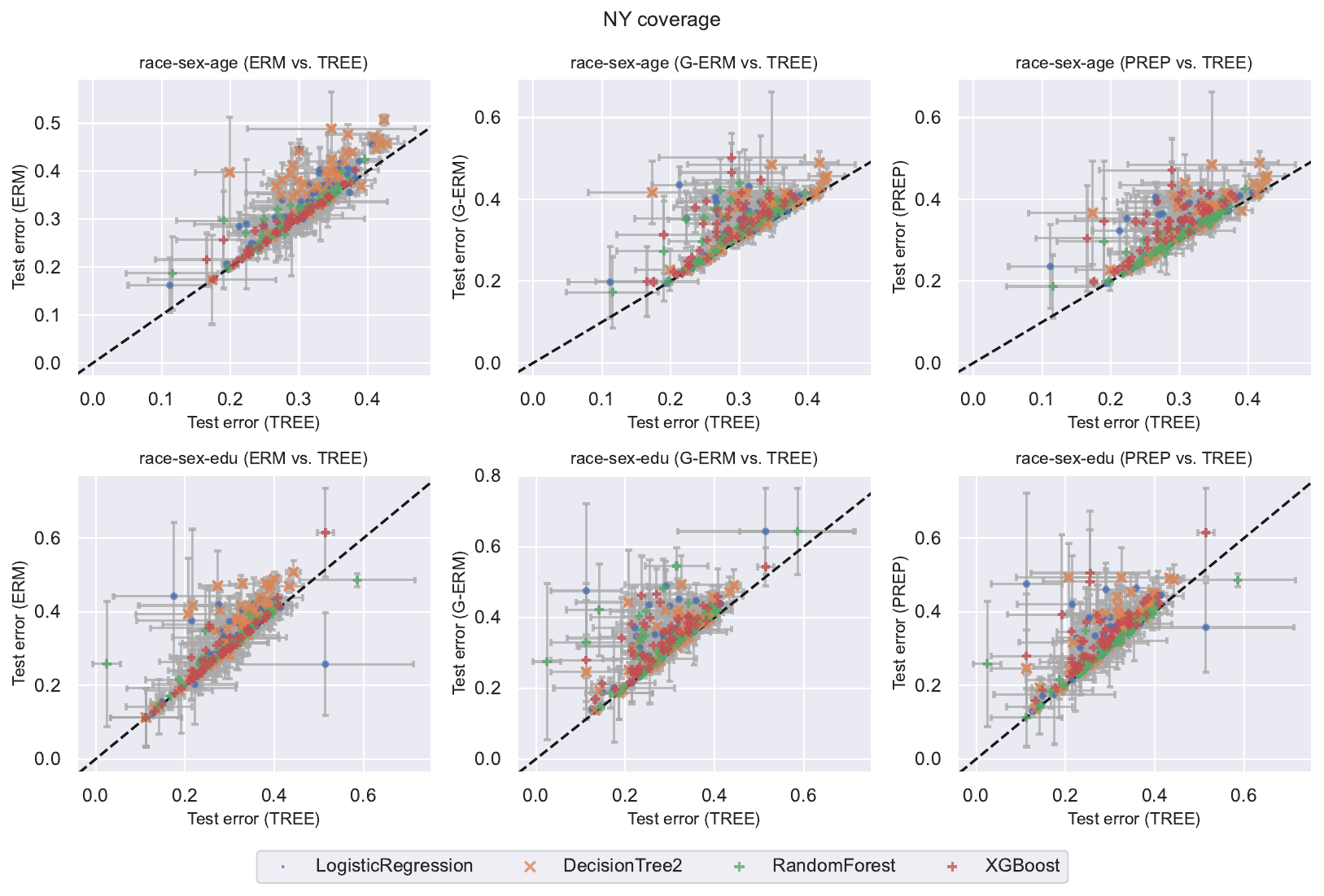}}
\caption{Test error of \mgltree vs.~ERM, Group ERM, and \prepend on \texttt{race}-\texttt{sex}-\texttt{age} and \texttt{race}-\texttt{sex}-\texttt{edu} groups (NY Coverage). Each point corresponds to a group; points above the $y = x$ line show that \mgltree generalizes better than the competitor method on that particular group. Benchmark hypothesis classes considered: logistic regression, decision trees with \texttt{max\_depth} 2, random forest, and XGBoost.} 
\label{fig:coverageNY}
\end{center}
\vskip -0.2in
\end{figure}

\begin{figure}[H]
\vskip 0.2in
\begin{center}
\centerline{\includegraphics[width=\columnwidth]{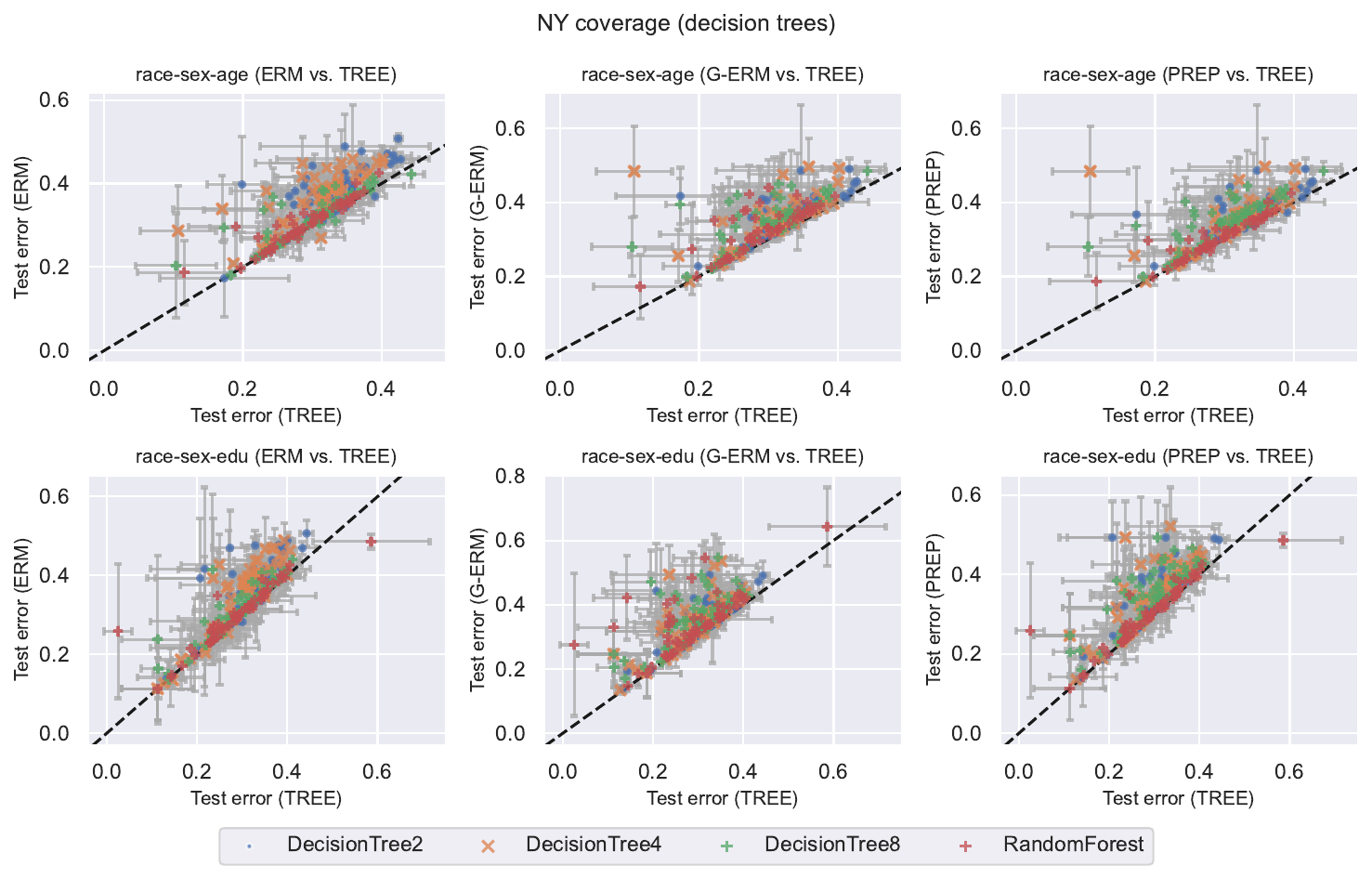}}
\caption{Test error of \mgltree vs.~ERM, Group ERM, and \prepend on \texttt{race}-\texttt{sex}-\texttt{age} and \texttt{race}-\texttt{sex}-\texttt{edu} groups (NY Coverage). Each point corresponds to a group; points above the $y = x$ line show that \mgltree generalizes better than the competitor method on that particular group. Benchmark hypothesis classes considered: decision trees with \texttt{max\_depth} 2, decision trees with \texttt{max\_depth} 4, decision trees with \texttt{max\_depth} 8, and random forest.} 
\label{fig:coverageNY_dt}
\end{center}
\vskip -0.2in
\end{figure}

\begin{figure}[H]
\vskip 0.2in
\begin{center}
\centerline{\includegraphics[width=\columnwidth]{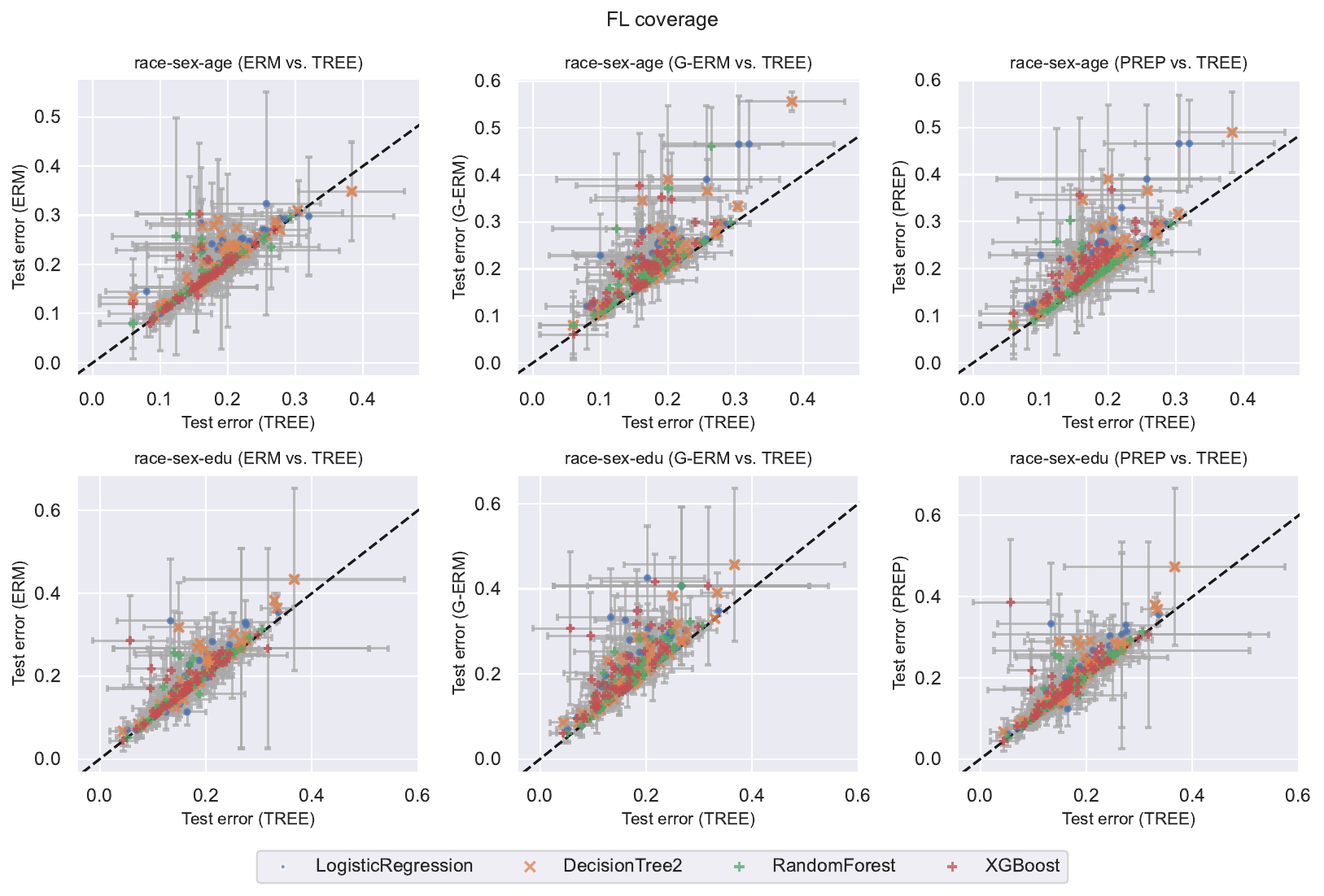}}
\caption{Test error of \mgltree vs.~ERM, Group ERM, and \prepend on \texttt{race}-\texttt{sex}-\texttt{age} and \texttt{race}-\texttt{sex}-\texttt{edu} groups (FL Coverage). Each point corresponds to a group; points above the $y = x$ line show that \mgltree generalizes better than the competitor method on that particular group. Benchmark hypothesis classes considered: logistic regression, decision trees with \texttt{max\_depth} 2, random forest, and XGBoost.} 
\label{fig:coverageFL}
\end{center}
\vskip -0.2in
\end{figure}

\begin{figure}[H]
\vskip 0.2in
\begin{center}
\centerline{\includegraphics[width=\columnwidth]{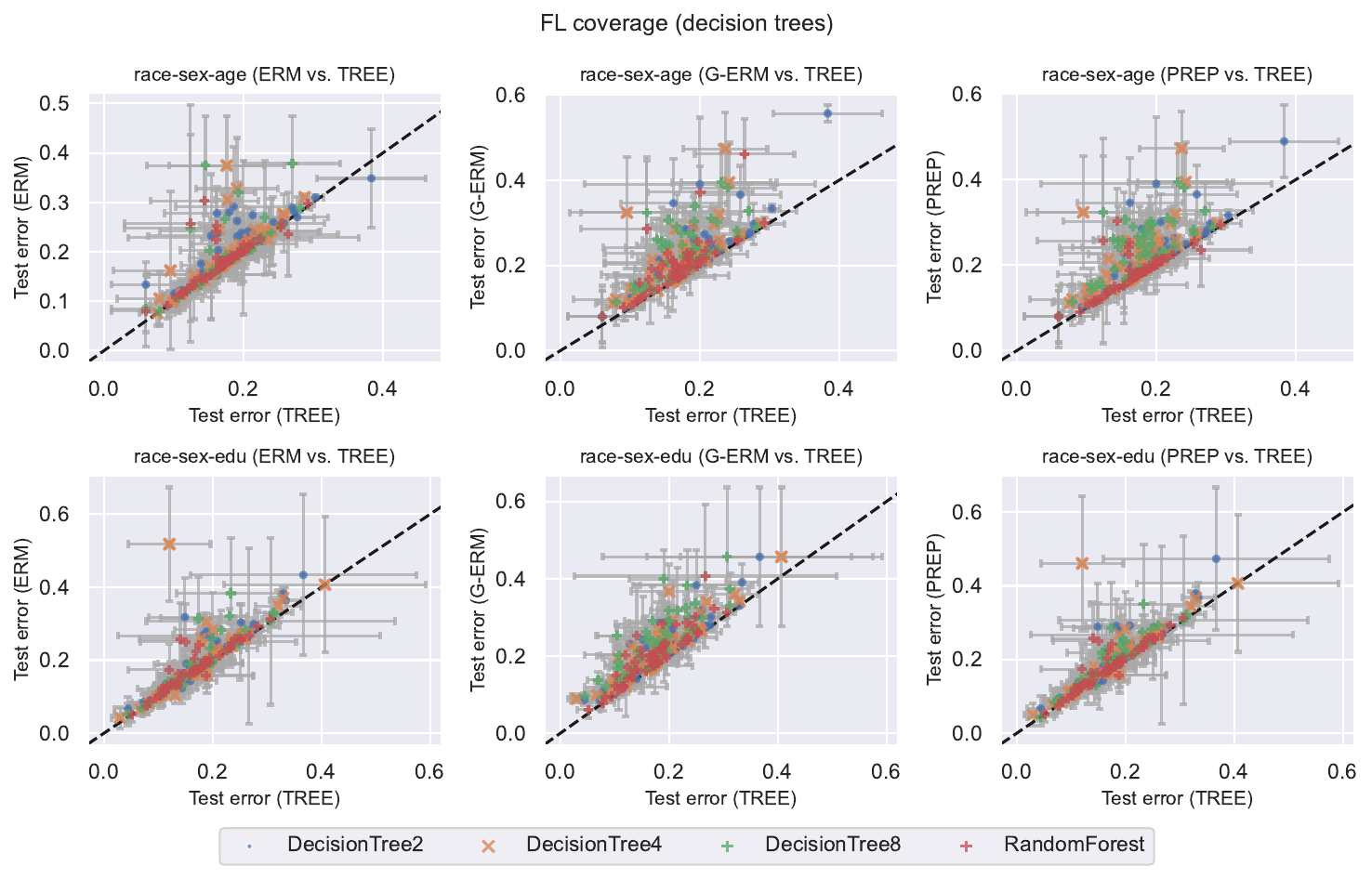}}
\caption{Test error of \mgltree vs.~ERM, Group ERM, and \prepend on \texttt{race}-\texttt{sex}-\texttt{age} and \texttt{race}-\texttt{sex}-\texttt{edu} groups (FL Coverage). Each point corresponds to a group; points above the $y = x$ line show that \mgltree generalizes better than the competitor method on that particular group. Benchmark hypothesis classes considered: decision trees with \texttt{max\_depth} 2, decision trees with \texttt{max\_depth} 4, decision trees with \texttt{max\_depth} 8, and random forest.} 
\label{fig:coverageFL_dt}
\end{center}
\vskip -0.2in
\end{figure}

\begin{figure}[H]
\begin{center}
\centerline{\includegraphics[width=\columnwidth]{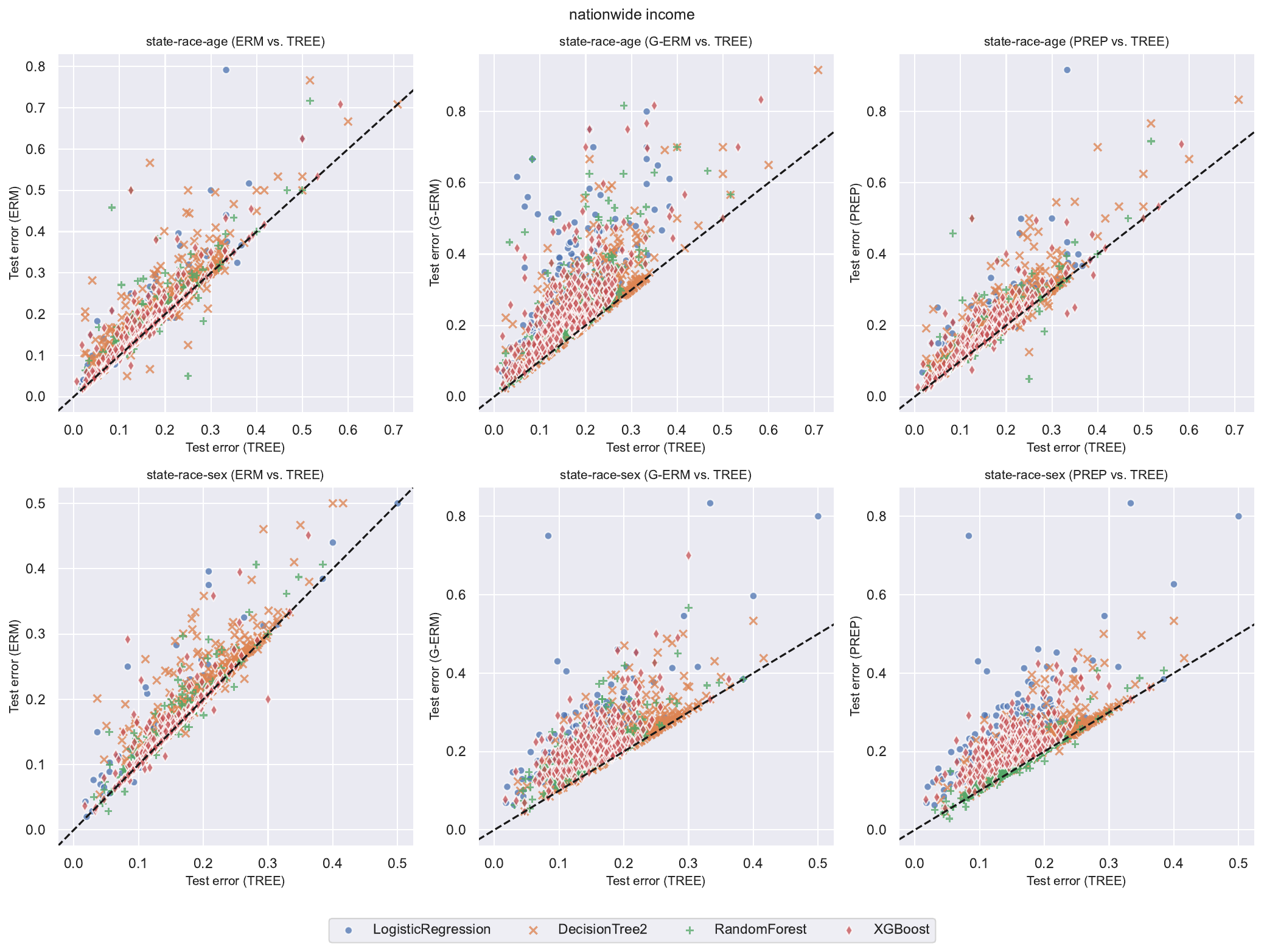}}
\caption{Test error of \mgltree vs.~ERM, Group ERM, and \prepend on \texttt{race}-\texttt{sex}-\texttt{age} and \texttt{race}-\texttt{sex}-\texttt{edu} groups on nationwide dataset for the Income task. Each point corresponds to a group; points above the $y = x$ line show that \mgltree generalizes better than the competitor method on that particular group. Benchmark hypothesis classes considered: logistic regression, decision trees with \texttt{max\_depth} 2, random forest, and XGBoost.} 
\label{fig:incomeST}
\end{center}
\end{figure}

\begin{figure}[H]
\vskip 0.2in
\begin{center}
\centerline{\includegraphics[width=\columnwidth]{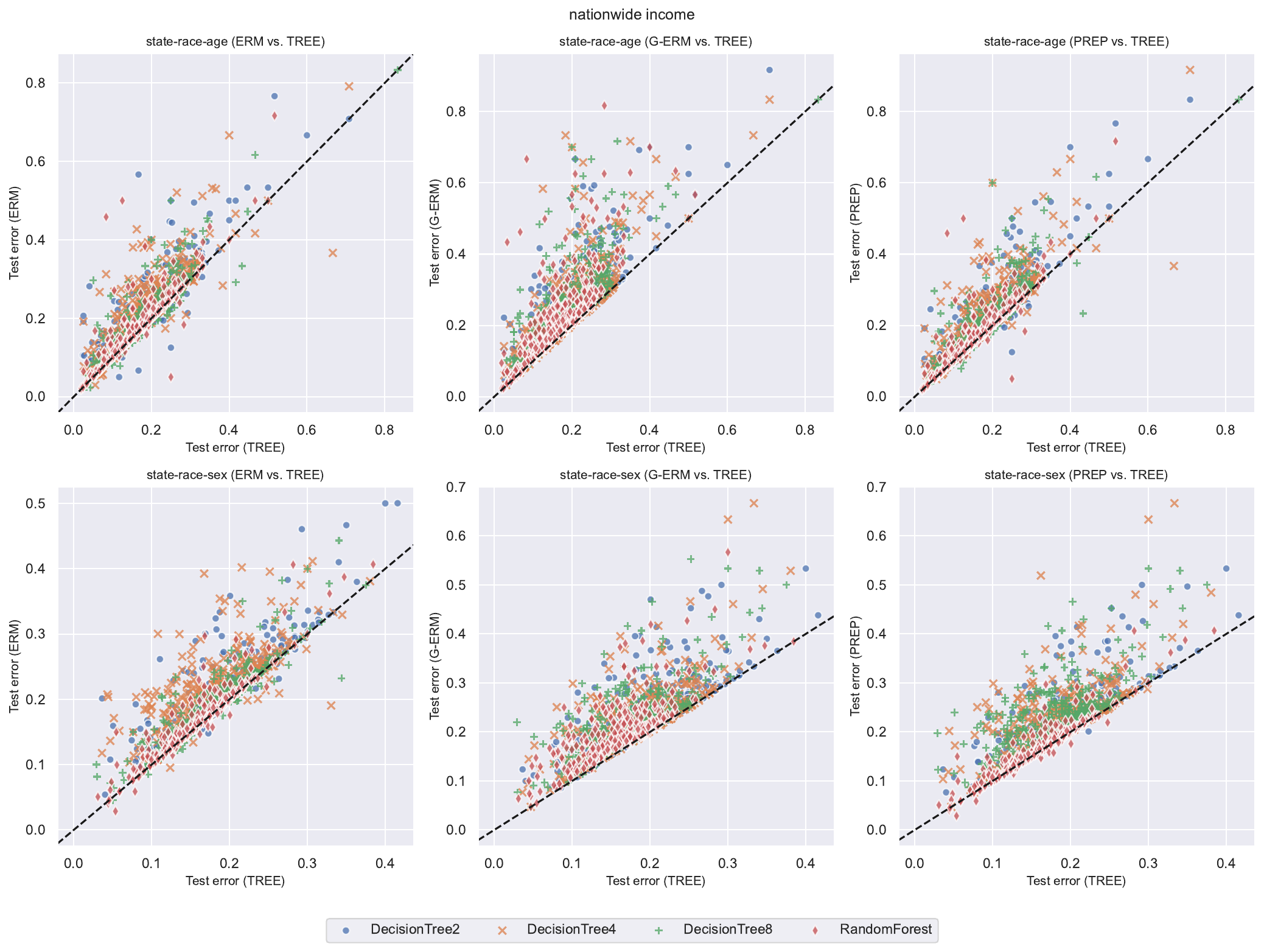}}
\caption{Test error of \mgltree vs.~ERM, Group ERM, and \prepend on \texttt{race}-\texttt{sex}-\texttt{age} and \texttt{race}-\texttt{sex}-\texttt{edu} groups on nationwide dataset for the Income task. Each point corresponds to a group; points above the $y = x$ line show that \mgltree generalizes better than the competitor method on that particular group. Benchmark hypothesis classes considered: decision trees with \texttt{max\_depth} 2, decision trees with \texttt{max\_depth} 4, decision trees with \texttt{max\_depth} 8, and random forest.} 
\label{fig:incomeST_dt}
\end{center}
\vskip -0.2in
\end{figure}

\begin{figure}[H]
\begin{center}
\centerline{\includegraphics[width=\columnwidth]{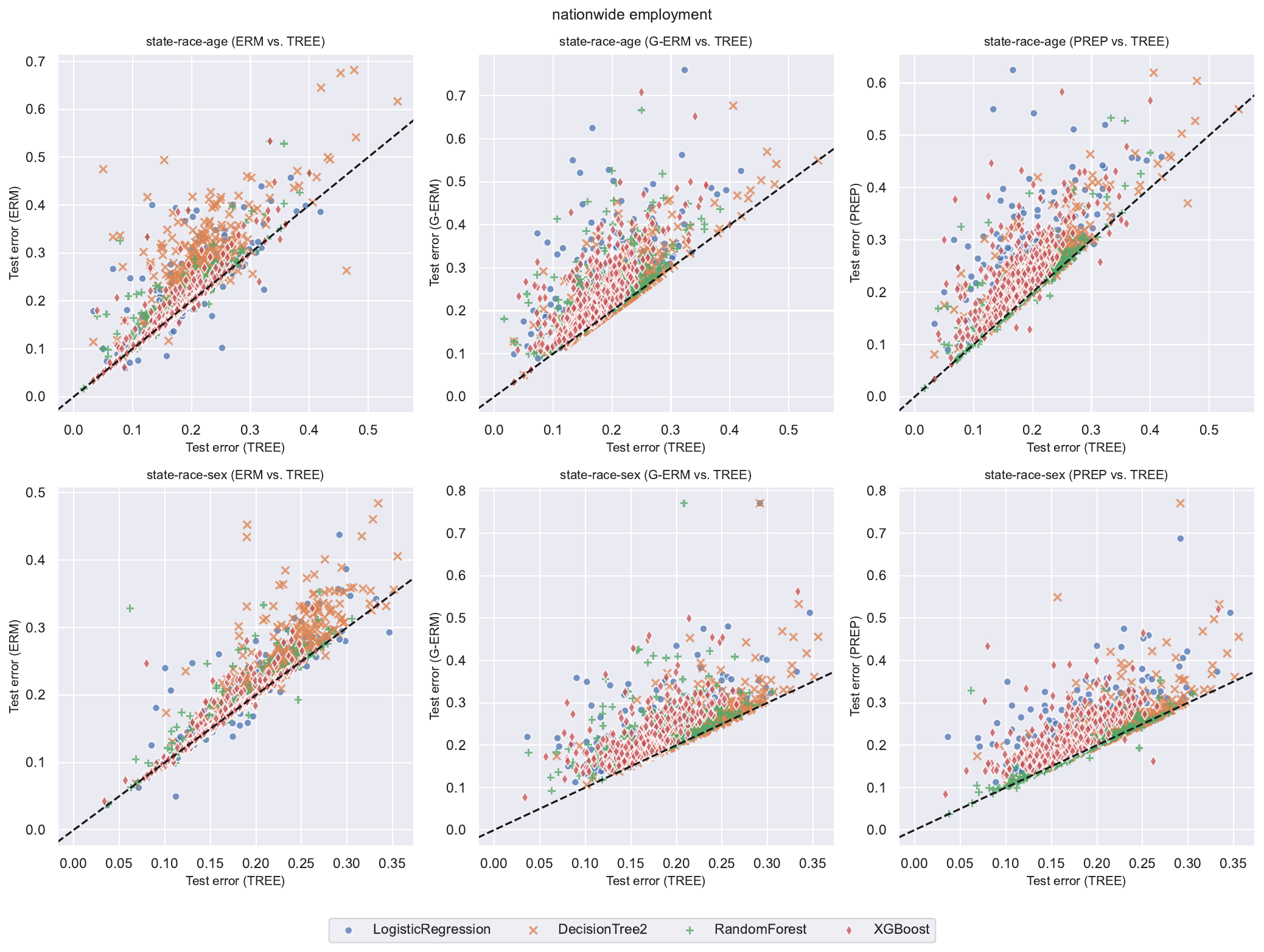}}
\caption{Test error of \mgltree vs.~ERM, Group ERM, and \prepend on \texttt{race}-\texttt{sex}-\texttt{age} and \texttt{race}-\texttt{sex}-\texttt{edu} groups on nationwide dataset for the Employment task. Each point corresponds to a group; points above the $y = x$ line show that \mgltree generalizes better than the competitor method on that particular group. Benchmark hypothesis classes considered: logistic regression, decision trees with \texttt{max\_depth} 2, random forest, and XGBoost.} 
\label{fig:employmentST}
\end{center}
\end{figure}

\begin{figure}[H]
\vskip 0.2in
\begin{center}
\centerline{\includegraphics[width=\columnwidth]{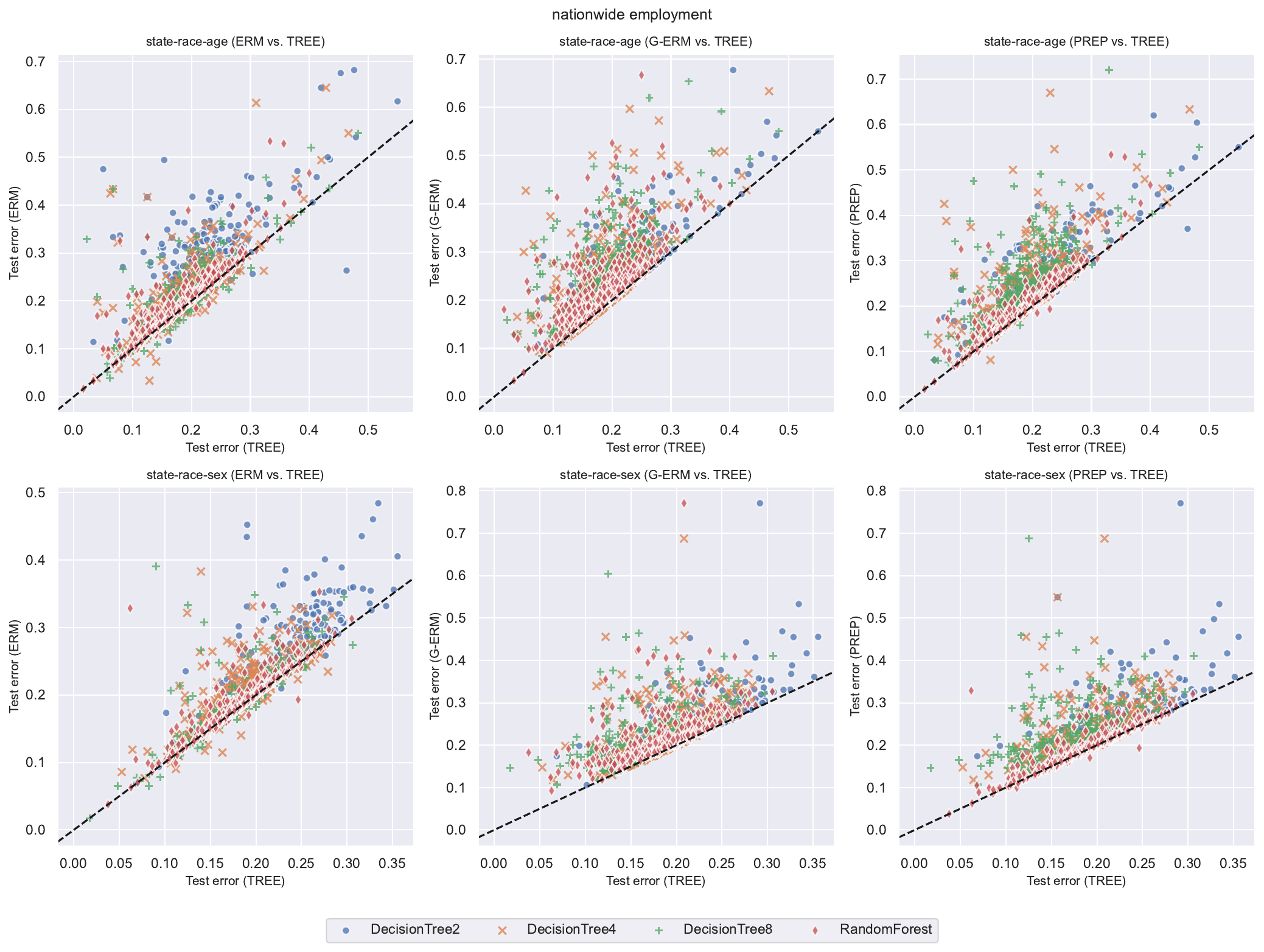}}
\caption{Test error of \mgltree vs.~ERM, Group ERM, and \prepend on \texttt{race}-\texttt{sex}-\texttt{age} and \texttt{race}-\texttt{sex}-\texttt{edu} groups on nationwide dataset for the Employment task. Each point corresponds to a group; points above the $y = x$ line show that \mgltree generalizes better than the competitor method on that particular group. Benchmark hypothesis classes considered: decision trees with \texttt{max\_depth} 2, decision trees with \texttt{max\_depth} 4, decision trees with \texttt{max\_depth} 8, and random forest.} 
\label{fig:employmentST_dt}
\end{center}
\vskip -0.2in
\end{figure}

\begin{figure}[H]
\begin{center}
\centerline{\includegraphics[width=\columnwidth]{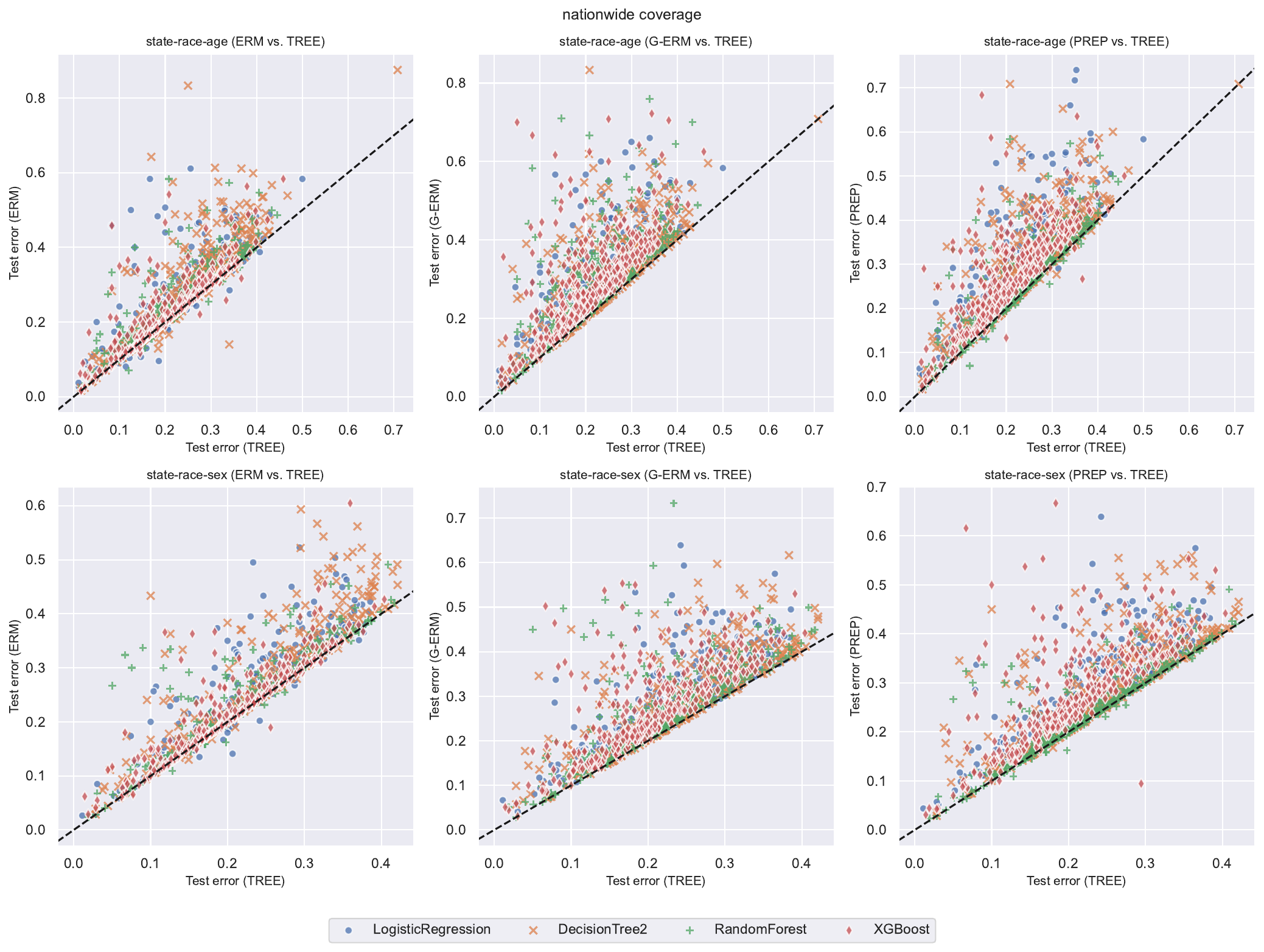}}
\caption{Test error of \mgltree vs.~ERM, Group ERM, and \prepend on \texttt{race}-\texttt{sex}-\texttt{age} and \texttt{race}-\texttt{sex}-\texttt{edu} groups on nationwide dataset for the Coverage task. Each point corresponds to a group; points above the $y = x$ line show that \mgltree generalizes better than the competitor method on that particular group. Benchmark hypothesis classes considered: logistic regression, decision trees with \texttt{max\_depth} 2, random forest, and XGBoost.} 
\label{fig:coverageST}
\end{center}
\end{figure}

\begin{figure}[H]
\vskip 0.2in
\begin{center}
\centerline{\includegraphics[width=\columnwidth]{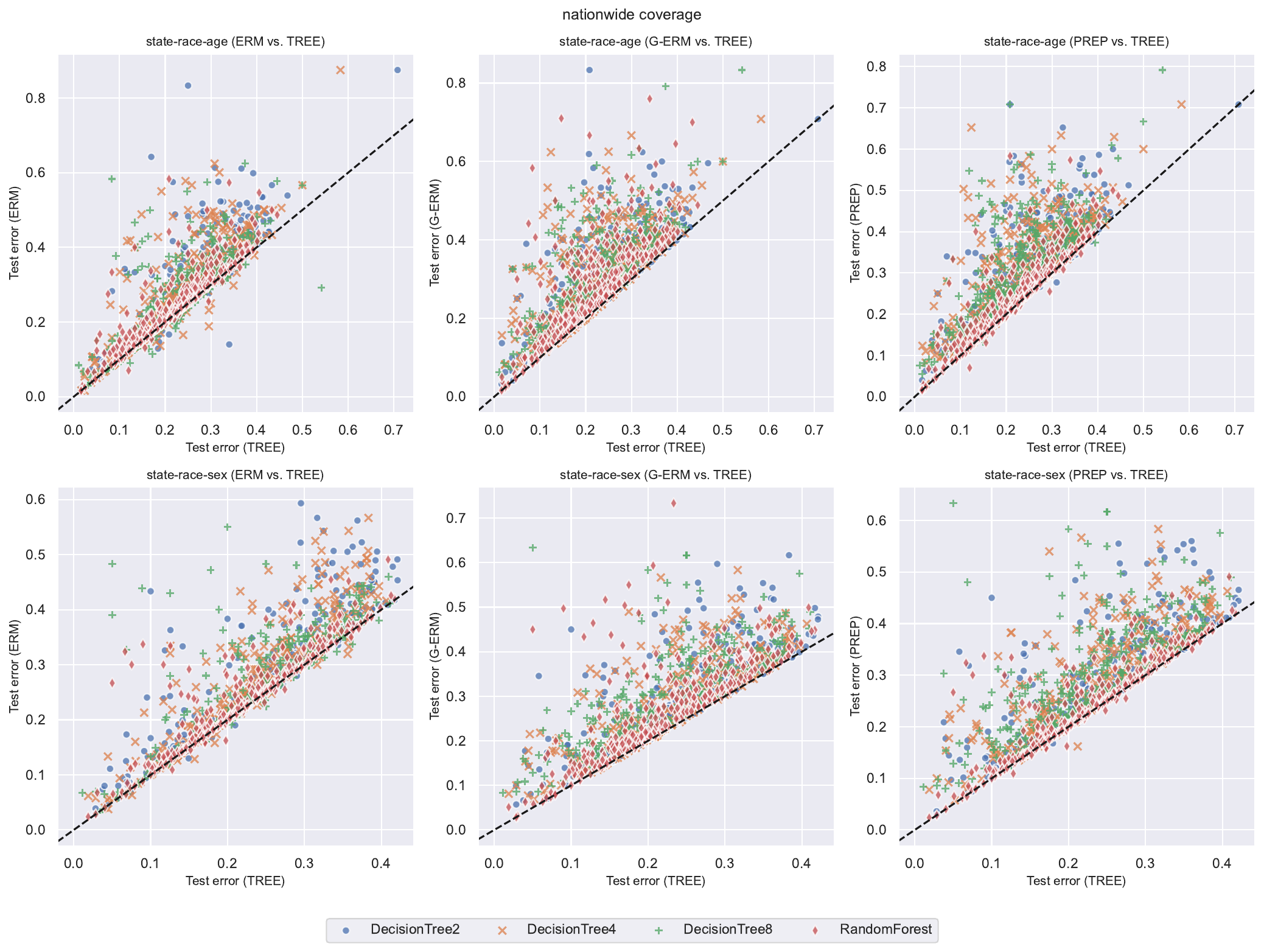}}
\caption{Test error of \mgltree vs.~ERM, Group ERM, and \prepend on \texttt{race}-\texttt{sex}-\texttt{age} and \texttt{race}-\texttt{sex}-\texttt{edu} groups on nationwide dataset for the Coverage task. Each point corresponds to a group; points above the $y = x$ line show that \mgltree generalizes better than the competitor method on that particular group. Benchmark hypothesis classes considered: decision trees with \texttt{max\_depth} 2, decision trees with \texttt{max\_depth} 4, decision trees with \texttt{max\_depth} 8, and random forest.} 
\label{fig:coverageST_dt}
\end{center}
\vskip -0.2in
\end{figure}